%% file: arxiv.tex
\documentclass{article}

\usepackage{microtype}
\usepackage{graphicx}
\usepackage{subcaption}
\usepackage{booktabs}
\usepackage{graphicx}
\usepackage{tabularx}
\usepackage{hyperref}

\usepackage[accepted]{icml2026}

\usepackage{amsmath}
\usepackage{amssymb}
\usepackage{mathtools}
\usepackage{amsthm}
\input{math_commands}

\usepackage[table]{xcolor}
\usepackage{bm}
\usepackage[dvipsnames]{xcolor}
\usepackage{amsmath,amsfonts,bm,pifont}
\usepackage{enumitem}
\usepackage{thm-restate}
\usepackage{mdframed}
\usepackage{multirow}
\usepackage{xspace}
\usepackage{inconsolata}
\usepackage{soul}

\usepackage[capitalize,noabbrev]{cleveref}

\newcommand{\beginsupplement}{%
        \setcounter{table}{0}
        \renewcommand{\thetable}{A.\arabic{table}}%
        \setcounter{figure}{0}
        \renewcommand{\thefigure}{A.\arabic{figure}}%
}

\newcommand{\zerodisplayskips}{%
  \setlength{\abovedisplayskip}{2mm}%
  \setlength{\belowdisplayskip}{2mm}%
  \setlength{\abovedisplayshortskip}{1mm}%
  \setlength{\belowdisplayshortskip}{1mm}}
\appto{\normalsize}{\zerodisplayskips}
\appto{\small}{\zerodisplayskips}
\appto{\footnotesize}{\zerodisplayskips}
\makeatother

\newcommand{\emetric}{\(\mathcal{E}\text{‑}\mathcal{W}_2\)\xspace}
\newcommand{\torusmetric}{\(\mathbb{T}\text{‑}\mathcal{W}_2\)\xspace}
\newcommand{\ticametric}{\(\mathrm{TICA}\text{‑}\mathcal{W}_2\)\xspace}
\newcommand{\namelong}{\textsc{Autoregressive Boltzmann Generators}\xspace}
\newcommand{\nameshort}{\textsc{ArBG}\xspace}
\newcommand{\shortname}{\nameshort}
\newcommand{\namemodel}{\textsc{Robin}\xspace}
\newcommand{\modelname}{\namemodel}

\newcommand{\proplinebreak}{\\&}
\definecolor{mypurple}{RGB}{108,60,153}

\usepackage[textsize=tiny]{todonotes}

\icmltitlerunning{Autoregressive Boltzmann Generators}

\begin{document}

\twocolumn[
  \icmltitle{Autoregressive Boltzmann Generators}

\icmlsetsymbol{equal}{*}

\begin{icmlauthorlist}
  \icmlauthor{Danyal Rehman}{mila,broad,aithyra,udem}
  \icmlauthor{Charlie B. Tan}{mila,udem,oxford}
  \icmlauthor{Yoshua Bengio}{mila,udem,cifar}
  \icmlauthor{Avishek Joey Bose}{mila,imperial,equal}
  \icmlauthor{Alexander Tong}{aithyra,equal}
\end{icmlauthorlist}

  \icmlaffiliation{mila}{Mila -- Qu\'{e}bec AI Institute}
  \icmlaffiliation{broad}{Broad Institute of MIT \& Harvard}
  \icmlaffiliation{aithyra}{Aithyra}
  \icmlaffiliation{oxford}{University of Oxford}
  \icmlaffiliation{udem}{Universit\'{e} de Montr\'{e}al}
  \icmlaffiliation{cifar}{CIFAR Senior Fellow}
  \icmlaffiliation{imperial}{Imperial College London}

  \icmlcorrespondingauthor{Danyal Rehman}{danyal.rehman@mila.quebec}
  \icmlcorrespondingauthor{Alexander Tong}{atong@aithyra.at}

  \icmlkeywords{Boltzmann Generators, AI for Science, Molecules, Peptides}

  \vskip 0.3in
]

\printAffiliationsAndNotice{\textsuperscript{*}Equal advising.}

\begin{abstract}
\looseness=-1
Efficient sampling of molecular systems at thermodynamic equilibrium is a hallmark challenge in statistical physics. This challenge has driven the development of Boltzmann Generators (BGs), which allow rapid generation of uncorrelated equilibrium samples by combining a generative model with exact likelihoods and an importance sampling correction. However, modern BGs predominantly rely on normalizing flows (NFs), which either suffer from limited expressivity due to strict invertibility constraints (discrete time) or computationally expensive likelihoods (continuous time). In this paper, we propose \namelong (\nameshort)---a novel autoregressive modelling framework---that overcomes these limitations by departing from the flow-based BG paradigm. \nameshort circumvents the topological constraints of flows and enables sequential inference-time interventions, while offering enhanced scalability by leveraging architectures effective in Large Language Models. We empirically demonstrate that \nameshort leads to significant improvements over flow-based models across all benchmarks, but particularly in larger peptide systems such as the 10-residue Chignolin. Furthermore, we introduce \namemodel, a 132 million parameter transferable model trained with the \nameshort framework which improves over the previous state-of-the-art, reducing the zero-shot energy error, \emetric, on 8-residue systems by over 60$\%$. The code can be found at the following link:\ \url{https://github.com/danyalrehman/autobg}.\end{abstract}
\everypar{\looseness=-1}
\section{Introduction}
\label{sec:introduction}

\looseness=-1
A central insight of statistical mechanics is that macroscopic phenomena---such as protein folding~\citep{noe_constructing_2009,lindorff-larsen_how_2011}, magnetization of an Ising model~\citep{yang1952spontaneous}, and crystal structure formation~\citep{parrinello1980crystal,matsumoto2002molecular}---are governed by the ensemble of microscopic states at equilibrium. This equilibrium distribution is known as the \textit{Boltzmann distribution}:\  $\mu_{\text{target}}(x) \propto \exp\left(-\gE(x)\right)$, where $\gE(x)$ is the dimensionless potential energy of a conformation $x \in \R^{n \times 3}$. Accordingly, the computational challenge is to efficiently draw statistically-independent samples from this target distribution, $\mu_{\text{target}}(x)$.

\looseness=-1
A key characteristic of this sampling problem is that states in thermodynamic equilibrium---i.e., modes of the distribution---are often sparse and well-separated by high-energy barriers~\citep{wirnsberger2020targeted,rizzi2021targeted}. The dominant approach for exploring this conformational landscape remains Molecular Dynamics (MD) simulations~\citep{alder1959studies,rahman1964correlations}, which seek to simulate the equations of motion with finely-discretized time-steps;\ however, this approach suffers from a severe timescale issue. Specifically, MD often use time-steps on the order of femtoseconds ($10^{-15} \text{s}$), but mode mixing across these high-energy barriers typically requires timescales of microseconds ($10^{-6} \text{s}$) to seconds ($10^0 \text{s}$)~\citep{OLSSON2026103213}. Consequently, the vast majority of MD computation is spent simulating high-frequency vibrations within local minima rather than exploring the global energy landscape, rendering MD computationally prohibitive for practical problems~\citep{perez2025breaking}. While many accelerated MD schemes have been explored~\citep{henin2022enhanced,syed_nonreversible_2021,klein2023timewarp,kapuśniak2025marsfmgenerativemodelingmolecular}, they still have difficulty with this fundamental mixing problem.
\begin{figure*}[!t]
  \centering
  \includegraphics[width=\textwidth]{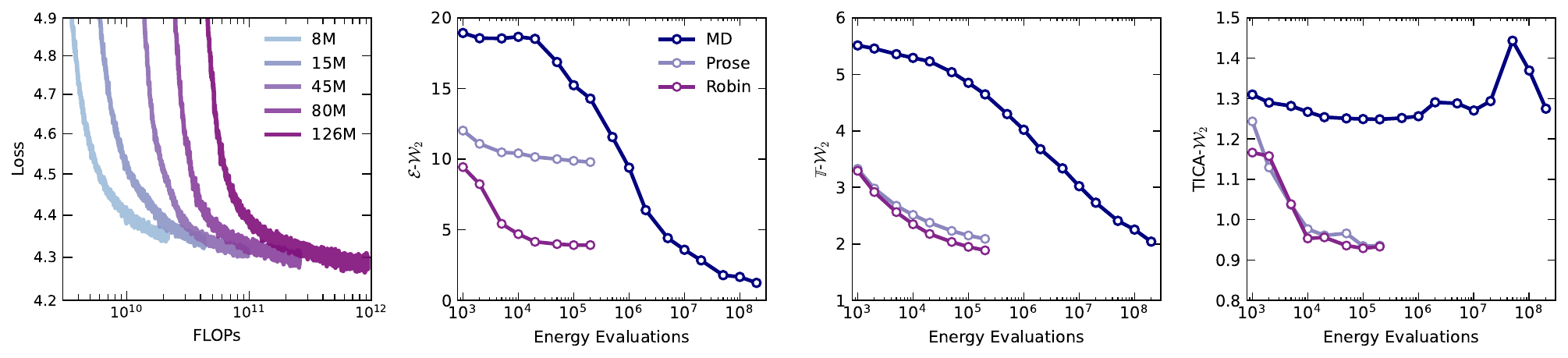}
  \caption{\textbf{(Left)} The training loss curves of models of varying scale (ranging from 8M to 126M parameters) as a function of the total number of floating point operations (FLOPs) on the decapeptide:\ Chignolin. \textbf{(Right)} Global and local performance metrics as a function of the number of energy evaluations demonstrating inference-time scaling between our transferable model \namemodel, the previous state-of-the-art BG Prose, and a short molecular dynamics (MD) chain for the same number of energy function evaluations.}
  \label{fig:scaling}
\end{figure*}

\looseness=-1
Boltzmann Generators (BGs)~\cite{noe2019boltzmann} have emerged as a powerful framework to circumvent this. BGs learn a generative model $p_\theta(x)$ to propose samples for importance sampling, leveraging the exact model likelihood $p_\theta(x)$ and target energy $\mathcal{E}(x)$. This allows for the parallel generation of independent and consistent samples \textit{without having to traverse between modes}, making BGs an attractive framework for equilibrium sampling of large molecular systems;\ however, to satisfy the requirement for tractable likelihoods, the field has relied almost exclusively on normalizing flows (NFs) in either discrete~\citep{tabak2010density,dinh2016density,rezende2015variational,tan_scalable_2025,rehman2025fortforwardonlyregressiontraining} or continuous time~\citep{chen_neural_2018,rehman2025falconfewstepaccuratelikelihoods}.

\looseness=-1
This reliance on flow-based architectures, however, imposes severe theoretical limitations. Fundamentally, NFs in practical instantiations are not only diffeomorphisms but also \emph{homeomorphisms}, as the prior is a single Gaussian~\citep{cornish2020relaxing,dupont2019augmented,runde2005taste}. Consequently, such generative models preserve the topology of their domain and struggle to model target distributions with distinct topologies to the prior, e.g., disjoint supports,  differing number of connected components, or ``holes". The equilibrium distribution of molecular conformations precisely exhibits these challenging topologies, as several metastable states are separated by regions of high-energy barriers. To morph the single connected mode of a Gaussian to these effectively separated states, a flow is forced to perform extreme deformations, stretching space across thin bridges and compressing it into modes. This leads to highly non-smooth mappings prone to exploding Lipschitz constants and ill-conditioned Jacobians, resulting in discrete time NFs being notoriously unstable, greatly limiting model expressivity.

\looseness=-1
Continuous normalizing flows (CNFs)~\citep{chen_neural_2018} are free from the same architectural constraints as discrete-time NFs. In addition, modern training strategies for flow-matching~\citep{peluchetti2021,liu_rectified_2022,lipman_flow_2022,albergo_building_2023} can greatly stabilize optimization, yet the ill-conditioning manifests itself during inference. The learned vector field becomes highly non-smooth, resulting in \textit{stiff} dynamics~\citep{hochbruck2010exponential,hochbruck2020convergence} leading to a large number of function evaluations being required at inference for accurate likelihoods to be obtained. This presents an unavoidable tradeoff:\ (\textbf{1}) Discrete-time NFs yield efficient likelihoods but struggle with poor sample quality and limited expressivity, while (\textbf{2}) CNFs are more expressive, but require expensive ODE integrators for likelihood evaluation, rendering importance sampling computationally prohibitive.

\looseness=-1
\xhdr{Present work} 
In this work, we propose \namelong (\nameshort), a novel alternative to flow-based BGs that circumvents these limitations. \nameshort employs an autoregressive paradigm to factorize the molecular density into a sequence of conditional densities:\ $p_\theta(x) = \prod_j p(x_j | x_{<j})$. This formulation offers three distinct advantages:\ (\textbf{1}) \nameshort overcomes the expressivity bottlenecks of discrete time flows without incurring the computational cost of continuous time flows;\ (\textbf{2}) \nameshort avoids the numerical instability of learning high-distortion diffeomorphisms, allowing it to model discontinuous jumps and separate modes present in complex multi-modal target densities;\ and (\textbf{3}) \nameshort benefits from the investment and advances in discrete generative modelling that power modern LLMs, and exhibits similar scaling properties in both model size and inference samples (as shown in \cref{fig:scaling}). 

\looseness=-1
Our main contributions are summarized as follows:\
\begin{itemize}[topsep=0pt, partopsep=0pt, itemsep=4pt, parsep=0pt, leftmargin=*]
    \item \looseness=-1 We introduce \namelong (\nameshort), the first scalable, autoregressive and diffeomorphism-free method for Boltzmann Generation.
    \item \looseness=-1 We investigate various proposal formulations, demonstrating that discrete binning not only offers superior training stability and scalability compared to continuous mixture models, but also unlocks sequential inference-time interventions that are not possible in flow-based architectures.
    \item \looseness=-1 We demonstrate that \nameshort consistently outperforms all baseline methods across every single-peptide benchmark, with especially strong performance and scalability demonstrated on the 10-residue Chignolin system (\cref{fig:scaling}). 
    \item \looseness=-1 We introduce \namemodel, a 132-million parameter \textit{transferable} autoregressive generative model that achieves zero-shot generalization to unseen peptides, reducing Energy-W2 error, \emetric, by over $60\%$ compared to the previous state-of-the-art approach on large peptides:\  Prose, a discrete-time normalizing flow~\citep{tan2025amortizedsamplingtransferablenormalizing}. 
\end{itemize}

\newpage\clearpage
\section{Background and Preliminaries}
\label{sec:background}

\looseness=-1
\xhdr{Thermodynamic Equilibrium Sampling} We consider molecular systems comprising of $n$ atoms, represented at all-atom resolution by its conformations $x \in \mathbb{R}^{n \times 3}$. The equilibrium behaviour of the system is characterized by the following target Boltzmann distribution:\
\begin{equation*}
    \mu_{\text{target}}(x) \propto \exp\left(-\mathcal{E}(x)\right), \hspace{2pt} \mathcal{Z} = \int_{\mathcal{X}} \exp\left(-\mathcal{E}(x)\right) dx.
\end{equation*}
\looseness=-1
Here, $\mathcal{E} : \mathbb{R}^{n \times 3} \to \mathbb{R}$ denotes the potential energy, and $\mathcal{Z}$ is the partition function, which is computationally intractable to evaluate exactly. Macroscopic properties are obtained by computing observables $\phi(x)$ as expectations with respect to the Boltzmann distribution $\mu_{\text{target}}(x)$. A commonly-used approach to obtain consistent samples from the Boltzmann distribution leverages self-normalized importance sampling (SNIS) with an easy-to-sample proposal distribution $p(x)$:
\begin{equation*}
    \sE_{ \mu_{\text{target}}(x) }\left[\phi(x)\right] = \sE_{p(x)} \left[ \phi(x) \bar{w}(x) \right] \approx \frac{\sum^K_{i=1} w(x^i) \phi(x^i)}{\sum^K_{i=1} w(x^i)},
\end{equation*}
\looseness=-1
where $w(x^i) = \exp \left( -\gE(x^i)  \right) / p(x^i)$ is the unnormalized importance weight and $x^i \sim p(x), i \in[K]$ are $K$ statistically-independent samples from the proposal $p(x)$. It is well known that SNIS is a consistent estimator whose variance depends on the distributional overlap of the chosen proposal $p(x)$ compared to $\mu_{\text{target}}(x)$~\citep{owen2013monte}.

\cut{
Calculating properties of the molecule requires computing expectations over observables $o(x)$ under this distribution, $\mathbb{E}_{x \sim p(x)}[o(x)]$. However, collecting samples under $p$ is difficult, and therefore alternative strategies, such as self-normalized importance sampling (SNIS), which can use an alternative proposal distribution to compute $p(x)$
\begin{equation}
    \mathbb{E}_{x \sim p(x)}[o(x)] = \mathbb{E}_{x \sim q(x)} \left [\frac{p(x) o(x)}{q(x)} \right]
\end{equation}
for some proposal distribution $q$.
}

\looseness=-1
\xhdr{Boltzmann Generators} A Boltzmann Generator~\citep{noe2019boltzmann} learns a parameterized proposal distribution $p_{\theta}(x)$ that serves to approximate the target distribution $\mu_{\text{target}}(x)$. Crucially, $p_{\theta}(x)$ admits a tractable likelihood for samples $x \sim p_{\theta}(x)$, enabling the computation of importance weights required for self-normalized importance sampling. For instance, when $p_{\theta}(x)$ is a normalizing flow defined by a composition of invertible maps $f_{\theta} = f_M \circ \dots \circ f_1$, the likelihood can be computed by the change-of-variables formula $\log p_{\theta}(x_M) = \log p_0(x_0) - \sum_{i=1}^M \log \left| \partial f_{i,\theta}(x_{i-1}) / \partial x_{i-1} \right|$, with $x_i = f_i(x_{i-1})$, and crucially $f_i$ being invertible.

\looseness=-1
In cases where obtaining a likelihood is theoretically possible, but computationally expensive---such as in CNFs---an SNIS-based resampling step becomes equally impractical. To see this more clearly, we can detail the Augmented ODE of $3n + 1$ dimensions, which tracks both the particle evolution $x_t$ across simulation time $t\in[0,1]$ using the learned velocity field associated with the CNF $v_{\theta}(x_t, t)$, and the corresponding evolution of the induced log-density $\log p_{t,\theta}(x_t)$. We can simulate this trajectory from time $t=0$ to time $t=1$ by sampling from a prior distribution $x_0 ~\sim p_0(x_0)$ and then integrating along the Augmented ODE:
\begin{equation*}
\label{eq:continuous_likelihood}
\begin{bmatrix}
    x_1 \\
    \log p_{1,\theta}(x_1)
\end{bmatrix}
 =\begin{bmatrix}
 x_0 \\
 \log p_0(x_0)
 \end{bmatrix}
 +
 \int_0^1 \begin{bmatrix}
v_{\theta}(x_t, t) \\
-\nabla \cdot v_{\theta}(x_t, t)\end{bmatrix} d t.
\end{equation*}
\looseness=-1
Here $\nabla \cdot$ is the divergence operator, which requires $O(d)$ function evaluations of the network, $v_{\theta}$, per-integration step with $d = n\times3$. While faster unbiased estimators of the divergence, such as the Hutchinson trace estimator~\citep{meyer2021hutch++} exist, the added variance incurred renders them unsuitable for Boltzmann Generators~\citep{klein2023equivariant}.

\cut{
and enables efficient independent sampling. Crucially, this reweighting step requires samples $x \sim q$ and an exact and computationally tractable likelihood $q_\theta(x)$ for every generated sample. If the likelihood is expensive to compute (as in CNFs), the cost of obtaining expectations over observables becomes prohibitively expensive. For a CNF-based proposal with vector field $v$, samples are computed a $d+1$ dimensional ODE:
\begin{equation}\label{eq:continuous_likelihood}
\begin{bmatrix}
    x \\
    \log q_\theta(x)
\end{bmatrix}
 = \int_0^1 \begin{bmatrix}
v(x_t, t) \\
-\text{tr} \left( \frac{\partial v}{\partial x_t} \right )\end{bmatrix} d t
\end{equation}
which requires $O(d)$ function evaluations per integration step.
}

\section{\namelong}

\looseness=-1
We now consider an alternative class of generative models for constructing the proposal distribution $p_{\theta}(x)$ in a Boltzmann Generator---autoregressive (AR) models. We introduce \namelong, the first diffeomorphism-free AR model for molecular systems that operates directly on atoms in their native Cartesian coordinates. In the context of Boltzmann Generators, the central challenge of AR modelling arises from the continuous nature of molecular configurations, in contrast to the discrete data on which AR models are most frequently employed. This setting, therefore, presents an opportunity for developing novel AR frameworks for continuous-state systems. We further motivate our model choice by identifying two key properties required of a proposal distribution within a Boltzmann Generator:\

\looseness=-1
\begin{enumerate}[label=\textbf{(\arabic*)},left=0pt,nosep]
\item 
\looseness=-1
\xhdr{Fast and Accurate Likelihood} A crucial requirement of any proposal in a BG is to facilitate SNIS-based correction of samples---necessitating access to fast, unbiased, and preferably exact likelihood evaluation $p_{\theta}(x)$.
\item
\looseness=-1
\xhdr{Scalability} We require expressive generative model families $p_{\theta}(x)$ that can capture the complex, sparse, and rugged energy landscape of high-dimensional molecular systems with predictable scaling behaviour. 
\end{enumerate}

\looseness=-1
We highlight that both desiderata are demonstrably satisfied by autoregressive models, but not necessarily flow-based models. In particular, autoregressive models \emph{exactly} factorize  the joint density over $x$ as a sequence of conditional distributions that is used to predict the next dimension conditioned on the history $x_{< j} = [x_1, \ldots x_{j-1}]$, for $j \in [d]$:
\begin{equation}
\label{eq:autoregressive}
    \log p_\theta(x) = \sum_{j=1}^d \log p_\theta(x_j | x_{<j}).
\end{equation}
\looseness=-1
Clearly, \eqref{eq:autoregressive} allows for exact likelihood computation in a single pass, avoiding Jacobian determinants or computationally expensive numerical solvers for Neural ODEs. Moreover, AR models have achieved empirical success across a spectrum of domains at scale, including large-scale discrete modelling of text~\citep{comanici2025gemini} and images~\citep{dosovitskiy2020image}. Indeed, the diversity of data domains tackled by autoregressive models comes without specific constraints, such as invertible architectures, or the need to model diffeomorphisms. The latter fact we argue is particularly important for molecular modelling due to  the non-smooth nature of the target Boltzmann $\mu_{\text{target}}(x)$.

\looseness=-1
We next outline how to build \nameshort in~\S\ref{sec:model_instantiation}, and explore new unlocked capabilities of an AR model for BGs in~\S\ref{sec:tools}.

\cut{
\xhdr{Autoregressive Modelling} In this work, we instead parameterize $q_\theta(x)$ as an autoregressive model. We decompose the joint probability of a conformation $x$ into a product of conditional probabilities:
\begin{equation}
    q_\theta(x) = \prod_{i=1}^d q_\theta(x_i | x_{<i})
\end{equation}
where $x_{< i} = [x_1, \ldots x_{i-1}]$. This formulation offers two distinct advantages for Boltzmann generation. First, it allows for exact likelihood computation in a single pass, avoiding expensive Jacobian determinants or ODE solves required by flows. Second, it avoids learning a diffeomorphism in $\R^d$, and is able to model discontinuous jumps and multimodal distributions.

Prior work models molecule conformations using either internal or Cartesian coordinates $x \in \R^{3N = d}$ directly. There are two recent methods for defining an autoregressive ordering over Cartesian coordinates. \citet{tan_scalable_2025} uses the standard ordering in Protein Data Bank (PDB) files, and \citet{tan2025amortizedsamplingtransferablenormalizing} uses a ``backbone permutation'' approach, where all backbone atoms are placed before all sidechain atoms (see \cref{app:ordering} for details). We use the standard PDB ordering as all atoms are added to the molecule. We choose the PDB ordering as it allows us to delay important decisions about backbone atoms as long as possible.
}

\subsection{Autoregressive Modelling of Conformations}
\label{sec:model_instantiation}

\looseness=-1
We consider inputs of the form $x \in \R^{n \times 3}$, which are flattened into a single $d$ dimensional vector by an AR model as defined in~\eqref{eq:autoregressive}. From here onwards, we use subscripts such as $x_j$ to denote the $j$-th dimension of the input vector $x$ rather than a continuous time index as done for CNFs. As AR models require an ordering to model molecular states, we use a residue-by-residue ordering in which the sidechains for each residue immediately follow the backbone atoms.

\looseness=-1
A key technical challenge of instantiating AR models over continuous spaces is determining the parametrization of the conditional distribution over dimensions $p_\theta(x_j | x_{<j})$. We explore several options by leveraging existing ideas from Mixture Density Networks (MDN)~\citep{bishop1994mixture}, which output the parameters of the conditional distribution as a mixture. In addition, we also offer a novel parametrization utilizing a uniform binning strategy that is simple yet enjoys closer alignment to LLM training---unveiling predictable scaling behaviours but now in the context of BGs. 

\looseness=-1
\xhdr{Conditionals as Discretized Mixtures}
Following the pioneering work of PixelCNN++~\citep{salimans2017pixelcnn++}, we demonstrate how to adapt such an approach to model molecular conformations, leading to MoL-PixelCNN++ and a novel extension in GMM-PixelCNN++. These serve as modernized instantiations of the MDN  for molecular conformations and later as baselines in our experiments~\S\ref{sec:experiments}.

 \looseness=-1
 To construct these conditional mixtures, we model the input $x_j$, e.g. a \emph{singular} spatial dimension of an atom, by first discretizing the space into $B$ uniform bins of width $\Delta$ over a pre-determined interval range  $\gI = [C_{\text{min}}, C_{\text{max}})$, with range cutoffs $C_{\text{min}}$ and $C_{\text{max}}$ picked through a data standardization step. Formally, the mapping $\mathcal{Q}: \R \to [L]$ from a continuous coordinate $ x_j \in \R$ to a bin index $b_l \in [L]$ is given by:
\begin{equation*}
    b = \mathcal{Q}(x_j) = \left \lfloor \frac{\tilde x_j - C_{\text{min}}}{\Delta} \right \rfloor, \, \tilde x_j = \texttt{clip}(x_j, C_{\text{min}}, C_{\text{max}}). 
\end{equation*}
\looseness=-1
After binning, we assume a latent assignment $ k$ sampled from a $K$-component mixture, $\pi_k \sim \text{Cat}\left( \pi_1, \dots, \pi_k \right)$. Each mixture component can then be prescribed by an easy-to-parametrize distribution, such as the logistic distribution, leading to a Mixture of Logistics (MoL) or a Gaussian Mixture Model (GMM). We can compute the discretized probability mass that an observed value $x_j$ falls in bin $b$, i.e. $p(X=x_j | x_{<j})$, by leveraging the integral of the CDF differences of the logistic distribution at $\tilde x_j + \Delta/2$ and $\tilde x_j - \Delta/2$. This leads to the following parameterizations of the conditional $p_{\theta}(x_j | x_{<j})$:
\begin{align}
\label{eq:mol}
& \sum_{k=1}^{K} \pi_k \left[ \varphi \left( \frac{\tilde x_j + \Delta/2 - \mu_k}{\sigma_k} \right) - \varphi \left( \frac{\tilde x_j - \Delta/2 - \mu_k}{\sigma_k} \right) \right] \nonumber \\
& \text{where} \, \pi_k \geq 0, \sum^K_{k=1} \pi_k = 1, \sigma_k>0.
\end{align}
\looseness=-1
$\pi_k$ is the weight of the $k$-th mixture component, $\varphi(\cdot)$ is the logistic function or the standard normal CDF, i.e., $\varphi(z) = 0.5\left(1 + \text{erf}\left(z/\sqrt{2}\right)\right)$, for the MoL and GMM cases, respectively. The edge cases of the first and last bin for the MoL-PixelCNN++ are handled by replacing $x_j - \Delta/2$ and $x_j + \Delta/2$ by $-\infty$ and $+\infty$, respectively. Finally, we can also handle the edge case for GMM-PixelCNN++'s leftmost and rightmost bins:
\begin{align*}
    p_k\left(\tilde x_j \in b_0\right) = \Phi\left(\frac{\tilde x_j + \Delta/2 - \mu_k}{\sigma_k}\right), \\
     p_k\left(\tilde x_j \in b_L\right) = 1 - \Phi\left(\frac{\tilde x_j - \Delta/2 - \mu_k}{\sigma_k}\right).
\end{align*}
\cut{
\looseness-1
Similarly, we can choose $p_{\theta}(x_j | x_{<j})$ to be a Gaussian Mixture Model (GMM) with $K$ components that adheres to the binning of $\gI$. The probability of a coordinate value $x_j$ is given by summing across the mixtures $\pi_k$:
\begin{align}
\label{eq:gmm}
    & \sum_{k=1}^{K} \pi_k  \underbrace{\left[\Phi\left(\frac{\tilde x_j + \Delta/2 - \mu_k}{\sigma_k}\right) - \Phi\left(\frac{\tilde x_j - \Delta/2 - \mu_k}{\sigma_k}\right)\right]}_{p_k(\tilde x_j \in b)} \nonumber \\
    & \text{where} \, \sum^K_{k=-1}\pi_k \geq 0, \,  \Phi(z) = \frac{1}{2}\left(1 + \text{erf}\left(\frac{z}{\sqrt{2}}\right)\right),
\end{align}
\looseness=-1
where $\Phi(z)$ is the standard normal and $\mu_k$ and $\sigma_k$ are the means and standard deviations for each Gaussian in the mixture. Finally, we can also handle the edge case for the leftmost and rightmost bins:

\begin{align*}
    p_k\left(\tilde x_j \in b_0\right) = \Phi\left(\frac{\tilde x_j + \Delta/2 - \mu_k}{\sigma_k}\right), \\
     p_k\left(\tilde x_j \in b_L\right) = 1 - \Phi\left(\frac{\tilde x_j - \Delta/2 - \mu_k}{\sigma_k}\right).
\end{align*}
}
\looseness=-1
In both cases, the conditional mixtures admit an analytic log-likelihood, enabling conventional MLE-based training.

\looseness=-1
\xhdr{Uniform Bin Parameterization}
While elegant in theory, Mixture Density Networks like the MoL-PixelCNN++ and GMM-PixelCNN++ are prone to mode collapsing of the mixture components $\pi_k$ to a small subset~\citep{deng2022deep}, potentially leading to suboptimal performance on harder systems of interest. We remedy this problem by abandoning mixture models altogether and introducing, arguably, the simplest parameterization of $p_{\theta}(x_j | x_{<j})$ by predicting directly the bin centres $b_l$ as a Categorical distribution during training. This allows us to directly use an autoregressive model, as commonly done for LLMs for next-token prediction, but now for molecular data. At inference, to recover a continuous coordinate, we can simply add uniform noise to the sampled bin centre: $x_j = b_l + u_l$, where $u_l \sim \text{Unif}(-\Delta/2, \Delta / 2)$. Such a uniform bin parameterization induces the following piecewise-constant conditional density:\
\begin{equation}
\label{eq:unif_bin_parameterization}
p_\theta(\tilde{x}_j | x_{<j})
= \sum_{l=1}^L \pi_\theta(b_l | x_{<j})\frac{\mathbf{1}\{\tilde{x}_j\in b_l\}}{\Delta},
\end{equation}
\looseness=-1
where $ \pi_\theta(b_l | x_{<j}) = \text{Cat}(b_0, \dots, b_L)$. In~\eqref{eq:unif_bin_parameterization} we observe that when $\Delta$ is the same for all bins---i.e., uniform binning---then the conditional log density has a constant offset $\log \Delta$ which vanishes with an increasing number of bins.   

\looseness=-1
Under the uniform bin parameterization, we are also able to quantify the exact log-likelihood error.
Let $p^\star(\tilde{x}_j | x_{<j})$ denote the true conditional density and define its bin masses:
\begin{equation*}
    p^\star(b_l | x_{<j}):=
\int_{b_l} p^\star(u | x_{<j})\,du.
\end{equation*}
\looseness=-1
Also, define the true conditional density restricted to a bin:
\begin{equation*}
p^\star \left(\tilde{x}_j | \tilde{x}_j \in b_l, x_{<j}\right)
:=
\frac{p^\star(\tilde{x}_j | x_{<j})}{p^\star(b_l |  x_{<j})}\,\mathbf{1}\{\tilde{x}_j\in b_l\}.
\end{equation*}
\looseness=-1
The lowest attainable error of $p_{\theta}(x_j|x_{<j})$ as measured by the KL-divergence, is given by the following proposition.

\begin{mdframed}[style=MyFrame2]
\begin{restatable}{proposition}{propklerror}
\label{prop:min_error_in_kl}
Let the true conditional density be given by $p^*(x_j | x_{<j})$ and the autoregressive model's conditional density $p_{\theta}(x_j | x_{<j})$ under the uniform bin parameterization. The resulting minimum achievable error of the autoregressive model in KL is,
\begin{align*}
&\inf_\theta \,
\KL\left(p^*(x_j | x_{<j})\Vert p_{\theta}(x_j | x_{<j})\right)
=\proplinebreak
\sum_{l=1}^L p^\star(b_l |  x_{<j})
\KL\left(p^\star \left(\tilde{x}_j | \tilde{x}_j \in b_l, x_{<j}\right)\Vert\mathrm{Unif}(\Delta)\right).
\end{align*}
\end{restatable}
\end{mdframed}
\looseness=-1
\Cref{prop:min_error_in_kl} quantifies the intrinsic loss of modelling resolution induced by piecewise-uniform de-quantization within bins. Intuitively, it explicates that the mismatch corresponds to precisely the true density within each bin not being exactly uniform with width $\Delta$, which is precisely the irreducible error of this AR model parameterization. Moreover, we highlight the KL term is merely the \emph{Shannon entropy} of $p^\star \left(\tilde{x}_j | \tilde{x}_j \in b_l, x_{<j}\right)$ up to a constant as we are comparing it against the uniform distribution $\mathrm{Unif}(\Delta)$.
In~\cref{fig:bin-count-by-coords} and~\cref{fig:bin-count-analysis} in the Appendix, we quantify the impact of uniform binning for single peptide systems we consider by analyzing the distribution of coordinates in each bin index as well as irreducible error due to this distributional mismatch. Importantly, choosing the appropriate number of bins \emph{still allows us} to perform effective resampling through SNIS as we demonstrate in our experiments~\S\ref{sec:experiments}.

\subsection{Tools Enabled by \nameshort} 
\label{sec:tools}
\looseness=-1
\nameshort's autoregressive factorization unlocks a large toolkit of inference-time interventions unavailable to flow-based approaches that generate all data dimensions concurrently. By decomposing generation into discrete steps over each conditional, we can leverage standard techniques from modern LLMs, such as temperature scaling for diversity control, and extend them to the molecular domain. Crucially, this autoregressive structure admits intermediate intervention;\ we can analyze, correct, or discard partial conformations before the full molecule is realized. We proceed to demonstrate this capability via a granular resampling scheme.

\begin{algorithm}[htb]
    \caption{Autoregressive SMC Inference}
    \label{alg:smc_inference}
\begin{algorithmic}[1]
    \REQUIRE Pre-trained proposal $p_{\theta}$, batch size $M$, and energy for $s$-length capped residues $\gE_s$
    \STATE $w_0 \gets 1$ 
    \FOR{$j$ in $1, \ldots, d$}
    \FOR{$m$ in $1, \ldots, M$}
    \STATE $x_j^m \sim p_\theta(x_j^m | x_{<j}^m)$
    \STATE $w_j^m \gets w_{j-1}^m \left(\frac{\psi_j(x^m_j)}{\psi_{j-1}(x^m_{j-1})}\right)$
    \ENDFOR
    \IF{$\textsc{is\_residue\_end}(j)$}
        \STATE $x_{\le j} \gets \textsc{Systematic-Resample}(x_{\le j}, w_j)$
        \STATE $w_j \gets 1$
    \ENDIF
    \ENDFOR
\end{algorithmic}

\end{algorithm}

\looseness=-1
\xhdr{Autoregressive Twisted Sequential Monte Carlo} SMC sampling has been a staple tool across varied domains~\citep{doucet2001sequential,del2006sequential}. In \nameshort, we can improve efficiency by early exiting the sampling process on physically implausible substructures (e.g., steric clashes). We implement this via an Autoregressive Twisted SMC. We define a series of \textit{twist} functions $\psi_j(x_j)$ which define a number of intermediate densities $\eta_j(x_j) := p_\theta(x_j | x_{<j}) \psi_j(x_j)$. Instead of sampling directly from $p_\theta(x_j | x_{<j})$ at every step, the goal is to sample from the twisted intermediate distribution $\eta_j$. This can be accomplished by first calculating the likelihoods and the twist function values, then performing resampling (\cref{alg:smc_inference}).
\looseness=-1
In our molecular case, we define the twist functions using partial energy evaluations:
\begin{equation}
\label{eqn:twist_smc}
\psi_j(x_j) = \frac{\exp(-\gE_s \left(r_s(x_{\leq j})\right))}{p_\theta\left(r_s(x_{\leq j}) | r_{s-1}(x_{\le j})\right)},
\end{equation}
\looseness=-1
where $r(x_{\leq j})$ defines the largest (capped) residue subset of $x_{\leq j}$, inclusive of $x_j$ (c.f.~\S\ref{sec:artsmc} for details on capped residues).
This twist function allows us to inject physical validity constraints using any molecular energy function $\gE_s(\cdot)$, a version of the original $\gE(\cdot)$, that operates on peptides of length $s<d$. While resampling can be done at any point, we tailor our SMC to the molecular setting by invoking systematic resampling~\citep{douc2005comparison}, at the end of the atomistic coordinates of a generated residue. We use residue-level granularity to leverage standard peptide force-fields, though finer atom-level twists are also possible and allow for earlier rejection. This approach is fundamentally distinct from prior flow-based resampling methods like SBG~\citep{tan_scalable_2025}. While SBG must generate complete candidates before resampling, meanwhile, \nameshort enables substructure-level steering, correcting the generative process as soon as an error is detected.

\cut{
While we do not explore other twist functions in this work, we note that learned twist functions (as well as proposals)~\citep{zhao2024probabilisticinferencelanguagemodels} are likely to improve performance at the cost of additional complexity and present an intriguing direction for future work, only possible with autoregressive BGs.
}

\begin{table*}[thb]
\centering
\caption{Tri-alanine (AL3), Alanine tetrapeptide (AL4), Hexa-alanine (AL6), and Chignolin (GYDPETGTWG) results. Evaluations are performed using $2 \times 10^5$ energy evaluations;\ all methods except SBG use SNIS. Best values are in \textbf{bold}, with second-best \underline{underlined}. Split into flow-based (top), single-pass autoregressive with temperature one (middle), and tuned temperature settings (bottom).}
\label{tab:single-system-main}
\vspace{-5pt}
\resizebox{\textwidth}{!}{
\begin{tabular}{@{}lcccccccc}
    \toprule
    & \multicolumn{2}{c}{Tri-alanine (AL3)} & \multicolumn{2}{c}{Tetrapeptide (AL4)} & \multicolumn{2}{c}{Hexa-alanine (AL6)} & \multicolumn{2}{c}{Chignolin (GYDPETGTWG)}  \\
    \cmidrule(lr){2-3}\cmidrule(lr){4-5}\cmidrule(lr){6-7}\cmidrule(lr){8-9}
    Algorithm & $\mathcal{E}$-$\mathcal{W}_2$ $\downarrow$ & $\mathbb{T}$-$\mathcal{W}_2$ $\downarrow$ 
    & $\mathcal{E}$-$\mathcal{W}_2$ $\downarrow$ & $\mathbb{T}$-$\mathcal{W}_2$ $\downarrow$
    & $\mathcal{E}$-$\mathcal{W}_2$ $\downarrow$ & $\mathbb{T}$-$\mathcal{W}_2$ $\downarrow$
    & $\mathcal{E}$-$\mathcal{W}_2$ $\downarrow$ & $\mathbb{T}$-$\mathcal{W}_2$ $\downarrow$ \\
    \midrule
    ECNF++ & 2.206 $\pm$ 0.813 & 0.962 $\pm$ 0.253 
           & 5.638 $\pm$ 0.483 & 1.002 $\pm$ 0.061
           & 10.668 $\pm$ 0.285& 1.902 $\pm$ 0.055
           & --- & ---\\
    RegFlow & 0.853 $\pm$ 0.105 & 1.577 $\pm$ 0.140 & 3.277 $\pm$ 0.546 & 2.342 $\pm$ 0.102  & --- & --- & --- & ---\\
    SBG  & {0.598 $\pm$ 0.084} & 0.503 $\pm$ 0.029
            & {1.007 $\pm$ 0.382} & 1.039 $\pm$ 0.069
            &  1.189 $\pm$ 0.357 & 1.444 $\pm$ 0.140 
            & 10.819 $\pm$ 7.206 & 3.778 $\pm$ 0.440 \\
FALCON-A & 1.385 $\pm$ 0.182 & 0.343 $\pm$ 0.004
         & 2.929 $\pm$ 0.068 & 1.094 $\pm$ 0.034
         & 1.211 $\pm$ 0.105 & 1.163 $\pm$ 0.112
         & --- & --- \\
FALCON   & 0.544 $\pm$ 0.013 & 0.452 $\pm$ 0.011
         & \underline{0.686 $\pm$ 0.047} & 0.858 $\pm$ 0.077
         & 0.892 $\pm$ 0.311 & 1.256 $\pm$ 0.132
         & --- & --- \\
\midrule
GIVT & 1.354 $\pm$ 0.058 & 0.343 $\pm$ 0.008 & 1.033 $\pm$ 0.449 & 1.113 $\pm$ 0.100 & 1.206 $\pm$ 0.056 & 1.527 $\pm$ 0.048 & 45.646 $\pm$ 20.989 & 3.031 $\pm$ 0.098 \\
MoL-PixelCNN++ & 0.506 $\pm$ 0.082  & 1.024 $\pm$ 0.686 & 1.643 $\pm$ 0.504 & 1.415 $\pm$ 0.110 & 1.429 $\pm$ 0.186	& 1.264 $\pm$ 0.205 & 140.717 $\pm$ 49.113 & 3.391 $\pm$ 0.093 \\
GMM-PixelCNN++ & \underline{0.249 $\pm$ 0.025} & 0.364 $\pm$ 0.016 & 1.434 $\pm$ 0.783 & 0.806 $\pm$ 0.056 & 1.164 $\pm$ 0.037 & 1.285 $\pm$ 0.058 & 23.339 $\pm$ 6.485 & 3.007 $\pm$ 0.086\\

\nameshort (ours) ($T = 1$)
& 0.271 $\pm$ 0.113 & \textbf{0.311 $\pm$ 0.009}
& 0.886 $\pm$ 0.076 & \textbf{0.593 $\pm$ 0.008}
& \underline{0.722 $\pm$ 0.063} & \textbf{1.085 $\pm$ 0.069}
& \underline{7.942 $\pm$ 1.053} & \underline{2.780 $\pm$ 0.105}\\
\midrule
\nameshort (ours) (tuned $T$)
& \textbf{0.202 $\pm$ 0.010} & \textbf{0.312 $\pm$ 0.003}
& \textbf{0.449 $\pm$ 0.030} & \textbf{0.592 $\pm$ 0.010}
& \textbf{0.328 $\pm$ 0.122} & \textbf{1.094 $\pm$ 0.052} 
& \textbf{1.723 $\pm$ 0.075} & \textbf{2.632 $\pm$ 0.044}\\
    \bottomrule
\end{tabular}}
\end{table*}

\section{Experiments}
\label{sec:experiments}
\looseness=-1
In this section, we empirically validate the efficacy of \namelong across a variety of molecular conformation sampling tasks. Our evaluation focuses on assessing the framework's scalability on single peptides systems ranging from alanine dipeptide to the 10-residue Chignolin in the same experimental data configuration of~\citet{tan_scalable_2025}. We also evaluate the zero-shot generalization capabilities on unseen sequences using our transferable model, \namemodel, using the ManyPeptidesMD dataset introduced in~\citep{tan2025amortizedsamplingtransferablenormalizing}. 

\looseness=-1
\xhdr{Baselines}
We consider a suite of prior baselines that include the equivariant CNF \citep{klein2023equivariant,klein_transferable_2024}, and an improved version:\ ECNF++ from \citet{tan_scalable_2025}. We additionally compare against discrete normalizing flows, in RegFlow~\citep{rehman2025fortforwardonlyregressiontraining}, and the prior state-of-the-art method for single-system Boltzmann sampling SBG~\citep{tan_scalable_2025}. We also benchmark performance relative to few-step CNFs, i.e., flow-maps~\citep{boffi2025buildconsistencymodellearning,geng2025mean}:\ FALCON/FALCON-A, which differ in their training objectives~\citep{rehman2025falconfewstepaccuratelikelihoods}. For ALDP, we further include BoltzNCE, an energy-based model trained via noise-contrastive estimation~\citep{aggarwal2025boltznce}. We further train a GIVT \citep{tschannen2024givtgenerativeinfinitevocabularytransformers}, MoL-PixelCNN++, and GMM-PixelCNN++ as described in~\cref{sec:model_instantiation}, following the same training procedure as \nameshort. For the transferable setting, we compare against Timewarp~\citep{klein2023timewarp}, BioEmu~\citep{lewis2025scalable}, UniSim~\citep{yu2025unisimunifiedsimulatortimecoarsened}, TarFlow \citep{zhai2024normalizing}, and the prior SOTA for transferable Boltzmann generation in Prose~\citep{tan2025amortizedsamplingtransferablenormalizing}.

\looseness=-1
\xhdr{Metrics}
We evaluate our models with three complementary Wasserstein-based metrics following prior work. (\textbf{1}) The 2-Wasserstein energy distance (\emetric), which measures agreement between generated and reference energy distributions, providing a sensitive test of local physical accuracy and consistency with the target Boltzmann distribution. (\textbf{2}) To assess structural mode coverage, we compute a 2-Wasserstein distance in torsional space (\torusmetric) that respects angular periodicity, capturing global conformational differences and missing modes that may not be reflected in energies alone. While this torus-based metric is effective at detecting large-scale structural mismatches, it can be insensitive to rare mode loss. (\textbf{3}) Lastly, we analyze a time-lagged independent component analysis (TICA)-based 2-Wasserstein distance (\ticametric) that compares samples in a lower-dimensional space spanned by the slowest dynamical modes fit on a reference trajectory. We exclude Effective Sample Size (ESS)~\citep{kish1957confidence} as a primary metric, as it is incompatible with SMC-based schemes while also disproportionately rewarding ``mode-seeking'' models that collapse into single energy minima to minimize variance~\citep{blessing2024elboslargescaleevaluationvariational}. We instead prioritize metrics that penalize mode-dropping to ensure accurate global distributional coverage (see~\S\ref{app:ess_critique} for a detailed discussion).

\subsection{Single Peptide Systems}
\looseness=-1
We evaluate the performance of \shortname on conformation sampling tasks for single peptides, ranging from the simple alanine dipeptide (ALDP) (2 residues) to the large Chignolin (10 residues). We report our results in~\cref{tab:single-system-main}, and defer the ALDP results to~\cref{tab:aldp_results} in~\S\ref{sec:alaninedipeptide_results} due to both the simplicity of the dataset and also problems with the dataset construction leading to mode-collapse of models. We find that \nameshort comprehensively outperforms on both \emetric and \torusmetric for all considered systems when the temperature is tuned on a validation set. Without temperature tuning ($T=1$) \nameshort is still the best performing method overall, especially when scaled larger systems (AL6 and Chignolin), but has slightly worse \emetric on smaller systems. We further observe MDN baselines like GMM-PixelCNN++ and GIVT achieve slightly worse but very promising results, demonstrating the overall potential of AR models for BG's in comparison to flow-based BGs.

\begin{figure*}[t]
  \centering
  \includegraphics[width=\textwidth]{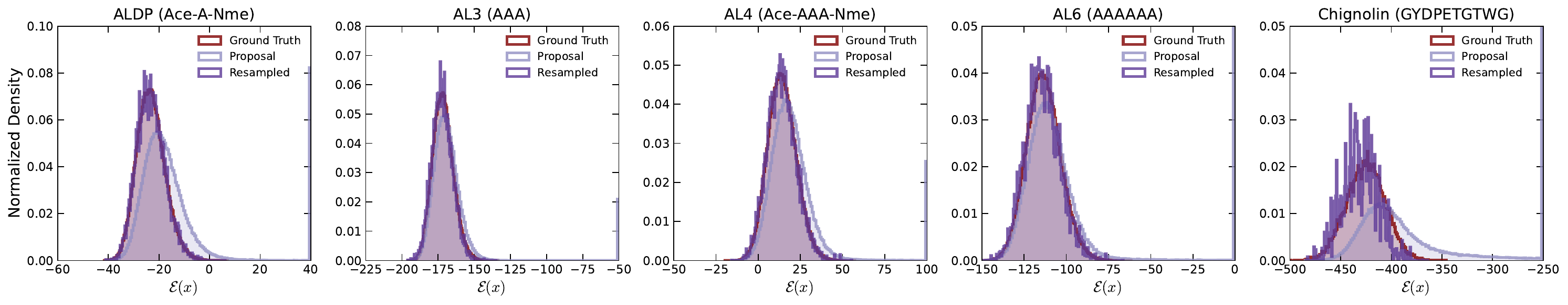}
  \vspace{-15pt}
   \caption{Energy histogram of the proposal and re-sampled distribution compared to ground truth MD data for all single peptide systems. }
  \label{fig:energy-distribution-all-molecules}
\end{figure*}

\begin{figure*}[t]
  \centering
  \includegraphics[width=\textwidth]{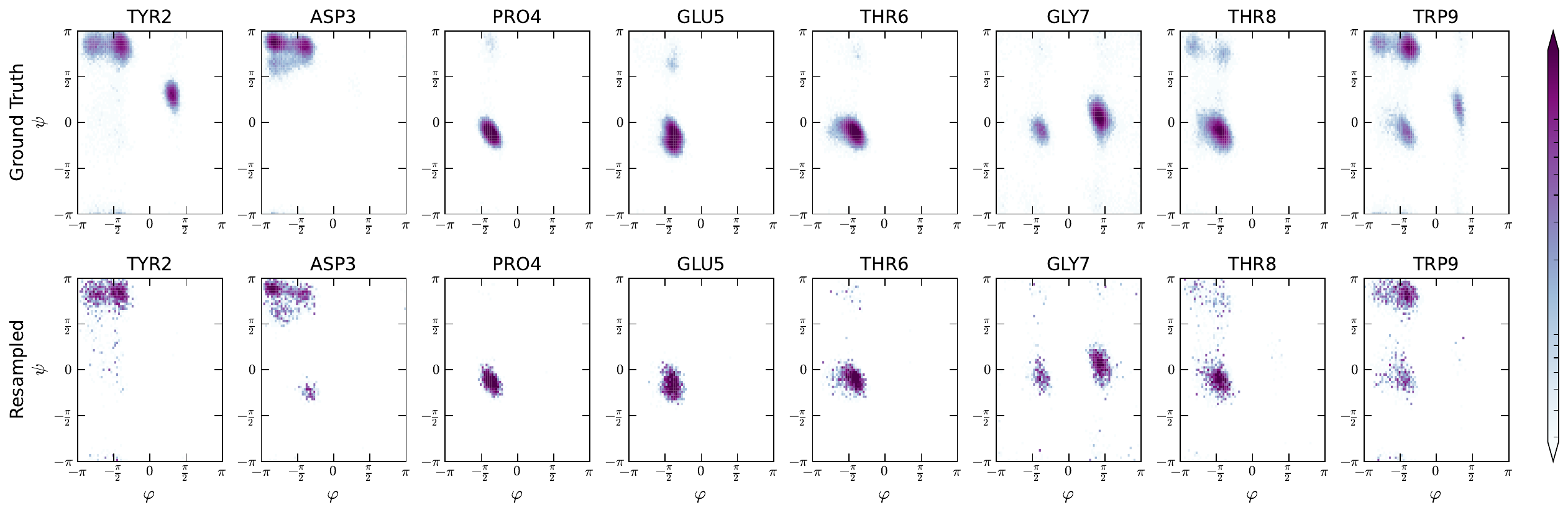}
  \vspace{-20pt}
   \caption{Ramachandran plots for Chignolin (top row:\ ground truth test set;\ bottom row:\ \shortname's predictions). }
  \label{fig:chignolin-rama}
\end{figure*}

\looseness=-1 
\xhdr{Scaling to Decapeptides}
SBG~\citep{tan_scalable_2025} first demonstrated that discrete NFs can be scaled to the decapeptide Chignolin---a particularly challenging molecular system due to the existence of the $\beta$-hairpin secondary structure. As shown in \cref{tab:single-system-main}, \nameshort significantly outperforms the SBG approach across all global and local evaluation metrics, further validating its scalability. Additionally, \cref{fig:energy-distribution-all-molecules} shows that the reweighted energy distribution proposed by our model closely matches that of MD simulation data. Finally, the Ramachandran plots~\citep{ramachandran1963stereochemistry} in~\cref{fig:chignolin-rama} indicate that our model accurately captures nearly all the conformational modes present in the test set. 

\begin{table}[ht]
\caption{Quantitative results across peptides of length 4 and 8. All methods evaluated a budget of $10^4$ energy evaluations (top) or $2 \times 10^5$ (bottom). Best values in \textbf{bold}, with second-best \underline{underlined}. }\centering
\label{tab:transferable-main}
\resizebox{1\linewidth}{!}{
\begin{tabular}{@{}lcccccccc}
    \toprule
    \# Residues $\rightarrow$ & \multicolumn3c{4AA \tiny{(30 systems)}} & \multicolumn3c{8AA \tiny{(30 systems)}}  \\
    \cmidrule(lr){2-4}\cmidrule(lr){5-7}
    Model $\downarrow$ & \emetric  & \torusmetric  & \ticametric  &  \emetric & \torusmetric   & \ticametric  \\
    \midrule
    TimeWarp&   7.237 & 2.204 & 0.993 &  --- & --- & --- \\
    BioEmu &  90.079 & 2.037 & 1.479 & 193.873 & 4.638 & 1.601 \\
    UniSim &  $>10^{4}$ & 2.766 & 1.733 & $>10^{3}$ & 6.156 & 1.495 \\
    \midrule
    ECNF++ & 10.032 & 1.121 & 0.572 & --- & --- & --- \\
    TarFlow &  1.260 & 0.924 & 0.492 & 11.298 & 2.733 & 1.087 \\
    Prose & \textbf{0.932} & \textbf{0.752} & \textbf{0.367} & 10.038 & 2.456 & 0.988 \\
    \namemodel (ours) & 1.168 & 0.886 & 0.471 & \textbf{4.251} & \underline{2.325} &  \textbf{0.943} \\
    \namemodel (ours) SMC & \underline{1.079} & \underline{0.874} & \underline{0.463} & \underline{4.263} & \textbf{2.315} & \underline{0.977} \\
    \midrule
    $2 \times 10^5$ evaluations \\
    TarFlow & 0.929 & 0.776 & 0.498 & 10.826 & 2.320 & 1.057 \\
    Prose &  \underline{0.646} & \textbf{0.607} & \textbf{0.349} & \underline{9.360} & \underline{2.019} & \underline{0.960} \\
  \namemodel (ours) & \textbf{0.531} & \underline{0.649} & \underline{0.379} & \textbf{3.615} & \textbf{1.902} & \textbf{0.882}\\
    \bottomrule
    \end{tabular}
    }
\end{table}

\subsection{Transferable Generation}

\looseness=-1
We now introduce \namemodel, a transferable model trained using the \nameshort framework with additional conditioning information---detailed in~\S\ref{app:arch_conditioning}---to allow zero-shot transfer to unseen peptides in the ManyPeptidesMD dataset~\cite{tan2025amortizedsamplingtransferablenormalizing}. We report our results in~\cref{tab:transferable-main}, which contain test set performance metrics averaged over 30 different sequences of length 4 and 8 residues. Empirically, we observe superior performance over the current state-of-the-art method (Prose) on all molecular systems of size 8 with competitive performance on sequences of size 4. 

\xhdr{Twisted SMC} We observe that Twisted SMC outperforms SNIS only marginally. We attribute this to the high quality of the base \namemodel proposal. Since the model has learned the Boltzmann distribution sufficiently well, the SMC correction yields diminishing returns;\ however, the twisted SMC approach confirms that \nameshort is amenable to intermediate steering, paving the way for increased efficiency using early rejection, applying constraints, or guidance in larger or more complex systems where the proposal is less accurate. We find that on 8AA peptides, around 7\% of samples have a final partial energy greater than $100+ \gE_{\min}$, with a 10,000 sample batch, where $\gE_{\min}$ is the empirical minimum energy obtained. For larger systems where the proposal is not as efficient, see~\cref{sec:experiments} for additional analysis.

\looseness=-1
\xhdr{Inference Scaling}
We also investigate the scaling behaviour of inference samples relative to molecular dynamics (MD) and Prose on 8AA in~\cref{fig:scaling} and~\cref{fig:scaling-appendix}. We find that \namemodel performs favourably against Prose and molecular dynamics \textit{achieving the same performance with an order of magnitude fewer samples vs.\ Prose and three orders of magnitude vs.\ MD}  in terms of \torusmetric. Furthermore, \namemodel also outperforms Prose for the same computational budget (\cref{fig:scaling-appendix}, despite operating on dimensions instead of atoms. We provide further analysis and also investigate the non-monotonic behaviour of the \ticametric for MD in \S\ref{sec:scaling}. Finally, we also provide TICA plots which demonstrate the strong zero-shot performance of \namemodel on an unseen octapeptide compared to ground truth MD in~\cref{fig:CGSWHKQR-tica}.

\subsection{Ablations}

\looseness=-1
\xhdr{Performance with Bin Resolution}
As we increase the number of bins used in \nameshort, the granularity of the coordinates generated by the model increases. In~\S\ref{sec:data_preprocessing}, we first discretize the coordinates into a fixed number of bins, then use uniformly sampled noise from each bin to reconstruct these molecules, demonstrating an upper bound on performance for a model with a fixed number of bins. In \cref{fig:temp-ablations-main} and \cref{fig:temperature-ablations}, for the tri-alanine (AL3) single peptide system, we ablate the number of bins, and empirically validate that an increasing bin count monotonically improves performance on the resampled \emetric. 

\begin{figure}[h]
  \centering
  \includegraphics[width=0.49\textwidth]{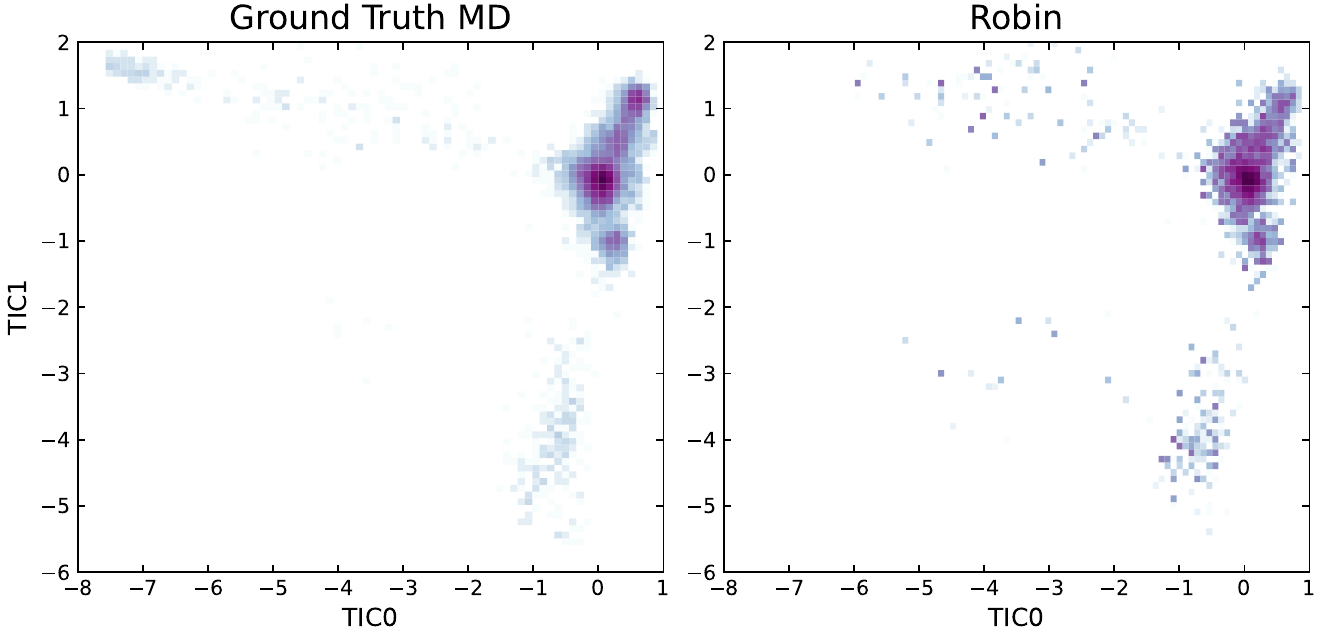}
   \caption{TICA plots comparing the true MD distribution against predictions from \namemodel for the octapeptide:\ \texttt{CGSWHKQR}.}
  \label{fig:CGSWHKQR-tica}
  \vspace{-10pt}
\end{figure}

\begin{figure}[!thb]
  \centering
  \includegraphics[width=0.49\textwidth]{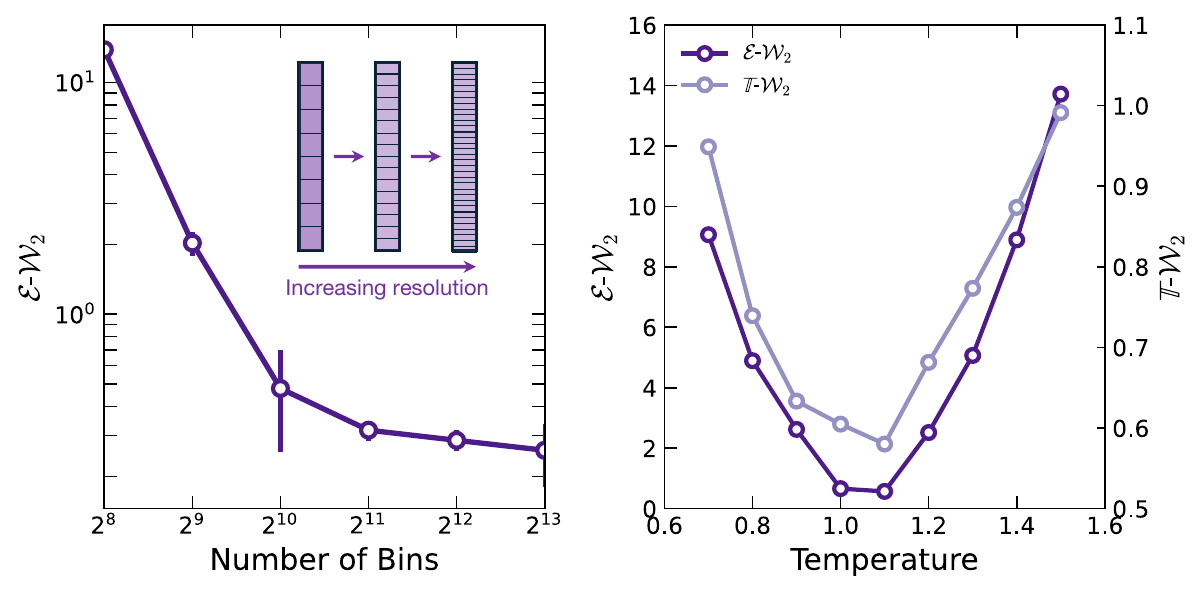}
   \caption{\textbf{Left}: Performance with increasing bin count for AL3. \textbf{Right}: Optimal temperature on AL3 across \emetric and \torusmetric.}
  \label{fig:temp-ablations-main}
  \vspace{-5pt}
\end{figure}

\looseness=-1
\xhdr{Sampling Temperature}
We perform inference-time ablations over the transformer's sampling temperature to identify conditions that yield optimal generative performance. As shown in~\cref{fig:temperature-ablations}, the energy distribution varies systematically with temperature:\ at low temperatures, probability mass concentrates on high-likelihood (low-energy) modes, which can suppress or entirely miss other modes, while at higher temperatures the distribution flattens, encouraging diversity but potentially over-sampling modes and degrading generative quality. We quantify this trade-off using our resampled metrics via a temperature sweep, which reveals an optimal sampling temperature near $T = 1.02$ for alanine tetrapeptide. We perform equivalent temperature sweeps for all other settings and summarize the optimal temperatures in~\S\ref{sec:optimal_temps}, finding that lower temperatures improve performance on larger molecular systems. 

\cut{
\looseness=-1
\xhdr{De-quantization Strategies}
To convert bins into continuous coordinates, we investigate three de-quantization techniques:\ (1) adding uniform noise to the chosen bin;\ (2) taking the midpoint of the bin;\ and (3) using the training data distribution to determine the most likely position within the bin of a given equilibrium sample. All approaches converge to the same solution as the bin width tends to 0. We found in practice that (3) improves performance marginally on single-peptide systems and (2) slightly improves performance on transferable systems on our metrics.
}

\section{Related Work}

\looseness=-1
\xhdr{Boltzmann Generators} The use of deep generative models in equilibrium sampling was popularized with the introduction of Boltzmann Generators (BGs)~\citep{noe2019boltzmann}. Most subsequent work has focused on refining the NF architecture used in BGs to improve expressivity~\citep{zhai2024normalizing,draxler2024free}, stability~\citep{schopmans_temperature-annealed_2025,vonklitzing2025learningboltzmanngeneratorsconstrained}, and generalization~\citep{klein_transferable_2024}. CNFs offer superior expressivity and easy handling of symmetries~\citep{kohler2020equivariant,klein2023equivariant}, but incur a high computational cost for likelihood evaluation during inference, which is only partially ameliorated with approximate few-step models~\citep{rehman2025falconfewstepaccuratelikelihoods}, architectural constraints~\citep{gloy2025hollowflowefficientsamplelikelihood}, or learned energy functions~\citep{aggarwal2025boltznce,akhoundsadegh2025progressiveinferencetimeannealingdiffusion}.

\looseness=-1
\xhdr{Autoregressive Models on Continuous Spaces} Several works have investigated transformer-based models for generating molecules, including latent space~\citep{murtada2025mdllm1largelanguagemodel} and flow-based~\citep{cheng2025scalableautoregressive3dmolecule,nextstepteam2025nextstep1autoregressiveimagegeneration} approaches. While it is possible to design an autoregressive SE(3) invariant architecture~\citep{gebauer_gschnet_2019}, these have not scaled as well as purely transformer-based methods. Mixture Density Networks~\citep{bishop1994mixture} have been explored for decades as a way to parametrize outputs over continuous spaces while providing exact densities~\citep{razavi2020frmdn}. These have been employed for generating sketches~\citep{ha2017neural}, images~\citep{salimans2017pixelcnn++}, world models~\citep{ha2018world}, and demand forecasting~\citep{li2023xrmdn}, to name a few. AR MDNs have also recently been reintroduced~\citep{tschannen2024givtgenerativeinfinitevocabularytransformers,li2024autoregressiveimagegenerationvector,Billera2024};\ however, all these methods focus on proposal quality, rather than interrogating the use of likelihoods for Boltzmann Generation. Moreover, MDNs have historically been prone to dropping modes in their mixtures~\citep{deng2022deep}, leading to suboptimal performance---as also observed for molecules in~\cref{tab:single-system-main}.

\section{Conclusion}
\looseness=-1
In this work, we introduced \namelong (\nameshort), a novel autoregressive model for Boltzmann Generators that offers a promising alternative to the dominant paradigm of flow-based approaches. In particular, \nameshort offers new tools that circumvent the expressivity and efficiency constraints that hinder discrete flow-based architectures while being more computationally efficient at likelihood estimation than CNFs. Importantly, \nameshort enjoys the same toolkit available to LLMs that comes with feature rich optimizations that enable scaling laws for language, token level-steering, which we demonstrated in the molecular setting for the first time within a Boltzmann Generator framework.

\looseness=-1
\xhdr{Limitations}
While \nameshort enables a drastically different approach to BGs, it comes with a few notable limitations. Firstly, AR models impose a specific ordering over dimensions, while molecules themselves do not possess a natural ordering and as a result, this choice may affect performance, e.g., in small molecules~\citep{cheng2025scalableautoregressive3dmolecule}. Secondly, the use of uniform binning bounds the precision of the model by $\Delta$, which may pose challenges on even larger systems with sharper energy profiles. Finally, flow-based BGs can benefit from the use of informative priors, such as in TFEP~\citep{wirnsberger2020targeted};\ an investigation of which for AR models we leave as a direction for future work.

\section*{Impact Statement}
This paper presents work whose goal is to advance the field of machine learning for scientific applications. There are many potential societal consequences of our work, none of which we feel must be specifically highlighted here.

\section*{Acknowledgements}
The authors would like to thank Benjamin Murrell for planting the seeds of this idea, as well as Luka Mucko, Tolga Birdal, and Matthew Wicker for feedback on an early draft of this work. Danyal Rehman received financial support from the Natural Sciences and Engineering Research Council's (NSERC) Banting Postdoctoral Fellowship under Funding Reference No.\ 198506. The authors acknowledge funding from UNIQUE, CIFAR, NSERC, Intel, and Samsung. The research was enabled in part by computational resources provided by the Digital Research Alliance of Canada (\url{https://alliancecan.ca}), Mila (\url{https://mila.quebec}), Aithyra (\url{https://www.oeaw.ac.at/aithyra}), and NVIDIA.

\bibliography{main}
\bibliographystyle{icml2026}

\newpage\clearpage
\appendix
\onecolumn
\beginsupplement
\section*{Appendix}

\section{Theory}
\label{app:theory}

\begin{mdframed}[style=MyFrame2]
\renewcommand{\proplinebreak}{}%
\propklerror*
\end{mdframed}
\begin{proof}
Fix a dimension $j$ and a context $x_{<j}$. Let $\{b_l\}_{l=1}^L$ be a measurable partition of the
domain of $\tilde{x}_j$ into disjoint bins, each with width $|b_l|=\Delta$, i.e.\
$\bigcup_{l=1}^L b_l$ covers the support of interest and $b_l \cap b_{l'}=\emptyset$ for $l\neq l'$.

Under uniform binning, the autoregressive model's conditional density
is constrained to be piecewise-uniform on bins. Concretely, there exist bin probabilities
$\pi_\theta(b_l | x_{<j}) \ge 0$ with $\sum_{l=1}^L \pi_\theta(b_l | x_{<j})=1$ such that
\begin{equation}
\label{eq:model_piecewise_uniform}
p_\theta(\tilde{x}_j | x_{<j})
=
\sum_{l=1}^L \pi_\theta(b_l | x_{<j}) \, \mathrm{Unif}(\Delta)
=
\sum_{l=1}^L \pi_\theta(b_l | x_{<j}) \, \frac{\mathbf{1}\{\tilde{x}_j \in b_l\}}{\Delta}.
\end{equation}
Equivalently, for $\tilde{x}_j \in b_l$,
\begin{equation}
\label{eq:model_inside_bin}
p_\theta(\tilde{x}_j | x_{<j}) = \frac{\pi_\theta(b_l | x_{<j})}{\Delta}.
\end{equation}

\looseness=-1
By the definition of the true conditionals we have,
\begin{equation}
\KL\!\left(p^\star(\tilde{x}_j | x_{<j}) \,\Vert\, p_\theta(\tilde{x}_j | x_{<j})\right)
=
\int p^\star(\tilde{x}_j | x_{<j})
\log \left(\frac{p^\star(\tilde{x}_j | x_{<j})}{p_\theta(\tilde{x}_j | x_{<j})}\right)\,d\tilde{x}_j.
\end{equation}
\looseness=-1
Since the bins form a disjoint partition, we can split the integral:
\begin{align}
\label{eq:split_bins}
\KL\!\left(p^\star \Vert p_\theta\right)
&=
\sum_{l=1}^L \int_{b_l}
p^\star(\tilde{x}_j | x_{<j})
\log \left(\frac{p^\star(\tilde{x}_j | x_{<j})}{p_\theta(\tilde{x}_j | x_{<j})}\right)\,d\tilde{x}_j.
\end{align}
Now use \eqref{eq:model_inside_bin} inside each bin:
for $\tilde{x}_j\in b_l$, $p_\theta(\tilde{x}_j| x_{<j})=\pi_\theta(b_l| x_{<j})/\Delta$. Therefore
\begin{align}
\label{eq:inside_bin_log_ratio}
\log\left(\frac{p^\star(\tilde{x}_j | x_{<j})}{p_\theta(\tilde{x}_j | x_{<j})}\right)
=
\log p^\star(\tilde{x}_j | x_{<j})
-
\log \pi_\theta(b_l | x_{<j})
+
\log \Delta.
\end{align}
Plugging \eqref{eq:inside_bin_log_ratio} into \eqref{eq:split_bins} yields
\begin{align}
\KL\!\left(p^\star \Vert p_\theta\right)
&=
\sum_{l=1}^L
\int_{b_l} p^\star(\tilde{x}_j | x_{<j})
\left(\log p^\star(\tilde{x}_j | x_{<j}) - \log \pi_\theta(b_l | x_{<j}) + \log \Delta\right)\,d\tilde{x}_j \nonumber \\
&=
\sum_{l=1}^L
\int_{b_l} p^\star(\tilde{x}_j | x_{<j}) \log p^\star(\tilde{x}_j | x_{<j})\,d\tilde{x}_j
\;-\;
\sum_{l=1}^L p^\star(b_l| x_{<j})\log \pi_\theta(b_l | x_{<j})
\;+\;\log \Delta. 
\label{eq:kl_expanded_first}
\end{align}
(Here we used $\int_{b_l} p^\star(\tilde{x}_j | x_{<j})\,d\tilde{x}_j = p^\star(b_l| x_{<j})$
and $\sum_l p^\star(b_l| x_{<j})=1$.)

Now insert and subtract $\log p^\star(b_l| x_{<j})$ to isolate a discrete KL. We can then rewrite the term involving $\log \pi_\theta$ as follows:
\begin{equation}
-\sum_{l=1}^L p^\star(b_l| x_{<j})\log \pi_\theta(b_l| x_{<j})
=
-\sum_{l=1}^L p^\star(b_l| x_{<j})\log p^\star(b_l| x_{<j})
+\sum_{l=1}^L p^\star(b_l| x_{<j})\log \left(\frac{p^\star(b_l| x_{<j})}{\pi_\theta(b_l| x_{<j})}\right).
\end{equation}
The second sum is exactly the KL divergence between the true bin-mass distribution and the model bin distribution:
\begin{equation}
\label{eq:discrete_kl}
\KL\!\left(p^\star(b_l | x_{<j}) \Vert \pi_\theta(b_l| x_{<j})\right)
=
\sum_{l=1}^L p^\star(b_l| x_{<j})\log\left(\frac{p^\star(b_l| x_{<j})}{\pi_\theta(b_l| x_{<j})}\right).
\end{equation}
Substituting back into \eqref{eq:kl_expanded_first} gives
\begin{align}
\label{eq:kl_discrete_plus_shape}
\KL\!\left(p^\star \Vert p_\theta\right)
&=
\KL\!\left(p^\star(b_l | x_{<j}) \Vert \pi_\theta(b_l| x_{<j})\right)
+
\sum_{l=1}^L \left[
\int_{b_l} p^\star(\tilde{x}_j | x_{<j}) \log\left(\frac{p^\star(\tilde{x}_j | x_{<j})}{p^\star(b_l| x_{<j})/\Delta}\right)\,d\tilde{x}_j
\right].
\end{align}
To interpret the second term, observe that for $\tilde{x}_j\in b_l$,
\[
\frac{p^\star(b_l| x_{<j})}{\Delta} = p^\star(b_l| x_{<j}) \cdot \mathrm{Unif}(\Delta).
\]
Moreover, using the true conditional density,
\[
p^\star(\tilde{x}_j | \tilde{x}_j\in b_l, x_{<j})
=
\frac{p^\star(\tilde{x}_j | x_{<j})}{p^\star(b_l| x_{<j})}\mathbf{1}\{\tilde{x}_j\in b_l\}.
\]
Hence, we can rewrite the bracketed integral as
\begin{align}
\int_{b_l} p^\star(\tilde{x}_j | x_{<j})
\log \left(\frac{p^\star(\tilde{x}_j | x_{<j})}{p^\star(b_l| x_{<j})/\Delta}\right)\,d\tilde{x}_j
&=
p^\star(b_l| x_{<j})
\int_{b_l} p^\star(\tilde{x}_j | \tilde{x}_j\in b_l, x_{<j})
\log \left(\frac{p^\star(\tilde{x}_j | \tilde{x}_j\in b_l, x_{<j})}{\mathrm{Unif}(\Delta)}\right)\,d\tilde{x}_j \nonumber \\
&=
p^\star(b_l| x_{<j})
\KL\!\left(
p^\star(\tilde{x}_j | \tilde{x}_j\in b_l, x_{<j})
\Vert
\mathrm{Unif}(\Delta)
\right).
\label{eq:within_bin_kl}
\end{align}
Combining \eqref{eq:kl_discrete_plus_shape} and \eqref{eq:within_bin_kl}, we obtain the exact decomposition
\begin{equation}
\label{eq:exact_decomposition}
\KL\!\left(p^\star(\tilde{x}_j | x_{<j}) \Vert p_\theta(\tilde{x}_j | x_{<j})\right)
=
\KL\!\left(p^\star(b_l | x_{<j}) \Vert \pi_\theta(b_l| x_{<j})\right)
+
\sum_{l=1}^L p^\star(b_l | x_{<j})
\KL\!\left(
p^\star(\tilde{x}_j | \tilde{x}_j\in b_l, x_{<j})
\Vert
\mathrm{Unif}(\Delta)
\right).
\end{equation}

\looseness=-1
The second term in \eqref{eq:exact_decomposition} depends only on the true distribution $p^\star$ and the fixed bins,
and is independent of $\theta$. The first term is a KL divergence over discrete distributions, hence it is always nonnegative and is minimized if and only if
\[
\pi_\theta(b_l| x_{<j}) = p^\star(b_l| x_{<j})\quad \text{for all }l.
\]
At this optimum, the discrete KL equals zero, so the minimum achievable KL within the piecewise-uniform model family is
\[
\inf_\theta \KL\!\left(p^\star(\tilde{x}_j | x_{<j}) \Vert p_\theta(\tilde{x}_j | x_{<j})\right)
=
\sum_{l=1}^L p^\star(b_l | x_{<j})
\KL\!\left(
p^\star(\tilde{x}_j | \tilde{x}_j\in b_l, x_{<j})
\Vert
\mathrm{Unif}(\Delta)
\right).
\]
As a result, this matches the statement of the proposition.
\end{proof}

\section{Data Preprocessing and Analysis}
\label{sec:data_preprocessing}\looseness=-1
We analyze the coordinate distribution of each peptide by placing the training data into a fixed number of bins (set to \texttt{num\_bins} $= 1024$ for simplicity) in~\cref{fig:bin-count-by-coords}. 

\begin{figure*}[!h]
  \centering
  \includegraphics[width=\textwidth]{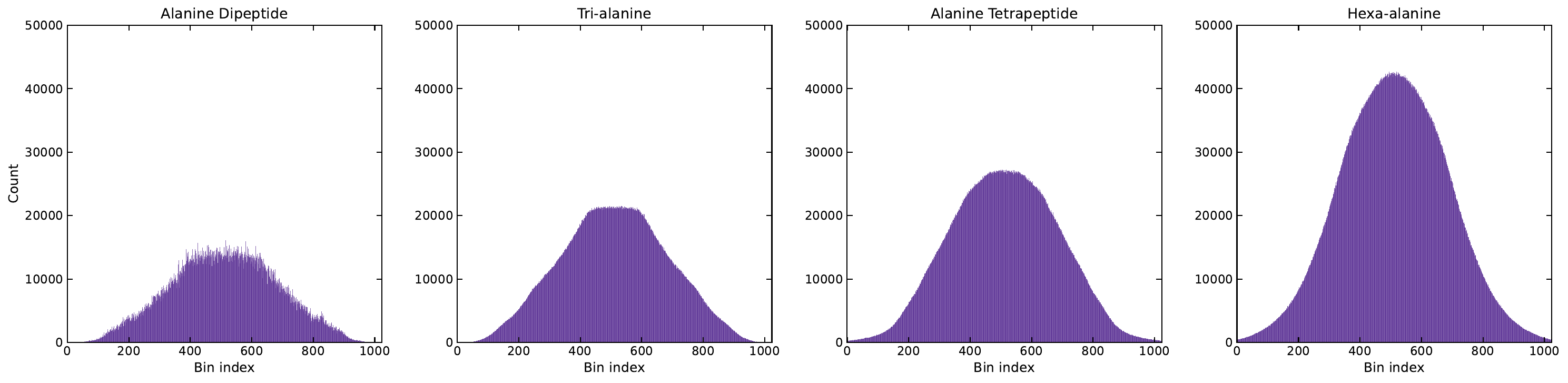}
  \caption{The distribution of coordinates across bins for fixed bin count with \texttt{num\_bins} = 1024.}
  \label{fig:bin-count-by-coords}
\end{figure*}

In addition to analyzing the distribution of samples across discretized bins, we evaluated a lower bound on \emetric and \torusmetric by de-quantizing via uniform noise injection (see~\cref{fig:bin-count-analysis}). Specifically, training data were discretized into bins and subsequently mapped back to continuous coordinates by reconstructing via uniformly sampled noise. This procedure was repeated over a range of bin sizes to characterize the bin-width-dependent lower bound on attainable performance across all alanine-based peptides. We further conducted ablations to assess the effect of bin size on \emetric and \torusmetric. As shown in~\cref{fig:bin-ablations}, performance exhibits a near-monotonic dependence on bin resolution, validating our original analysis. Based on these results, we used 4096 bins for alanine dipeptide and tri-alanine, and 8192 bins for larger molecules and \namemodel.

\begin{figure*}[!ht]
  \centering
  \includegraphics[width=\textwidth]{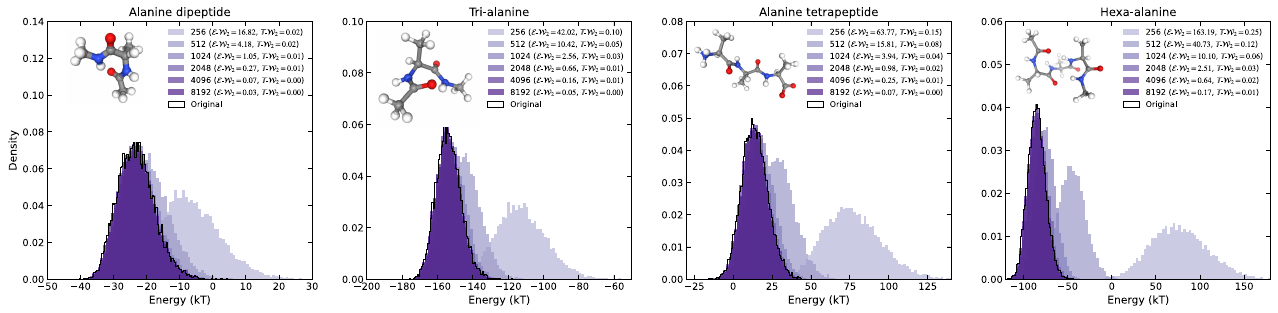}
  \caption{Energy distributions on the training data as a function of bin discretization. The corresponding \emetric and \torusmetric values are reported for each discretization, demonstrating an upper bound on the learnability of these metrics.}
  \label{fig:bin-count-analysis}
\end{figure*}

\begin{figure*}[!h]
  \centering
  \includegraphics[width=0.6\textwidth]{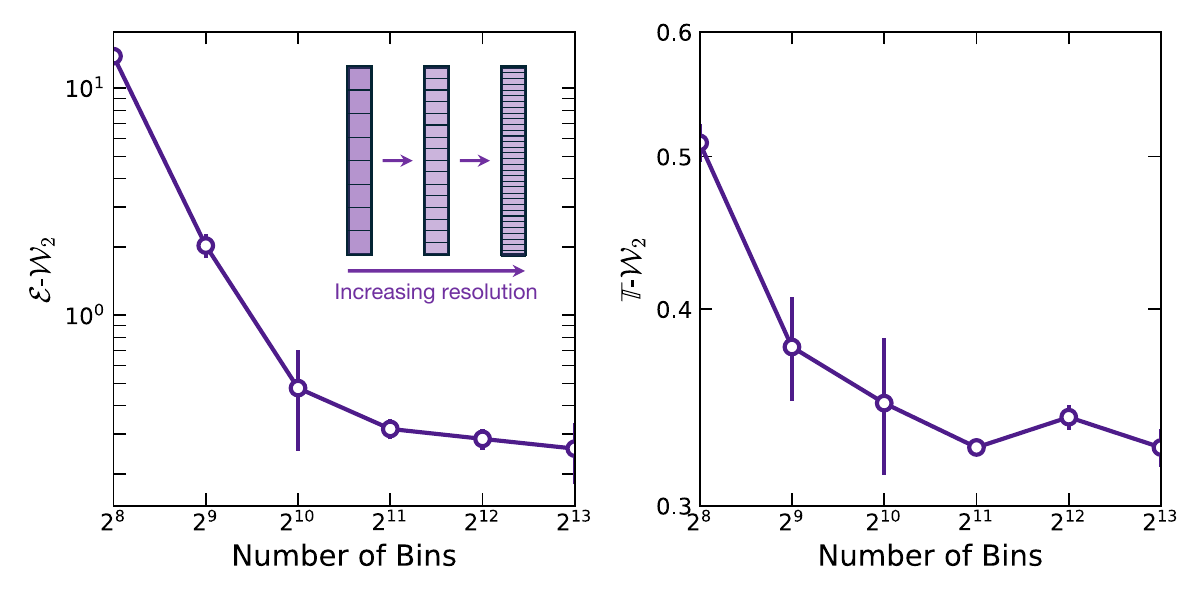}
   \caption{We vary the model's bin count and evaluate its impact on the resampled \emetric and \torusmetric for tri-alanine.}
  \label{fig:bin-ablations}
\end{figure*}

\section{Metrics}
Below, we introduce the metrics used to evaluate model performance and describe their computation. The proposed metrics capture both local and global behaviour. Energy-based metrics assess the accuracy of local interactions, as small geometric perturbations can induce large energy variations. Complementary global metrics--including torus- and TICA-based measures--evaluate mode coverage and the ability of models to capture multi-modal structure. We omit the effective sample size (ESS), as its interpretation is invalidated by the use of SMC.

\subsection{Main Geometric Metrics}
\looseness=-1
\paragraph{2-Wasserstein Energy Distance (\emetric)} To quantify the agreement between generated and reference energy distributions, we compute the squared 2-Wasserstein distance between the energies of generated samples and those obtained from MD. Let $p, q \in \mathcal{P}(\mathbb{R})$ denote the probability distributions over energy values for the generated and reference samples, respectively, and let $\Pi(p,q)$ denote the set of admissible couplings between them. The Wasserstein energy distance is then defined as:
\begin{equation}
\gE\text{-}\gW_2(p,q)^2
\;\triangleq\;
\min_{\pi \in \Pi(p,q)}
\int_{\mathbb{R} \times \mathbb{R}} |x - y|^2 \, \mathrm{d}\pi(x,y).
\end{equation}

This metric measures how closely the generated energy landscape matches the reference distribution. Because molecular energies are highly sensitive to local structural change such as bond lengths/angles, \emetric is particularly effective at detecting physically relevant discrepancies. Lower values correspond to better agreement with the target Boltzmann distribution.

\looseness=-1
\paragraph{Torus 2-Wasserstein Distance (\torusmetric)} To assess structural similarity in torsional space, we compute a 2-Wasserstein distance defined on the torus. For a molecule with $L \in \mathbb{N}$ residues, each conformation is represented by its vector of dihedral angles:
\begin{equation}
\mathrm{Dihedrals}(x)
=
(\phi_1, \psi_1, \ldots, \phi_{L-1}, \psi_{L-1})
\in [0,2\pi)^{2(L-1)}.
\end{equation}

To account for the periodicity of angular variables, the squared cost between two conformations $x$ and $y$ is defined as:
\begin{equation}
c_{\mathbb{T}}(x,y)^2
=
\sum_{i=1}^{2(L-1)}
\left[
\big(
\mathrm{Dihedrals}(x)_i
-
\mathrm{Dihedrals}(y)_i
+
\pi
\big)
\bmod 2\pi
-
\pi
\right]^2.
\end{equation}

The corresponding torus Wasserstein distance between two distributions
$p, q \in \mathcal{P}([0,2\pi)^{2(L-1)})$ is then defined as:
\begin{equation}
\gT\text{-}\gW_2(p,q)^2
\;\triangleq\;
\min_{\pi \in \Pi(p,q)}
\int c_{\mathbb{T}}(x,y)^2 \, \mathrm{d}\pi(x,y).
\end{equation}

This metric captures global conformational differences in torsional space while respecting angular periodicity. Unlike energy-based distances, \torusmetric is sensitive to missing or misrepresented conformational modes, providing a complementary assessment of structural diversity and coverage in generative Boltzmann models. One point of note is that although this claim generally holds, in cases where there are few samples from a given mode that are lost, this does not substantially impact the \torusmetric, meaning that we can see a reduced value even in the presence of mode loss---one clear example of this phenomenon is presented in~\cref{sec:alaninedipeptide_results}, where we demonstrate mode collapse despite decreasing \torusmetric.

\looseness=-1
\paragraph{TICA 2-Wasserstein Distance (\ticametric)}
To compare the long-timescale dynamical structure of trajectories, we evaluate discrepancies in a reduced space defined by time-lagged independent component analysis (TICA). TICA identifies collective coordinates that maximize autocorrelation, isolating the slow modes governing conformational dynamics.

Given a mean-centered time series $\{\tilde{x}_t\}_{t=1}^T \subset \mathbb{R}^n$ and a lag time $\tau$, we estimate the empirical covariance matrices:
\begin{equation}
\hat{C}_{00}
=
\frac{1}{T-\tau}\sum_{t=1}^{T-\tau} \tilde{x}_t \tilde{x}_t^\top,
\qquad
\hat{C}_{0\tau}
=
\frac{1}{T-\tau}\sum_{t=1}^{T-\tau} \tilde{x}_t \tilde{x}_{t+\tau}^\top.
\end{equation}
The dominant slow modes are obtained by solving the generalized eigenvalue problem
\begin{equation}
\hat{C}_{0\tau} w = \lambda \hat{C}_{00} w,
\end{equation}
where each eigenvector $w$ defines a linear projection with maximal normalized autocorrelation at lag $\tau$. In practice, we retain the first two TICA components $\{w_1,w_2\}$, which capture the slowest dynamical processes.

Using these projections, we define an $\ell_2$ cost between configurations $x,y \in \mathbb{R}^n$ as their Euclidean distance in TICA space:
\begin{equation}
c_{\textsc{TICA}}(x,y)^2
=
\sum_{j=1}^{2} \left( w_j^\top x - w_j^\top y \right)^2.
\end{equation}
The corresponding TICA Wasserstein distance between generated and reference distributions $p,q \in \mathcal{P}(\mathbb{R}^n)$ is then
\begin{equation}
\mathrm{TICA}\text{-}\mathcal{W}_2(p,q)^2
\;\triangleq\;
\min_{\pi \in \Pi(p,q)}
\int c_{\textsc{TICA}}(x,y)^2 \, \mathrm{d}\pi(x,y).
\end{equation}

This metric directly assesses agreement in the slow dynamical subspace learned from the reference trajectory. By construction, \ticametric is sensitive to mismatches in metastable state populations and transition pathways, making it well-suited for evaluating models intended to reproduce long-timescale molecular kinetics.

\subsection{On the use of Geometric over Likelihood-based metrics in High-dimensions}
\label{app:ess_critique}
\looseness=-1
In this work, we prioritize Wasserstein-based metrics (TICA, Torus, Energy) over likelihood-based metrics. While ESS is a standard diagnostic for the efficiency of SNIS estimators, it is widely recognized as a potentially misleading proxy for sample quality in high-dimensional spaces. In high-dimensional spaces, importance sampling is susceptible to the ``curse of dimensionality'', referred to as \textit{weight collapse} in the particle filtering literature~\citep{ObstaclestoHighDimensionalParticleFiltering}. As the dimensionality increases, the overlap between the typical sets of the proposal and target distributions vanishes exponentially. Consequently, the variance of the importance weights becomes dominated by rare samples that land in the small region of overlap. This results in an estimator variance that explodes, rendering ESS an unreliable metric for performance in systems with hundreds of degrees of freedom (e.g., Decapeptides), as the metric becomes sensitive to global scaling factors rather than local mode coverage.

\paragraph{The Bias Toward Mode Collapse}
The core limitation of the effective sample size (ESS) in the context of Boltzmann Generation is its tendency to reward ``mode-seeking'' behaviour over ``mass-covering'' behaviour. ESS is derived from the variance of the importance weights $w_i(x) = \mu_{\text{target}}(x_i)/p_{\theta}(x_i)$. Specifically we use the \textit{normalized} effective sample size that ranges between $[1/N,1]$,
\begin{equation*}
\text{ESS}\left(\{w_i\}_{i=1}^N\right) = \frac{\left( \sum_{i=1}^{N} w_i \right)^2}{N\sum_{i=1}^{N} w_i^2},
\end{equation*}

where $x_i \sim p_\theta$. A generative model $p_{\theta}$ can maximize ESS by collapsing its probability mass into a single, highly stable metastable state (a single mode of $\mu_{\text{target}}$). In this scenario, the ratio $\mu_{\text{target}}(x)/p_{\theta}(x)$ remains stable within that specific region, yielding a high ESS;\ however, this comes at the cost of failing to sample other metastable states (mode dropping). In~\cref{fig:al3_collapsed_mode}, we compare the Ramachandran plots between the ground truth MD data, FALCON, and \nameshort for one of the torsion angles present in tri-alanine. It can clearly be seen that the mode between 0 and $\frac{\pi}{2}$ exists in the training data, while being lost in FALCON---a model that obtains a higher ESS than \nameshort.

\looseness=-1
Conversely, a model that attempts to cover the full diversity of the Boltzmann distribution (``mass-covering'') is much more likely to assign non-zero probability to high-energy regions where $\mu_{\text{target}}(x) \approx 0$. This results in high variance of the importance weights and a low ESS, despite the model being superior in terms of exploring the global conformational space. 

\begin{figure*}[!h]
  \centering
  \includegraphics[width=0.7\textwidth]{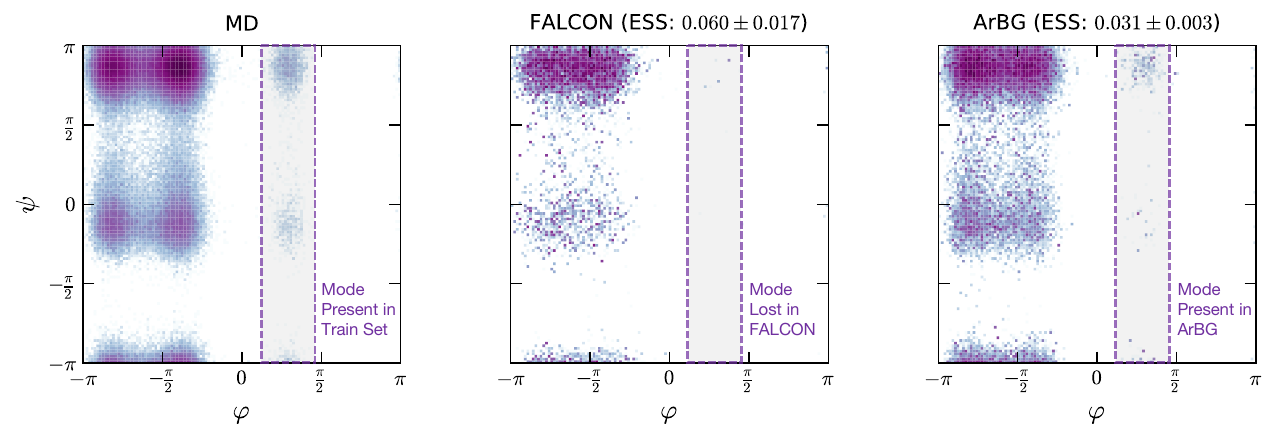}
  \caption{\textbf{Left}: Ramachandran plot from the ground truth MD data for tri-alanine;\ \textbf{Center}:\ FALCON's torsion angle predictions;\ \textbf{Right}:\ Torsion angle predictions from \nameshort. FALCON clearly loses one of the conformational modes at inference.}
  \label{fig:al3_collapsed_mode}
\end{figure*}

\looseness=-1
\paragraph{On the Interaction with Numerical Error in Practice} In practice, with models that have some numerical error, the variance of importance weights---and therefore ESS---is often dominated by numerical errors where a small fraction of samples will have an unusually high likelihood. This causes the creation of a large importance weight, which is why in practice, all models use a form of clipping to ensure reasonable importance weight values and to prevent collapse. In this work, we use a clipping value of $0.002$ where the samples with the largest $0.002$ fraction of importance weights are clipped following prior work~\citep{klein2023equivariant,midgley_se3_2023,tan2025amortizedsamplingtransferablenormalizing,rehman2025falconfewstepaccuratelikelihoods}. In practice, ESS is highly sensitive to numerical precision and is often dominated by outliers, particularly in high-dimensional settings. By contrast, geometry-based metrics are substantially more robust to such numerical effects and provide a more reliable characterization of global model behaviour.

\section{Additional Results}
\subsection{Temperature Tuning}\label{sec:optimal_temps}
Sampling temperature controls the entropy of the model distribution by scaling logits at inference time. The optimal temperatures for all systems are reported in~\cref{tab:optimal_temperatures}. For smaller systems, temperatures near 1.0 are sufficient, with slight gains observed for values marginally above 1.0. In contrast, larger and more complex systems benefit from lower temperatures, which likely mitigate underfitting.

\begin{table}[!h]
\centering
\caption{Optimal sampling temperatures identified via inference-time temperature sweeps across molecular systems.}
\label{tab:optimal_temperatures}
\resizebox{0.35\linewidth}{!}{
\begin{tabular}{lc}
\toprule
System & Optimal Temperature, $T$ \\
\midrule
Alanine Dipeptide        & 1.03 \\
Tri-alanine              & 0.99 \\
Alanine Tetrapeptide     & 1.02 \\
Hexa-alanine             & 0.98 \\
Chignolin                & 0.88 \\
\midrule
\namemodel (Transferable)     & 0.95 \\
\bottomrule
\end{tabular}}
\end{table}

\begin{figure*}[t]
  \centering
  \includegraphics[width=\textwidth]{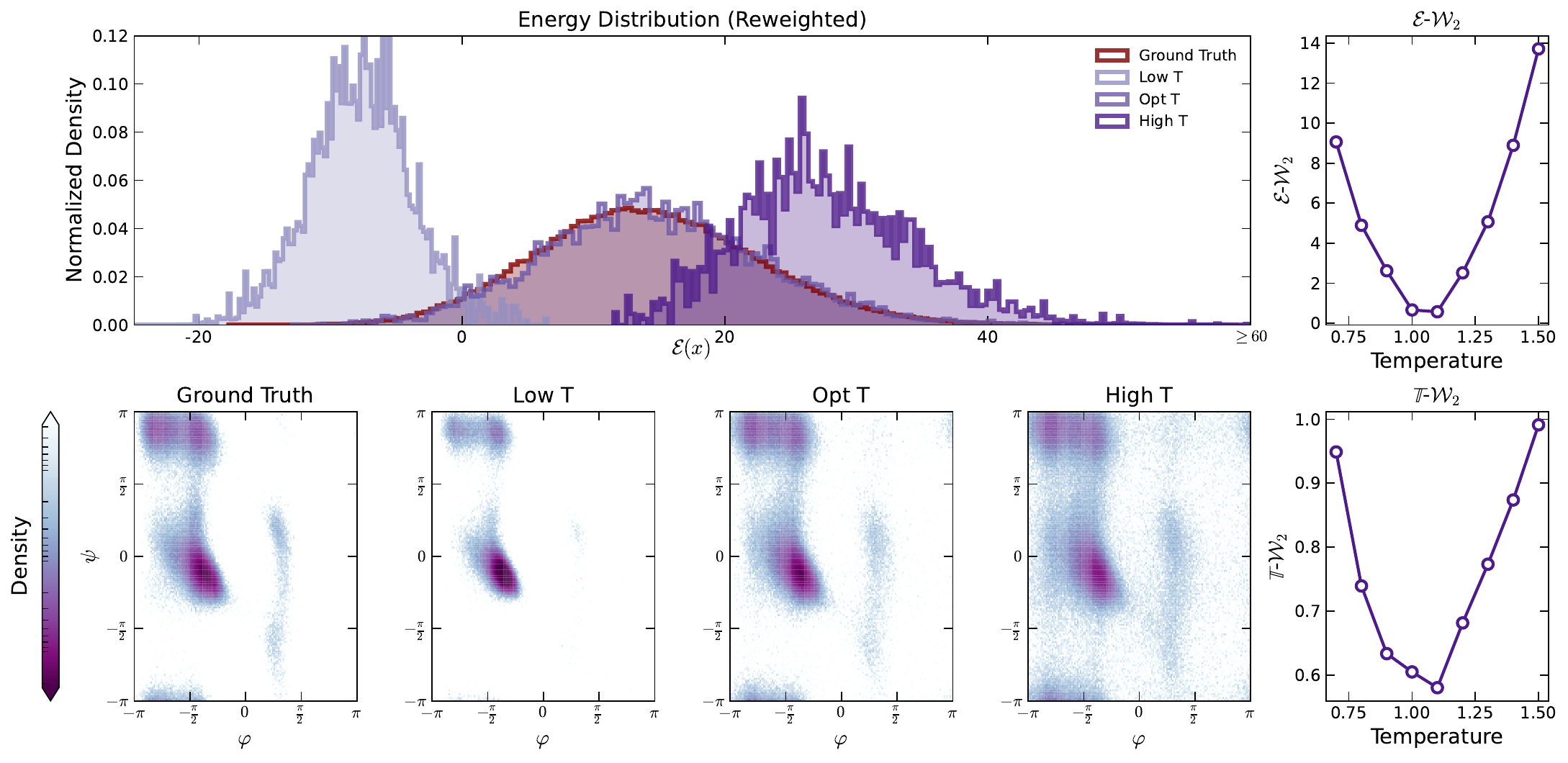}
  \vspace{-10pt}
   \caption{Ablations on model temperature. For the energy distribution, we demonstrate that lower temperatures sample lower energy modes more frequently, while the converse holds for higher temperatures. We also show how the modes become more prominent at high temperatures and are lost at lower temperatures. Finally, we show how an optimal temperature exists for optimizing \emetric and \torusmetric.}
  \label{fig:temperature-ablations}
\end{figure*}

\paragraph{Temperature on Alanine Tetrapeptide} In~\cref{fig:temperature-ablations}, we study the effects of temperature tuning on the proposal and re-weighted distributions on the alanine tetrapeptide system. First, as noted in the main text, we find that the optimal temperature is slightly higher than 1.0. This is interesting and in line with prior works that find a slightly more diffuse proposal may be slightly better for Boltzmann Generation metrics, as it allows better coverage of the space.

\begin{table}[h]
    \centering
    \caption{Inference speed in samples per second for best performing models in the transferable setting. \namemodel is around 50\% faster than Prose per model evaluation, but is slower in terms of samples per second due to operating over dimensions instead of atom coordinates and therefore requires $3\times$ the model evaluations.}
    \label{tab:inference-speed}
    \begin{tabular}{l r r r}
    \toprule
        & 2AA & 4AA & 8AA \\
        \midrule
         TarFlow & 737 & 329 & 126 \\
         PROSE & 338 & 158 & 66 \\
         \namemodel & 260 & 87 & 29\\
         \bottomrule
    \end{tabular}
\end{table}

\subsection{Inference Time}
In~\cref{tab:inference-speed}, we report inference throughput (samples per second) for transferable models with 2, 4, and 8 residues, averaged over 30 systems. While \namemodel is slower than Prose, its substantially higher sample quality yields superior performance under a fixed sampling budget (see~\cref{fig:scaling}).

\subsection{Alanine Dipeptide}
\label{sec:alaninedipeptide_results}
Below, we summarize the results of all \nameshort variants in conjunction with other baselines on ALDP. \nameshort outperforms all competing models---including both discrete flows and CNFs, on \emetric, with competitive performance on \torusmetric. 

\begin{table}[!ht]
\centering
\caption{Results on alanine dipeptide. Best results are \textbf{bolded}, with second-best \underline{underlined}. }
\label{tab:aldp_results}
\begin{tabular}{lcc}
    \toprule
    & \multicolumn{2}{c}{Alanine dipeptide (ALDP)} \\
    \cmidrule(lr){2-3}
    Algorithm $\downarrow$ & \emetric $\downarrow$ & \torusmetric $\downarrow$ \\
    \midrule
    BoltzNCE & 0.27 $\pm$ 0.02 & 0.57 $\pm$ 0.00 \\
    SE$(3)$-EACF & 108.202 & 2.867 \\
    ECNF & 0.419 & 0.311 \\
    RegFlow & 0.501 $\pm$ 0.011 & 0.951 $\pm$ 0.054 \\
    ECNF++  & 	0.914 $\pm$ 0.122& 0.189 $\pm$ 0.019 \\
    SBG  &  0.741 $\pm$ 0.189 & 0.431 $\pm$ 0.141 \\
    FALCON-A & {0.512 $\pm$ 0.038} & \underline{0.180 $\pm$ 0.005} \\
    FALCON &  \underline{0.225 $\pm$ 0.104} & 0.402 $\pm$ 0.021 \\
    \midrule
    GIVT &  0.256 $\pm$ 0.033 & \textbf{0.175 $\pm$ 0.171} \\
    MoL-PixelCNN++ & 1.447 $\pm$ 0.277 & 0.528 $\pm$ 0.028 \\
    GMM-PixelCNN++ & 0.763 $\pm$ 0.118 & 0.354 $\pm$ 0.098 \\

    \midrule
    \shortname & \textbf{0.209 $\pm$ 0.041} & 0.402 $\pm$ 0.008\\
    \bottomrule
\end{tabular}
\vspace{0mm}
\end{table}

\looseness=-1
\paragraph{Mode Collapse} When training models on the alanine dipeptide dataset from \citet{klein2023equivariant}, we observe that we can continue improving our performance across both global and local metrics if we train our models for longer;\ however, in this process, part of the performance improvement comes from losing the mode, which artificially inflates ESS. Torus, which is designed to be a global metric, also suffers given that there are an insufficient number of points in that mode to radically impact the degradation of performance, yielding a nearly monotonic trend in performance improvement as training time increases. For larger systems, like tri-alanine and above, this behaviour is not observed.

In~\cref{fig:mode_loss}, we provide a clear demonstration of the lost mode on ALDP. The Ramachandran plots are shown for two different instances in the training process---Epoch 110 and Epoch 370. We show that earlier in training, the mode exists, but as training continues, it disappears. This can also be observed on the training loss curve as annotated. We also provide the ESS and Torus results during training to demonstrate that the loss of the mode improves performance on metrics. 
\begin{figure*}[t]
  \centering
  \includegraphics[width=\textwidth]{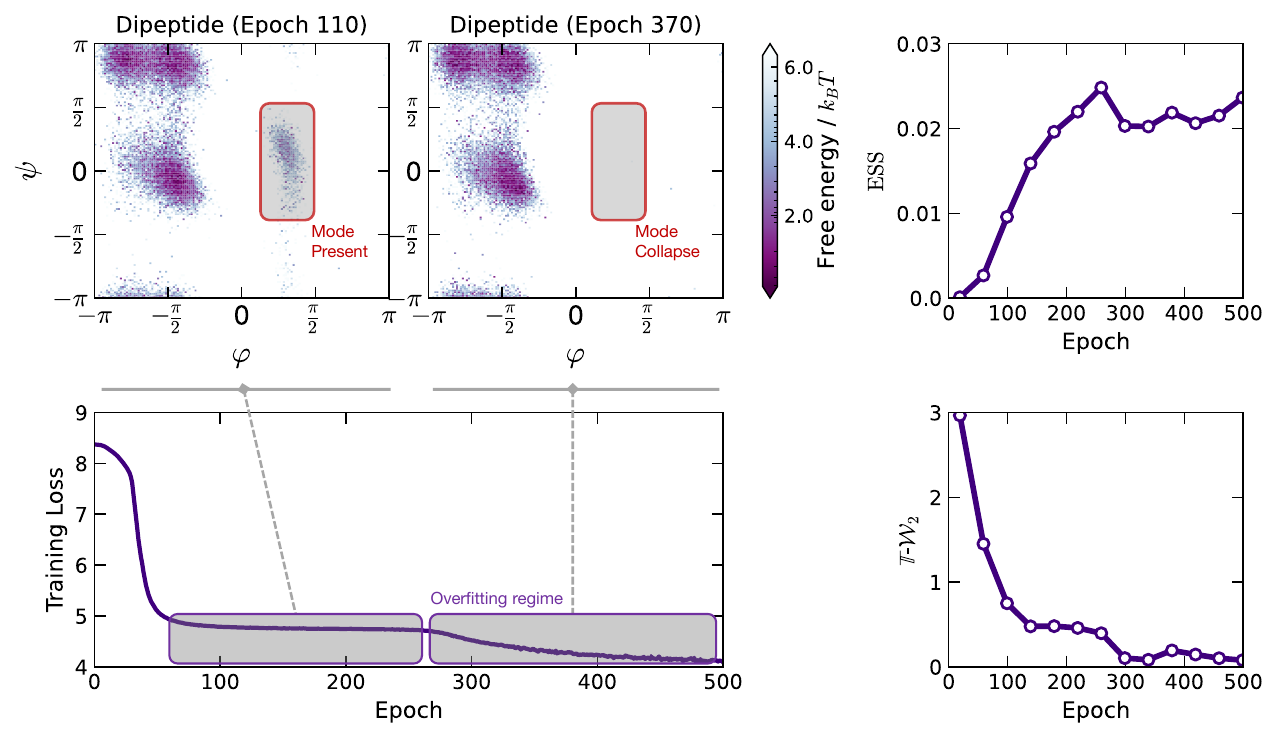}
  \caption{We demonstrate that training models for too long leads to overfitting on the training data, which despite improving resampled metrics, yields undesirable behaviour.}
  \label{fig:mode_loss}
\end{figure*}

We believe this stems from the method used to generate the dataset in \citet{klein2023equivariant}. Specifically, the dataset was generated in the following way:
\begin{enumerate}[topsep=0pt, partopsep=0pt, itemsep=3pt, parsep=0pt, leftmargin=*]
    \item MD simulation using \textit{Amberff99SBildn} force-field at 300K for 1 ms using \texttt{openMM}~\cite{eastman2024openmm} with a timestep of 1 femto-second. 
    \item Relaxation of $10^5$ uniformly randomly selected states from the MD data for $100$ femto-seconds each using the \textit{GFN2-xTB} forcefield \cite{Bannwarth2018GFN2xTBAnAA} and the ASE library \cite{Hjorth_Larsen_2017} with a friction constant of 0.5 a.u.
    \item To make the density nearly equal between negative and positive $\varphi$ dihedral angles, importance sampling is performed using weights from a von Mises distribution $f_{vM}$. Specifically, weights for each sample are computed as:
    \begin{equation}
        \omega(\varphi) = 150 f_{vM}(\varphi | \mu=1, \kappa=10) + 1
    \end{equation}
    with $10^5$ training samples drawn from the weighted distribution.
\end{enumerate}
Specifically, this final reweighting step makes it possible for powerful models to overfit on the positive $\varphi$ mode. The reweighting step causes there to be multiple instances of exactly the same data sample in the training set. For powerful models, seeing the same datapoint multiple times (even with data augmentation) causes overfitting. We observe that the likelihood of these exact training samples explodes, causing the distribution after importance sampling to remove the positive $\varphi$ mode.

This mixed energy function usage creates somewhat of a problem for importance sampling. In practice, following previous work, use the \textit{Amberff99SBildn} force-field at 300K as a target energy function for reweighting, but note that this is not quite a perfect fit as the samples are relaxed slightly with the \textit{GFN2-xTB} forcefield which may create a slight mismatch between the target distribution and the actual \textit{Amberff99SBildn}-defined equilibrium distribution. 

\paragraph{Recommendation} For newer and more powerful Boltzmann Generator models, we recommend using training sets without importance sampling, as these are much more difficult to overfit on specific training samples. It is important to be mindful of overfitting-type behaviour on these small datasets with relatively powerful models.

\subsection{Ramachandran Plots for Other Single Peptide Systems}
Here, we demonstrate the competitive performance of \nameshort across single peptide systems by showing the Ramachandran plots for all systems considered. In all cases considered, \nameshort captures nearly every mode present in the test data, clearly illustrating the quality of the learned likelihoods and their synergy with SNIS. 
\begin{figure*}[!ht]
  \centering
  \includegraphics[width=0.35\textwidth]{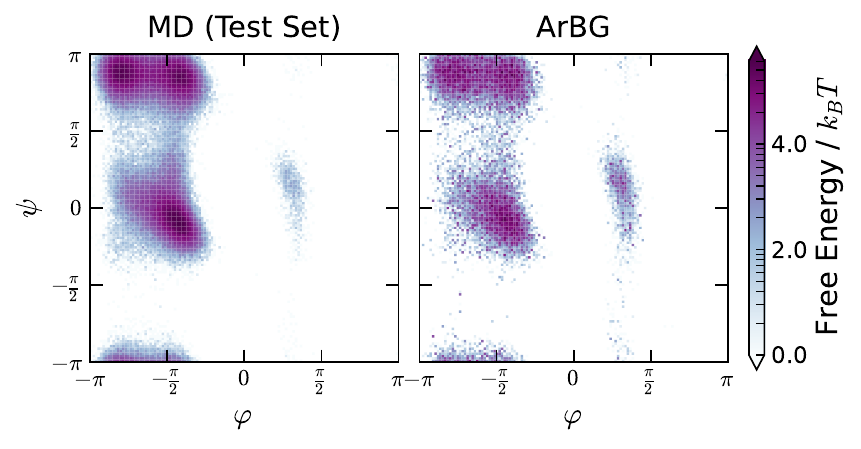}
  \caption{\textbf{Left}: Test data for alanine dipeptide; \textbf{Right}:\ \nameshort's angular predictions for alanine dipeptide.}
  \label{fig:aldp-rama}
\end{figure*}
\begin{figure*}[!ht]
  \centering
  \includegraphics[width=0.65\textwidth]{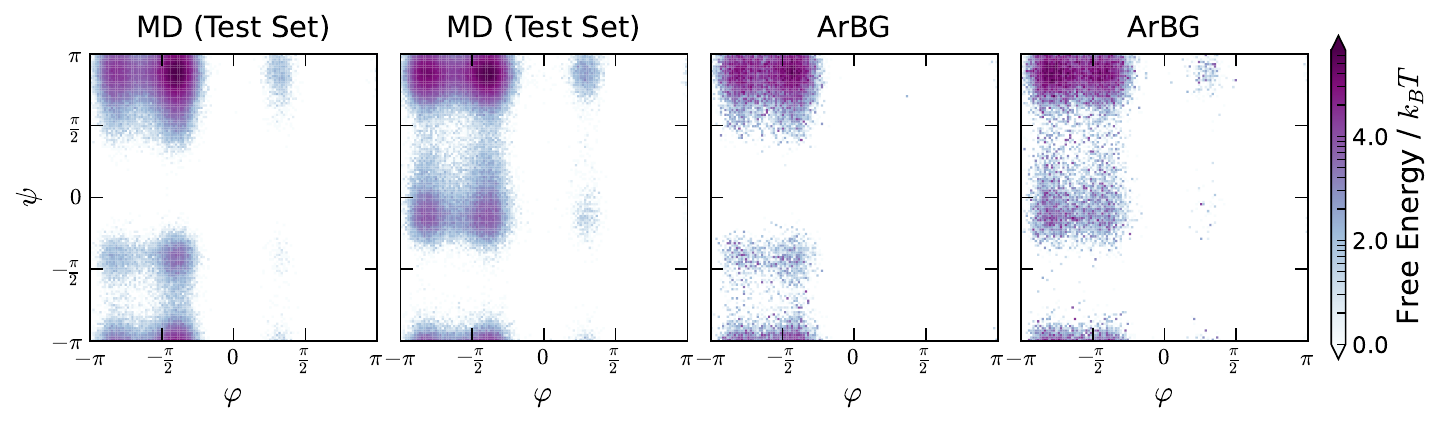}
  \caption{\textbf{Left}: Test data for tri-alanine;\ \textbf{Right}:\ \nameshort's angular predictions for tri-alanine.}
  \label{fig:al3-rama}
\end{figure*}
\begin{figure*}[!ht]
  \centering
  \includegraphics[width=\textwidth]{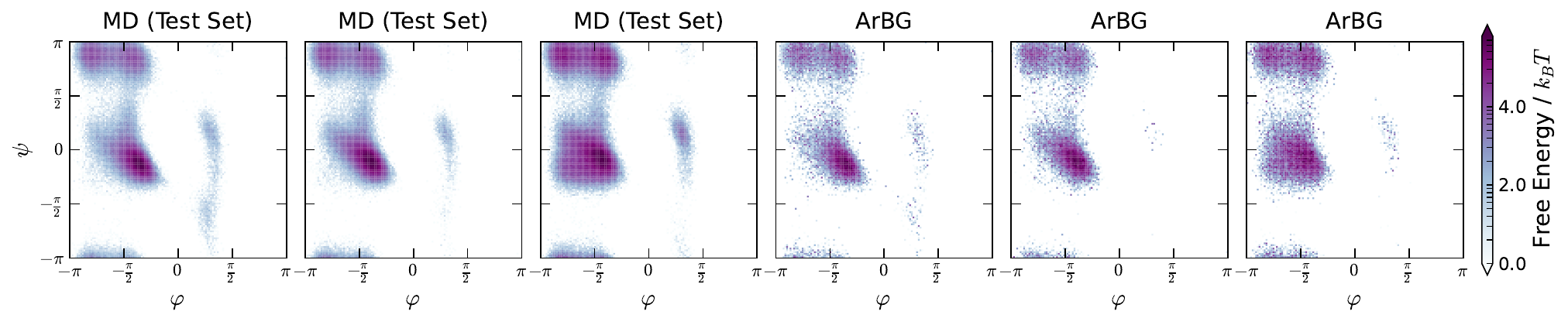}
  \caption{\textbf{Left}: Test data for alanine tetrapeptide;\ \textbf{Right}:\ \nameshort's angular predictions for alanine tetrapeptide.}
  \label{fig:al4-rama}
\end{figure*}
\begin{figure*}[!ht]
  \centering
  \includegraphics[width=\textwidth]{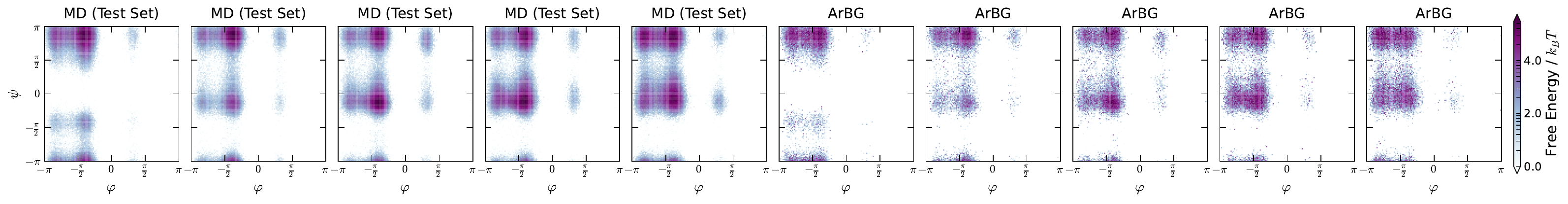}
  \caption{\textbf{Left}: Test data for hexa-alanine;\ \textbf{Right}:\ \nameshort's angular predictions for hexa-alanine.}
  \label{fig:al6-rama}
\end{figure*}

\subsection{De-quantization Strategies}
For our discrete model to generate continuous coordinates, we explored three different sampling strategies:\
\begin{enumerate}[topsep=0pt, partopsep=0pt, itemsep=3pt, parsep=0pt, leftmargin=*]
    \item Sampling uniformly from the discrete bin selected by the model. This is our reasonable default choice and defines a piecewise-constant continuous density in $\R^d$. 
    \item Using the center value of the bin. This strategy reduces a bit of variability from generation and may make bond lengths slightly more uniform. With enough bins, this is not necessary. Empirically, we observed small benefits for this de-quantization strategy over uniformly sampling from the bin.
    \item Using the biased training data distribution for the chosen molecule to determine empirical offsets that we apply to the chosen bin at inference time. Especially for larger bin sizes, we may be able to fit more interesting distributions concerning the training dataset within the bin. Here, we use a Dirac distribution over the empirical mean of the training dataset within the bin. We find this gives a small boost in performance, especially when the training data is centered.
\end{enumerate}

\subsection{Transferable Generation}
\paragraph{Autoregressive Twisted SMC Efficiency Benefits} 
We evaluate the efficiency gains that can be obtained using our Autoregressive Twisted SMC algorithm. The main idea is that it is easy to detect samples that will not result in valid low energy samples early on during inference. Using the twist function defined in \eqref{eqn:twist_smc}, we are able to essentially stop inference early for any sample that exhibits a high partial energy. In~\cref{fig:smc-partial}, we investigate the partial energy distributions on the sequence \texttt{SQQKVAFE} 8AA test set peptide for \namemodel. To investigate this, we perform SMC inference without resampling for 10,000 generations. We record the partial energy of each sample at each residue checkpoint. We find that residue 2 has around 2\% of samples that have poor energy samples while residue 7 has $>7$\% of samples that have high energies and represent likely steric clashes or other high energy features. These samples represent ``wasted'' compute, in that they will not contribute to the final distribution of samples. Therefore, additional efficiency can be gained by filtering these out early. On this peptide, using the $min(\mathcal{E}(x)) + 100$ filter on energy at the earliest time a sample is registered as high energy, we find a savings of roughly 3\% over a method without intermediate resampling.

While this is a relatively minor saving, we expect that for more complicated and larger systems where the proposal has more failure modes, the advantage of autoregressive twisted SMC here would increase substantially in terms of cost savings or sampling efficiency, depending on the exact method and utilization of this concept.

\begin{figure}[!h]
    \centering
    \includegraphics[width=0.7\linewidth]{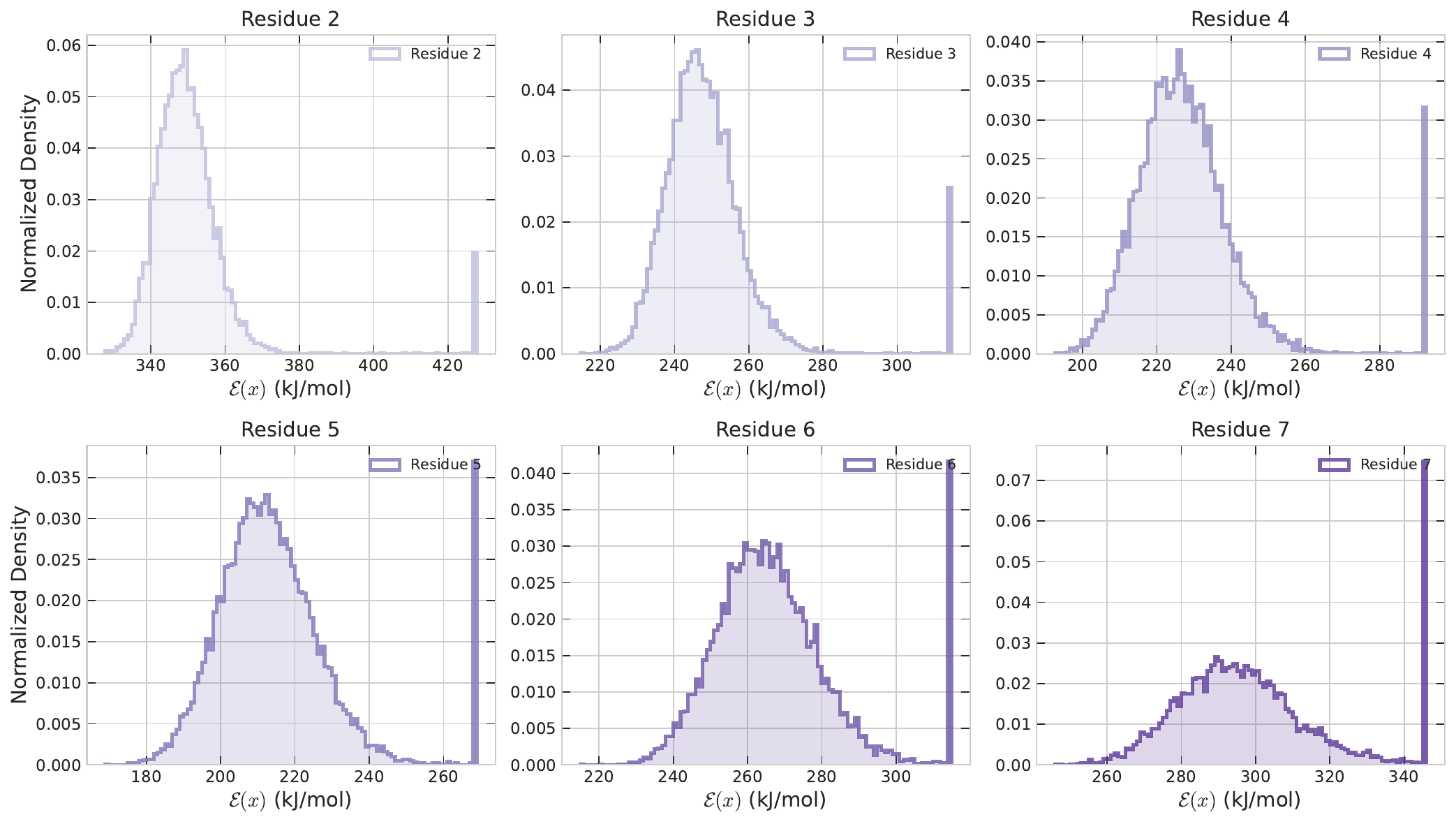}
    \caption{Histogram of the energy values for 10,000 samples on \texttt{SQQKVAFE} for each residue, where we perform resampling for SMC. We clip the maximum energy to $\min(\mathcal{E}(x)) + 100$. The spikes represent all values greater than or equal to that histogram value. We can see that by residue 7, $>7$\% of samples have extremely large energies, which likely represent clashes or erroneous bond lengths.}
    \label{fig:smc-partial}
\end{figure}

\paragraph{TICA Plots for Unseen Octapeptides}
To demonstrate the quality of \namemodel and its zero-shot performance, we include TICA plots for seven different unseen octapeptides in~\cref{fig:unseen_octapeptides}. In this process, we show the predictive capacity of \namemodel as it nearly perfectly captures all the modes of these unseen molecules. 

\begin{figure*}[!ht]
  \centering
  \includegraphics[width=\textwidth]{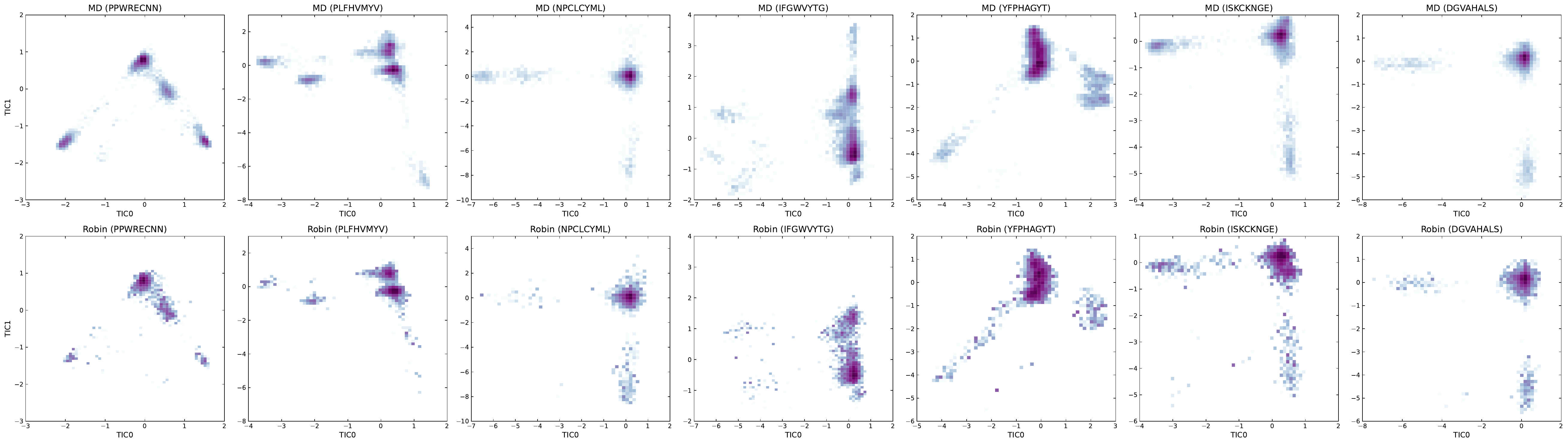}
  \caption{TICA modes using \namemodel on seven unseen octapeptides (from left to right):\ \texttt{PPWRECNN}, \texttt{PLFHVMYV}, \texttt{NPCLCYML}, \texttt{IFGWVYTG}, \texttt{YFPHAGYT}, \texttt{ISKCKNGE}, \texttt{DGVAHALS}.}
  \label{fig:unseen_octapeptides}
\end{figure*}

\paragraph{Peptide-level Performance and Learnability}
We evaluate the learned model on each peptide in the test set and report the resampled \emetric and \torusmetric in~\cref{fig:perpeptide-results,fig:perpeptide-results-torus} for \namemodel, Prose, and TarFlow. Overall, performance is broadly comparable across models on a per-sequence basis; however, several peptides remain challenging for all methods. For example, all models perform poorly on the \emetric for \texttt{KRRGFFLE}. Further analysis indicates that sequences with substantial charge contributions are particularly difficult to learn, especially those containing \texttt{R} and \texttt{Y} residues. Although all amino acids incur steep energy penalties outside favourable conformations, certain side chains are exceptionally sensitive to small geometric perturbations. Arginine's planar, highly charged guanidinium group exhibits strongly orientation-dependent electrostatics, while tyrosine's aromatic ring engages in highly directional non-bonded interactions~\citep{yoo2016improved,jing2019polarizable}. Consequently, minor deviations in side-chain geometry can produce large energy fluctuations, complicating the learning of equilibrium conformational distributions for these residues.

\begin{figure*}[!h]
  \centering
  \includegraphics[width=0.9\textwidth]{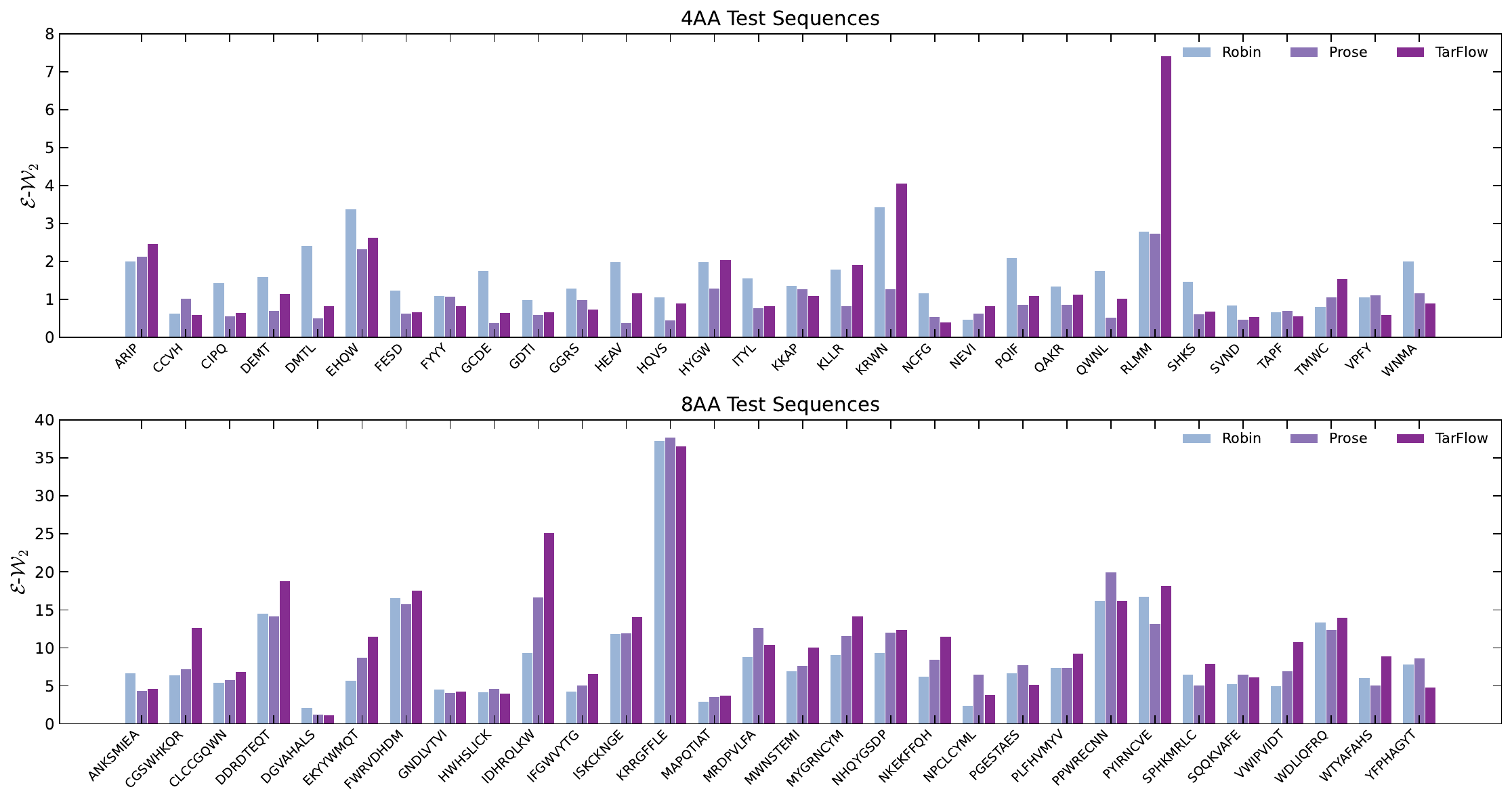}
  \caption{The resampled \emetric across peptides when comparing \namemodel, Prose, and SBG. Models were evaluated using  $10^4$ samples.}
  \label{fig:perpeptide-results}
\end{figure*}

\begin{figure*}[!h]
  \centering
  \includegraphics[width=0.9\textwidth]{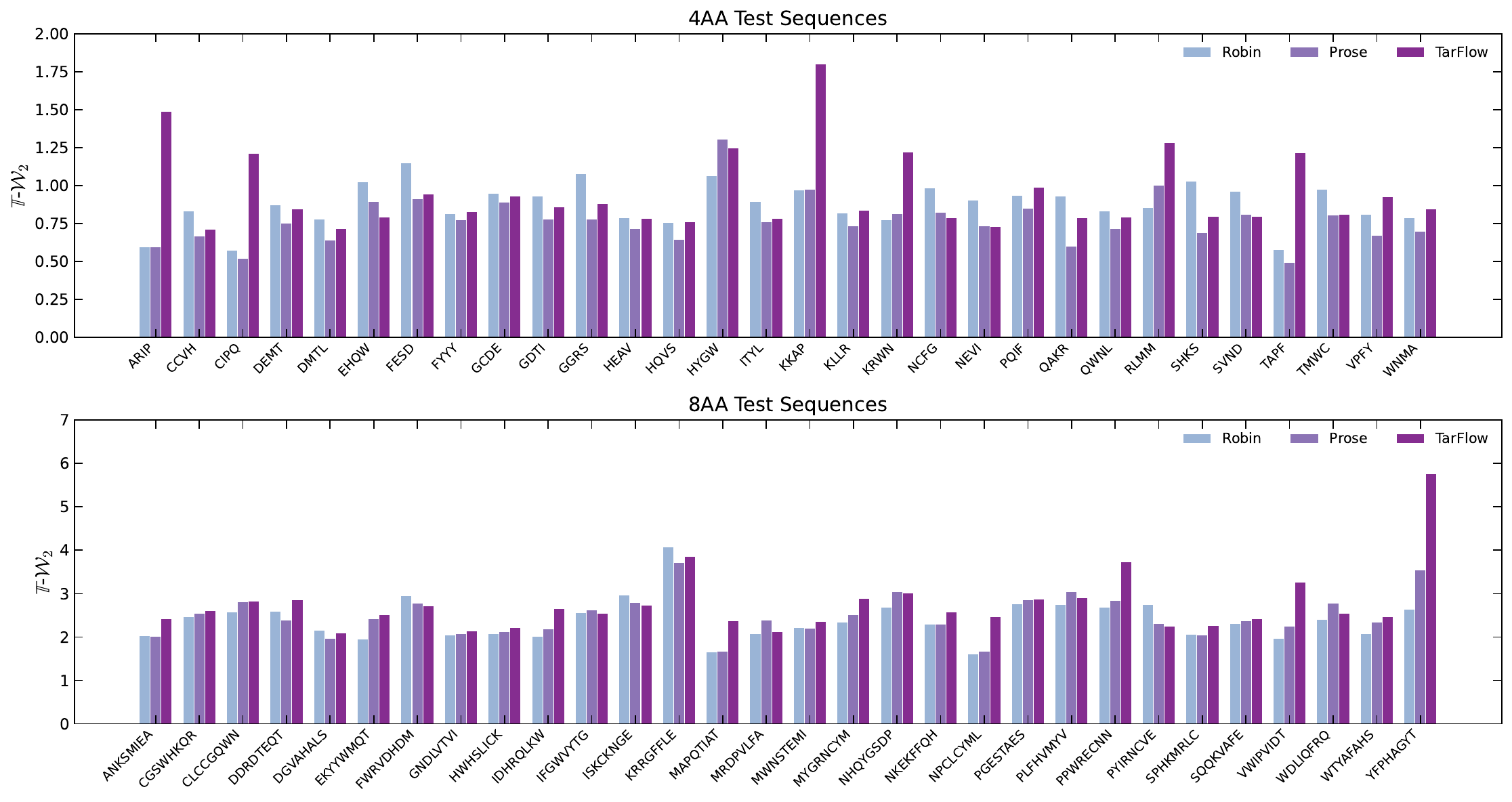}
  \caption{The resampled \torusmetric across peptides when comparing \namemodel, Prose, and SBG. Models were evaluated using  $10^4$ samples.}
  \label{fig:perpeptide-results-torus}
\end{figure*}

\subsection{Inference Scaling}\label{sec:scaling}

In~\cref{fig:scaling}, we demonstrated the scaling performance of \namemodel against Prose and MD on unseen octapeptide systems in terms of the number of function and energy evaluations. In this section, we continue an investigation into the inference scaling behaviour of \namemodel on octapeptides.

\paragraph{GPU hour performance} In~\cref{fig:scaling-appendix}, we investigate not only an equal number of energy evaluations, but also the scaling performance in terms of GPU hours. While \namemodel is slower than Prose or MD per model or energy function evaluation, its performance for the same number of GPU hours is better, especially on the \torusmetric.

Here, and in~\cref{fig:scaling}, we notice an unexpected trend in the \ticametric plots. Specifically, the \ticametric is relatively flat for MD, then spikes at around $10^7$-$10^8$ energy evaluations. We investigate this further by breaking the performance out sequence by sequence in \cref{fig:perpeptide-scaling-1,fig:perpeptide-scaling-2,fig:perpeptide-scaling-3}, where we show all 30 test set eight residue sequences. We notice a few interesting things, as stated in the following:\
\begin{enumerate}[topsep=0pt, partopsep=0pt, itemsep=3pt, parsep=0pt, leftmargin=*]
    \item As noted in \cref{fig:perpeptide-results-torus,fig:perpeptide-results}, the variability between sequences is quite high. Often, the performance is non-monotonic. For some sequences, one model is better than another, particularly at low energy evaluations;\ however, as the number of energy evaluations grows, \namemodel generally outperforms Prose.
    \item \ticametric for MD often has extremely non-monotonic elements often after $10^7$ energy evaluations due to mode jumps. Specifically, it takes around $10^7$ MD steps for chains to jump to the next mode. This often drastically changes \ticametric as the relative weights between modes evolve, especially when the new mode has less free energy than the starting mode. We demonstrate this clearly in~\cref{fig:tica-scaling}, where we see the MD over-sample a mode that is incorrectly weighted, leading to significant spikes in the \ticametric.
\end{enumerate}
\looseness=-1
We note that these plots are compared to the test set trajectories which have been run for 50 times as long as the longest MD chain. Given that we see the first mode mixing events at around $10^7$ energy evaluations, this implies that the test chains may not be fully mixed.

An interesting difference between MD and BG traces is that BG traces are often more monotonic than the MD trajectories. This is reasonable as BGs (both Prose and \namemodel) sample independently from their proposal where MD is autocorrelated. This means that the performance is often significantly better for BG type models in the few-step regime for unseen peptides. \namemodel outperforms all others on 8 residue systems on average.

\begin{figure*}[!h]
  \centering
  \includegraphics[width=0.75\textwidth]{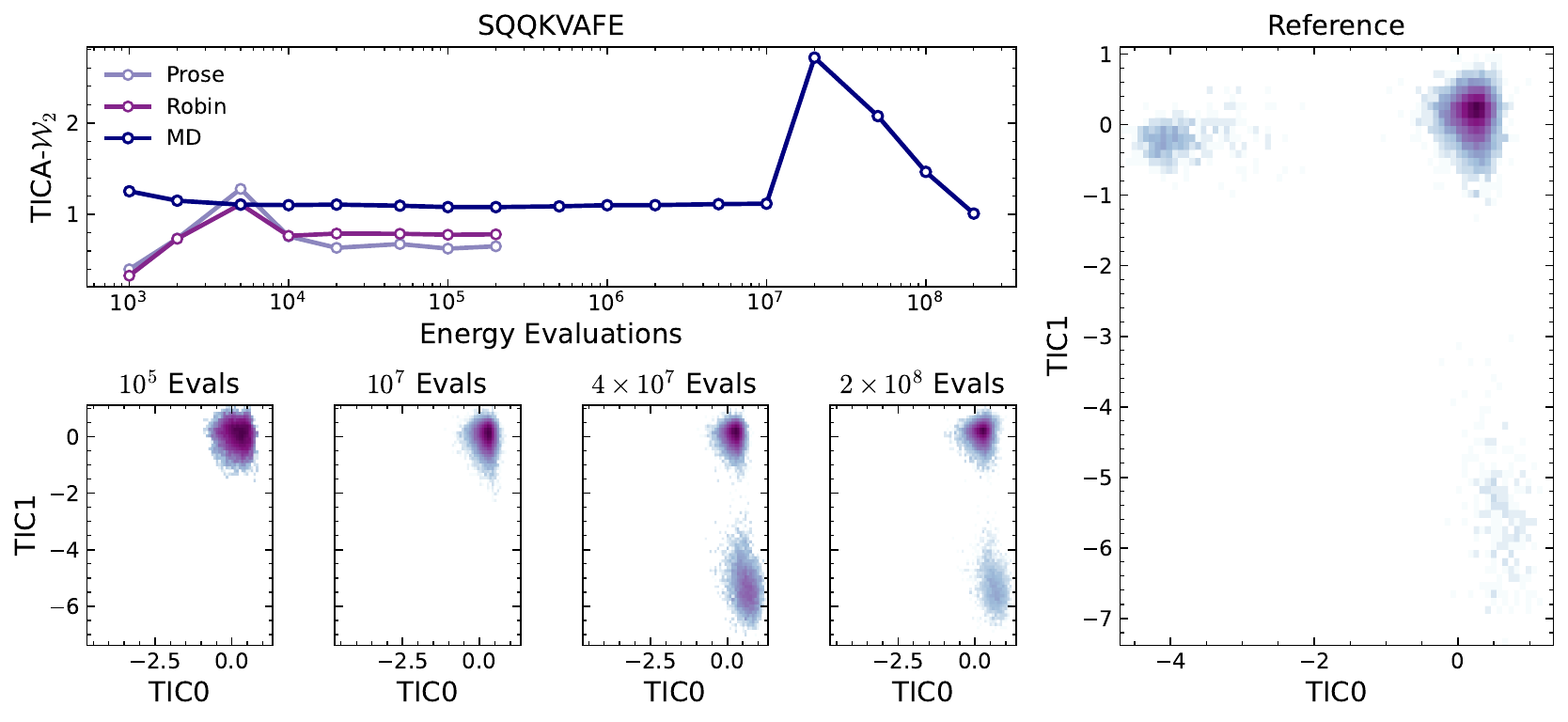}
  \caption{How the \ticametric varies as a function of energy evaluations across MD, Prose, and \namemodel. In addition, in the bottom row we see how the MD simulation slowly discovers modes as more energy evaluations are performed;\ in this process, it often searches in incorrect regions, amplifying the \ticametric, and then recovering from it with additional samples.}
  \label{fig:tica-scaling}
\end{figure*}

\begin{figure*}[!h]
  \centering
  \includegraphics[width=\textwidth]{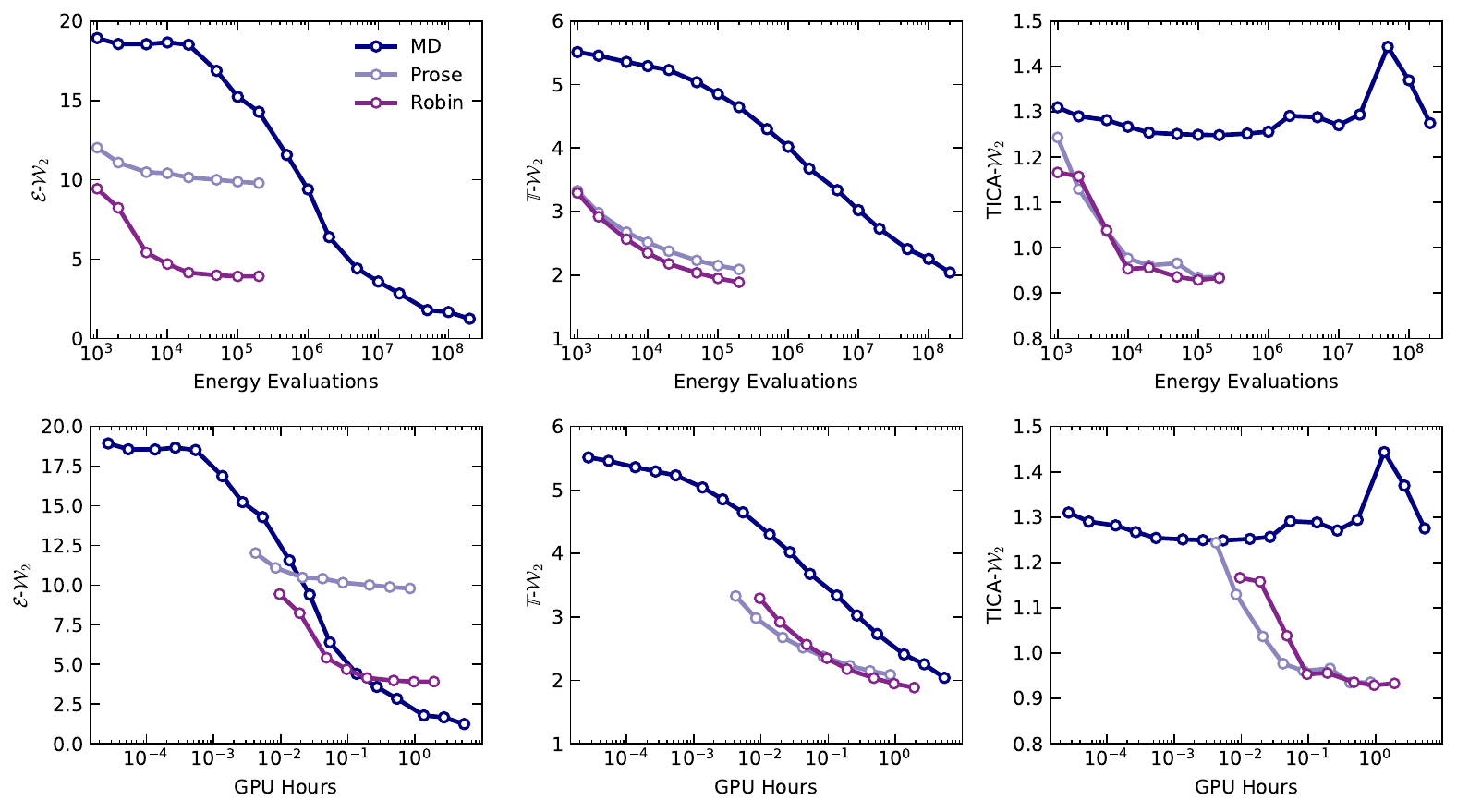}
  \caption{\emetric, \torusmetric, and \ticametric against the number of energy evaluations and GPU hours for MD, Prose, and \namemodel.}
  \label{fig:scaling-appendix}
\end{figure*}

\newpage\clearpage
\begin{figure*}[h]
  \centering
  \includegraphics[width=\textwidth]{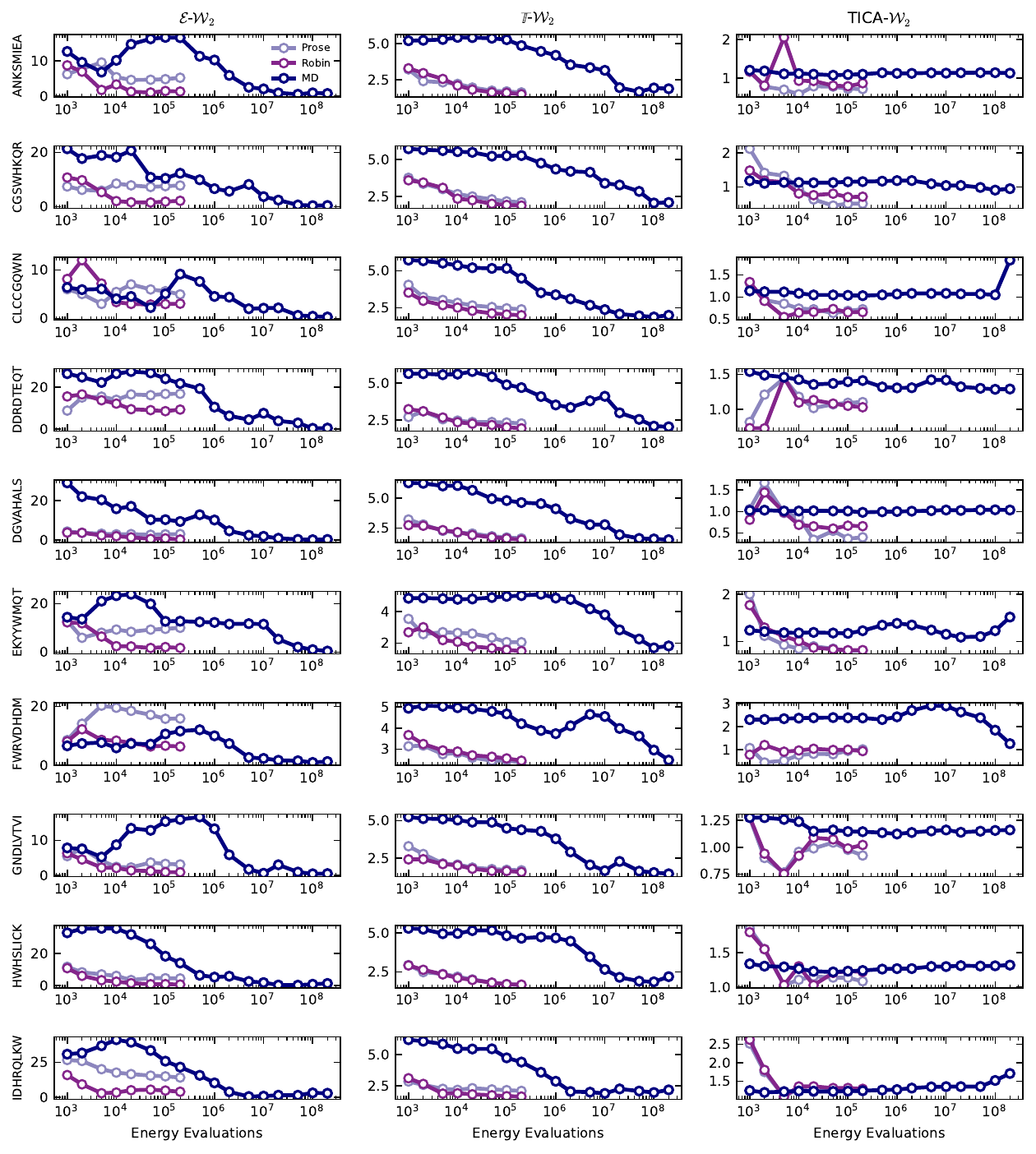}
  \caption{\emetric, \torusmetric, \ticametric per peptide vs.\ Energy Evaluations.}
  \label{fig:perpeptide-scaling-1}
\end{figure*}

\begin{figure*}[h]
  \centering
  \includegraphics[width=\textwidth]{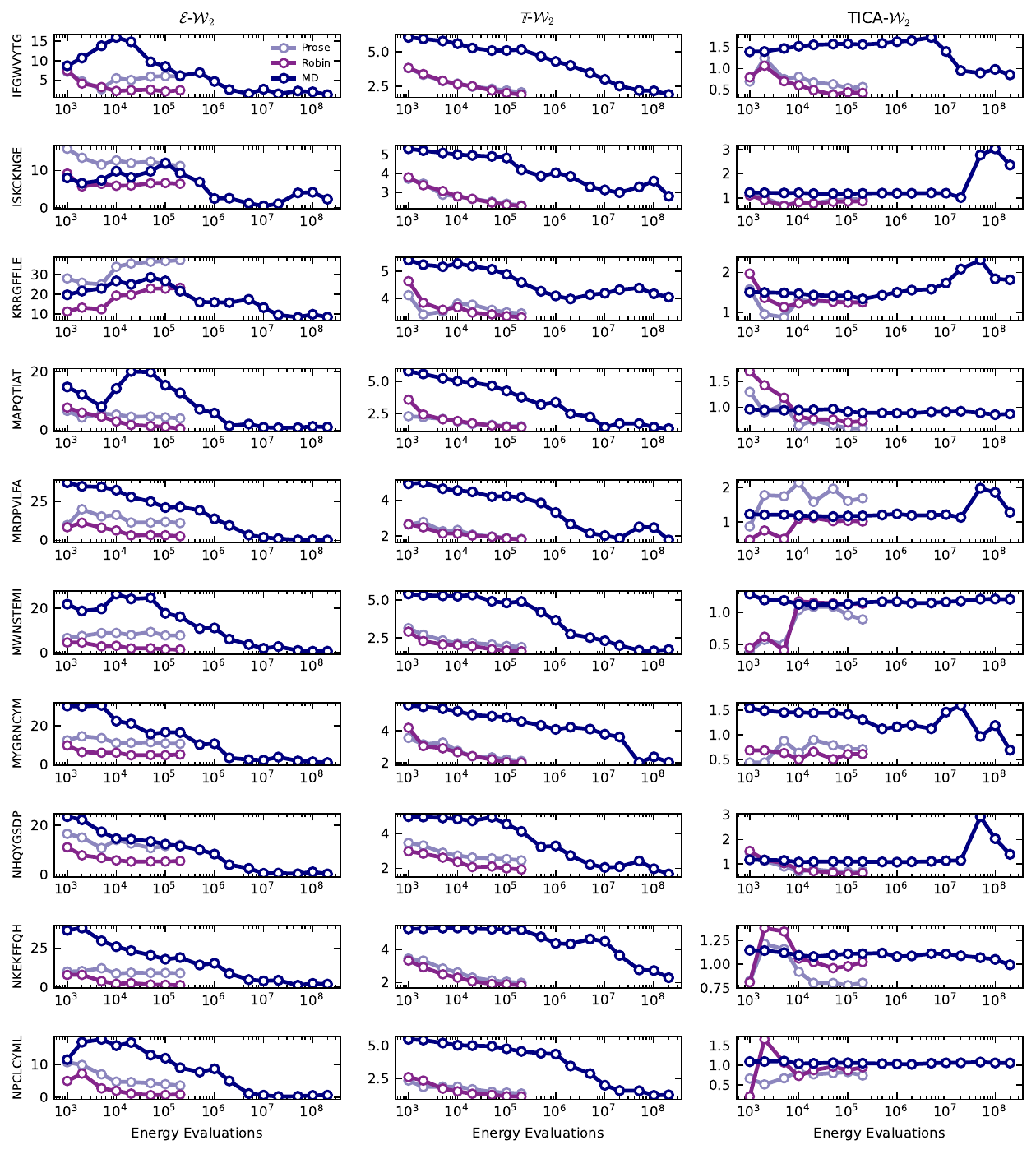}
  \caption{\emetric, \torusmetric, \ticametric per peptide  vs.\ Energy Evaluations.}
  \label{fig:perpeptide-scaling-2}
\end{figure*}

\begin{figure*}[h]
  \centering
  \includegraphics[width=\textwidth]{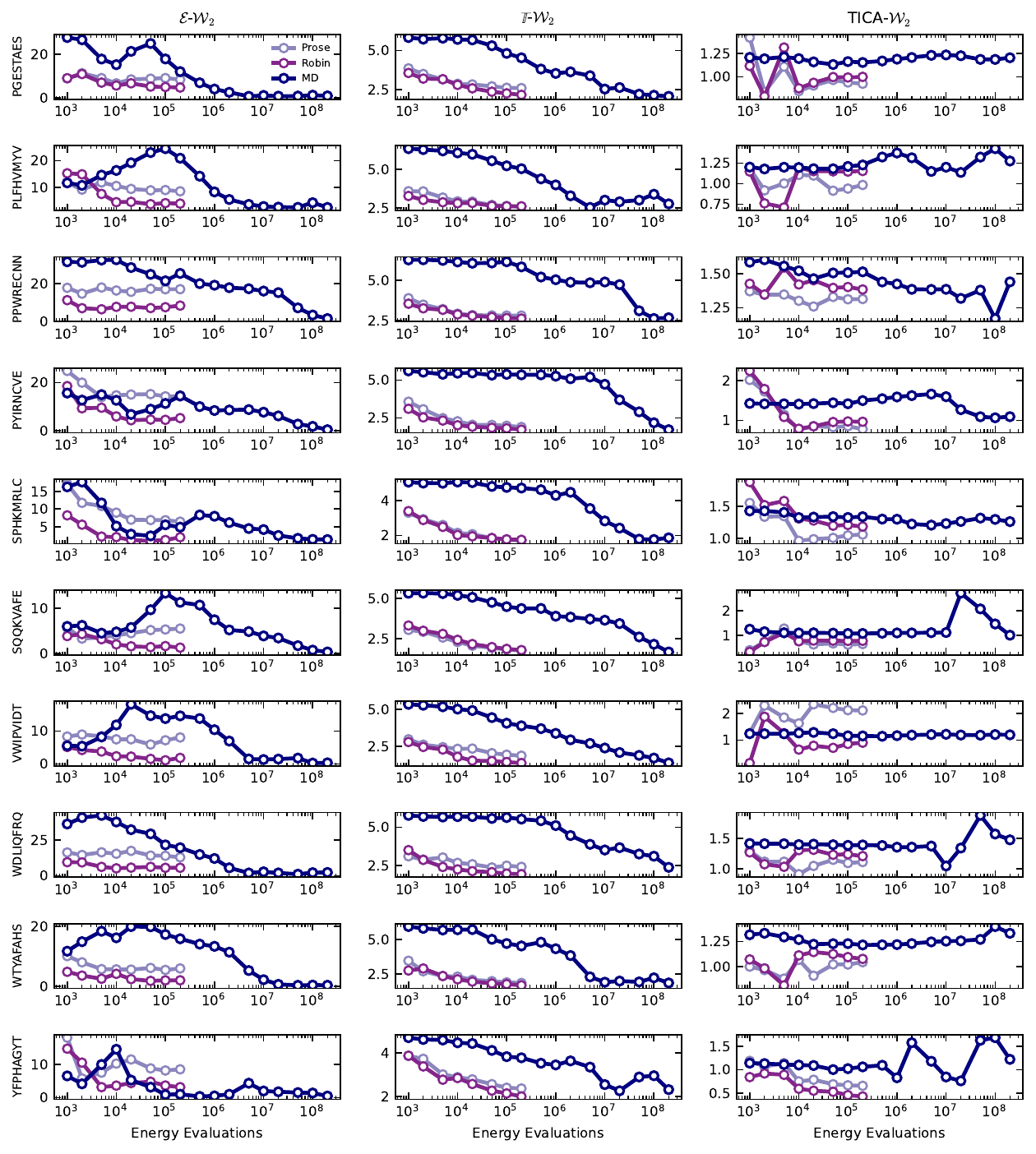}
  \caption{\emetric, \torusmetric, \ticametric per peptide  vs.\ Energy Evaluations.}
  \label{fig:perpeptide-scaling-3}
\end{figure*}

\newpage\clearpage

\section{Experimental Configurations}

\subsection{Architecture} The architecture choice employed follows a standard but performant recipe of Transformer-based building blocks. In~\cref{fig:architecture}, we capture the model specifications within a Transformer Block, which models the conditional distribution $p_{\theta}(x_j| x_{<j})$ and is composed of causal self-attention, RMSNorm~\citep{zhang2019rootmeansquarelayer}, and SwiGLU activations~\citep{shazeer2020gluvariantsimprovetransformer}. The two main differences between the set of single peptide experiments and the transferable setting were:\ (1) the scale of the model;\ and (2) the conditioning, which we cover in detail below. Unique to the molecular setting, we include an additional source of conditioning information through embeddings of the atom type and residue types that are injected into the main transformer block. We discuss the details behind all the enhancements and best practices below.

\begin{figure}[!h]
  \centering
  \includegraphics[width=0.45\textwidth]{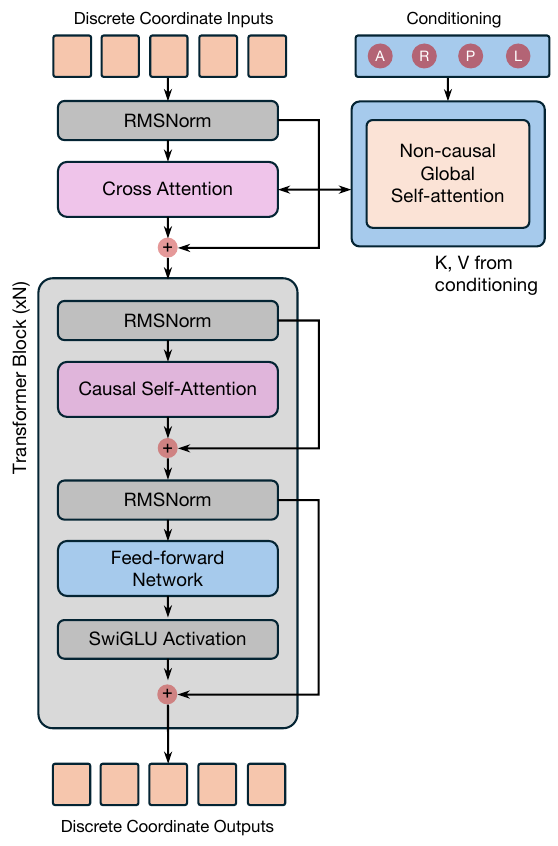}
  \caption{The transformer-based architecture variants that were considered. \textbf{Left}: The decoder-only like architecture that takes the conditioning information in at the initial generation step;\ \textbf{Right}: The encoder-decoder like architecture that repeatedly has a cross attention block that interacts with every transformer layer.}
  \label{fig:architecture}
\end{figure}

\subsection{Conditioning}
\label{app:arch_conditioning}
\looseness=-1l
In the transferable setting, we explore various conditioning strategies. In line with~\citet{tan2025amortizedsamplingtransferablenormalizing}, the conditioning information considers:\ (1) atom type, $A$;\ (2) residue type, $R$;\ (3) residue position, $P$; and (4) sequence length, $L$. To feed this conditioning information into the model, we consider the decoder-only architectures commonly employed in modern LLMs.

More specifically, in this variant, we pass the conditioning information through a separate transformer model with non-causal masking (token by token predictions should have access to global conditioning information). The representations obtained from the transformer are subsequently injected directly into the first layer of the larger transformer. For language, most conditional information is passed in at the beginning of the model as context, with additional passing at each layer excluded.

\subsection{Model Sizes}
\paragraph{\nameshort and \namemodel}
For all single peptide and transferable generation experiments, we concluded upon the model configurations reported in~\cref{tab:robin_sizes}. For ease of contrast, we include the configurations used for competing models:\ SBG and Prose in \cref{tab:prose_sizes}. In addition, although the model configurations appear identical between all alanine datasets for \nameshort and \namemodel, the difference in parameter count can be attributed to the number of bins used. As stated in~\cref{tab:arbg_configs}, for alanine dipeptide and tri-alanine, we use 4096 bins, while for alanine tetrapeptide and larger, we use 8192 bins. 

\begin{table}[!h]
\centering
\caption{SBG and Prose configurations across molecular systems \citep{tan_scalable_2025,tan2025amortizedsamplingtransferablenormalizing}.}
\label{tab:tarflow_sbg}
\begin{tabular}{lccccc}
\toprule
System & Layers / Block & Blocks & Channels & Parameters (M) \\
\toprule
SBG (Alanine Dipeptide)        & 4 & 4 & 256 & 13 \\
SBG (Tri-alanine)              & 6 & 6 & 256 & 29 \\
SBG (Alanine Tetrapeptide)     & 6 & 6 & 384 & 64 \\
SBG (Hexa-alanine)             & 6 & 6 & 384 & 64 \\
SBG (Chignolin)                & 8 & 8 & 384 & 114 \\
\midrule
Prose (Transferable)           & 8 & 8 & 384 & 285 \\
\bottomrule
\end{tabular}\label{tab:prose_sizes}
\end{table}
\begin{table}
\centering
\caption{\nameshort and \namemodel configurations for single system experiments and transferable sampling. Including the standard deviation of the centered training data and the bin width in picometers (1/100 of an Angstrom).}
\label{tab:arbg_configs}
\resizebox{\linewidth}{!}{
\begin{tabular}{lccccccccc}
\toprule
System & Heads & Head Dim.\ & Layers & Channels & Expansion & Parameters (M) & Bins & Std & Bin Width (pm)\\
\toprule
\nameshort (Alanine Dipeptide)     & 8 & 32 & 8 & 256 & 4 & 7.4 & 4096 & 0.163 & 0.358\\
\nameshort (Tri-alanine)           & 8 & 32 & 8 & 256 & 4 & 7.4 & 4096 & 0.210 & 0.461\\
\nameshort (Alanine Tetrapeptide)  & 8 & 32 & 8 & 256 & 4 & 10.5 & 8192 & 0.227 & 0.499 \\
\nameshort (Hexa-alanine)          & 8 & 32 & 8 & 256 & 4 & 10.5 & 8192 & 0.299 & 0.328 \\
\nameshort (Chignolin)             & 8 & 64 & 12 & 512 & 4 & 39.9 & 8192 & 0.345 & 0.379 \\
\midrule
\namemodel (Transferable)          & 8 & 64 & 16 & 768 & 4 & 132.0 & 8192 & 0.350 & 0.385 \\
\bottomrule
\end{tabular}\label{tab:robin_sizes}}
\end{table}

\looseness=-1
\xhdr{MoL/GMM-PixelCNN++ and GIVT}
For fair comparison, with \nameshort we inherit the same de-quantization strategy and architectural blocks as \nameshort when constructing the MoL/GMM-PixelCNN++ baselines. For all single peptide experiments, we set the bin count to $|B|=2048$. For GIVT, as it is a fully continuous model in the vein of a true Mixture Density Network, there is no de-quantization needed. In all three baselines, the model includes an additional output projection head that outputs the parameters of the mixture distribution and has the shape:
\begin{equation}
    \mathrm{output proj} = 3K+ 2,
\end{equation}
where $K$ is the number of mixtures, and for each mixture we output the means, scales, and logits over the mixture components. Lastly, we use two additional parameters as dependency coefficients to model the linear dependency coefficients for better modelling of correlated coordinates. In each case, the models are trained by computing the negative log likelihood under the mixture distribution. All remaining training settings are identical to the main model \nameshort for single peptide systems.

\subsection{Training}
\paragraph{Optimizer, Learning Rate, and Scheduler}
Following the recent success of the Muon optimizer in accelerating LLM training~\citep{jordan6muon}, we adopt it in our experiments. Muon is a momentum-based optimizer that applies Newton--Schulz orthogonalization to gradient updates. For weight matrices, it maintains an orthogonalized momentum buffer computed using five Newton--Schulz iterations, with Nesterov momentum ($\mu = 0.95$) and a learning rate of $0.02$. For one-dimensional parameters, such as biases and normalization layers, the optimizer falls back to AdamW~\citep{loshchilov2019decoupledweightdecayregularization} with a learning rate of $0.002$ and $(\beta_1, \beta_2) = (0.9, 0.999)$. We apply decoupled weight decay of $0.01$ to all parameters and combine the optimizer with a cosine learning rate schedule with a warm-up phase covering $5\%$ of the training iterations.

\paragraph{Flash Attention} To improve training efficiency and reduce memory consumption, we employ FlashAttention for all models~\citep{dao2023flashattention2}. FlashAttention computes the attention operation using fused kernels that significantly reduce the number of memory reads/writes by avoiding the explicit materialization of attention matrices. This substantially reduces memory overhead and improves throughput, enabling faster training and better hardware utilization without altering the underlying attention computation or model behaviour.

\paragraph{Lower precision training and inference} We tested training in both bf16 and float32. We found that for smaller models, bf16 performed equally well to float32 training. However, for larger models, and primarily \modelname, we observed that bf16 training sometimes resulted in training instability. We therefore chose to train \modelname using float32 precision. We also found that for inference, float32 inference was more consistent than bf16, which anecdotally had some numerical irregularities. 

\paragraph{Hardware} Training and inference was performed across multiple heterogeneous clusters containing a variety of NVIDIA GPUs. We primarily utilized L40S and RTX6000 Pro GPUs for inference and training of single-system \nameshort models and H100/H200 GPUs for training of \modelname.

\subsubsection{Single-system}

For the Chignolin scaling plot in \cref{fig:scaling}, we removed the scheduler, fixed the learning rate to $3 \times 10^{-4}$, and set the batch size to $256$ across all trained models to ensure a fair comparison.

\subsubsection{\namemodel training}

We train \namemodel using the same number of training steps as Prose, with a comparable batch size of 448 (7/8 of the Prose batch size 512). We use a learning rate $5 \times 10^{-4}$, with a cosine annealing schedule without weight decay.

\subsection{Inference}

\subsubsection{Autoregressive Twisted Sequential Monte Carlo}\label{sec:artsmc}

In this section, we detail our usage of SMC and the details on how it is applied in our peptide setting. In the ideal setting with a terminal reward $\mu_{\text{target}}(x)$, we would directly have access to the optimal intermediate density i.e.\

\begin{equation}
    \eta_j^\star \propto p_\theta(x_{\le j}) \psi_j^\star(x_{\le j}).
\end{equation}

with optimal twist functions:

\begin{equation}
    \psi_j^\star(x_{\le j}) \propto \sum_{x_{> j}} p_\theta(x_{>j} | x_{\le j}) \mu_{\text{target}}(x)
\end{equation}

However, these optimal twist functions are, in general, difficult to obtain. While many works attempt to learn them using a variety of objective,s including soft $Q$-learning~\citep{mudgal2023controlled}, noise contrastive estimation~\citep{lawson2022sixosmoothinginferencetwisted}, and classification~\citep{Yang_2021}, we already have a reasonable twist function using pre-defined energy functions. We approximate the twist function by the relative likelihood of a sample under the target energy function and our model. To encourage samples that are lower energy during our autoregressive generation.

We use the intermediate signals of a peptide-based energy function (either  Amber ff99SBildn or Amber 14, depending on the system)~\citep{tan_scalable_2025}. However, these peptide-based energy functions only function correctly on complete peptides. In the case of a partial peptide (for instance, the subset of atoms PPW in the peptide PPWRECNN), these atoms are not able to be processed by the energy function because they do not form a complete peptide, which has a C-terminus oxygen cap, often denoted as OXT, the terminal oxygen.

We are therefore not able to efficiently evaluate the partial energy of a subsequence like PPW. Instead, we generate one more atom, the nitrogen atom of the next residue, which can tell us a reasonable direction for the oxygen atom to go. Our procedure is then as follows:
\begin{enumerate}[topsep=0pt, partopsep=0pt, itemsep=3pt, parsep=0pt, leftmargin=*]
    \item Generate dimensions until we have generated a full residue plus the nitrogen atom of the next residue
    \item Find the direction of the nitrogen atom relative to the carbon atom its attached to
    \item Replace this nitrogen atom with an oxygen atom in the same direction off of the carbon atom, but at the optimal distance for carbon-oxygen bonds at $0.125$ nanometers.
    \item Evaluate the energy of this subset of the peptide using the relevant Amber energy function.
\end{enumerate}
This procedure allows us to calculate the intermediate twist functions, which then allows for resampling to improve the distribution during autoregressive sampling.

We perform SMC over the entire batch of samples for the best performance. This is much larger than can fit on a single GPU. Therefore, we need to operate in batches. We perform batched generation with KV caching per residue generated for minimal slow down. This has a couple of advantages over other generation procedures. First, we only need to hold a single gpu batch worth of Keys and Values at once. Second, we only need to regenerate the cache every residue. This means for a large batch of B samples, we need to regenerate the cache at most $(L - 1) B$ times, where $L$ is the length of the peptide. This minimizes the additional overhead of SMC to a few more model evaluations, which is less than a 10\% overhead.

\subsubsection{KV-caching} KV-caching is a technique used during inference that stores the attention keys and values from previous tokens during decoding to prevent recomputing them at every step~\citep{pope2023efficiently}. By caching these tensors, the model avoids redundant attention computations over the entire prefix, reducing per-token complexity and lowering latency at inference time. Consequently, we adopt it to reduce inference time cost when generating samples from the equilibrium distribution. 

\subsection{Table and Figure Specific Details}

\paragraph{\cref{tab:single-system-main}} ECNF++, RegFlow, SBG, FALCON-A, and FALCON results are taken from \citet{tan_scalable_2025} and \citet{rehman2025falconfewstepaccuratelikelihoods}. SBG here is SBG with Sequential Monte--Carlo sampling instead of self-normalized importance sampling (SNIS), as this performed slightly better. All other methods utilize SNIS. ECNF++ dashes represent models that were not run on that system due to scaling concerns.

\cut{
\paragraph{\cref{tab:transferable-main}} 
We evaluate against both Boltzmann Emulators and Generators. We note that the TimeWarp, BioEmu, and UniSim benchmarks are not trained the same ManyPeptidesMD dataset and are therefore at a disadvantage;\ we provide them for reference. We also note that both TimeWarp and ECNF were not evaluated on sequences with N-terminal proline due to absence in the TimeWarp training dataset. All other methods are trained on the ManyPeptidesMD dataset. Dashes represent methods that were not tested on that sized system due to scaling issues. 
}

\end{document}

%% file: math_commands.tex
\usepackage{amsthm}
\usepackage{xcolor}
\usepackage{amsmath,amsfonts,bm}
\makeatletter
\newtheorem*{rep@theorem}{\rep@title}
\newcommand{\newreptheorem}[2]{%
\newenvironment{rep#1}[1]{%
 \def\rep@title{#2 \ref{##1}}%
 \begin{rep@theorem}}%
 {\end{rep@theorem}}}
\makeatother

\newreptheorem{theorem}{Theorem}

\definecolor{myred}{RGB}{215,48,39}
\definecolor{mygreen}{RGB}{26,152,80}

\newcommand{\halfmark}{\textcolor{gray}{\checkmark\kern-1.1ex\raisebox{.7ex}{\rotatebox[origin=c]{125}{--}}}}
\usepackage[framemethod=TikZ]{mdframed}
\mdfdefinestyle{MyFrame}{%
    linecolor=black,
    outerlinewidth=.3pt,
    roundcorner=5pt,
    innertopmargin=1pt, 
    innerbottommargin=1pt, 
    innerrightmargin=1pt,
    innerleftmargin=1pt,
    backgroundcolor=black!0!white}

\mdfdefinestyle{MyFrame2}{%
    linecolor=white,
    outerlinewidth=1pt,
    roundcorner=2pt,
    innertopmargin=5pt,
    innerbottommargin=5pt,
    innerrightmargin=5pt,
    innerleftmargin=5pt,
    backgroundcolor=black!10!white}

\mdfdefinestyle{MyFrameEq}{%
    linecolor=white,
    outerlinewidth=0pt,
    roundcorner=0pt,
    innertopmargin=0pt,
    innerbottommargin=0pt,
    innerrightmargin=7pt,
    innerleftmargin=7pt,
    backgroundcolor=black!3!white}

\newcommand{\RNum}[1]{\uppercase\expandafter{\romannumeral #1\relax}}

\newcommand{\R}{\mathcal{R}}

\newcommand{\vertiii}[1]{{\left\vert\kern-0.25ex\left\vert\kern-0.25ex\left\vert #1 
    \right\vert\kern-0.25ex\right\vert\kern-0.25ex\right\vert}}
\newcommand{\vertiiii}[1]{{\vert\kern-0.25ex\vert\kern-0.25ex\vert #1 
    \vert\kern-0.25ex\vert\kern-0.25ex\vert}}

\usepackage{mathtools}

\newcommand{\xhdr}[1]{{\noindent\bfseries #1}.}
\renewcommand{\paragraph}[1]{{\noindent\bfseries #1}.}
\newcommand{\cut}[1]{}

\newcommand{\removelatexerror}{\let\@latex@error\@gobble}

\def\eqref#1{Eq.~\ref{#1}}

\def\1{\bm{1}}

\DeclareMathAlphabet{\mathsfit}{\encodingdefault}{\sfdefault}{m}{sl}
\SetMathAlphabet{\mathsfit}{bold}{\encodingdefault}{\sfdefault}{bx}{n}

\def\gE{{\mathcal{E}}}

\def\gI{{\mathcal{I}}}

\def\gT{{\mathcal{T}}}

\def\gW{{\mathcal{W}}}

\def\sE{{\mathbb{E}}}

\def\R{{\mathbb{R}}}

\newcommand{\KL}{\mathbb{D}_{\mathrm{KL}}}



%% file: main.bib
@article{murtada2025mdllm1largelanguagemodel,
      title={MD-LLM-1: A Large Language Model for Molecular Dynamics}, 
      author={Mhd Hussein Murtada and Z. Faidon Brotzakis and Michele Vendruscolo},
      year={2025},
      journal={arXiv}
}

@article {Billera2024,
	author = {Billera, Lukas and Oresten, Anton and St{\r a}lmarck, Aron and Sato, Kenta and Kaduk, Mateusz and Murrell, Ben},
	title = {The Continuous Language of Protein Structure},
	elocation-id = {2024.05.11.593685},
	year = {2024},
	doi = {10.1101/2024.05.11.593685},
	publisher = {Cold Spring Harbor Laboratory},
	journal = {bioRxiv}
}

@misc{nextstepteam2025nextstep1autoregressiveimagegeneration,
      title={NextStep-1: Toward Autoregressive Image Generation with Continuous Tokens at Scale}, 
      author={NextStep Team and Chunrui Han and Guopeng Li and Jingwei Wu and Quan Sun and Yan Cai and Yuang Peng and Zheng Ge and Deyu Zhou and Haomiao Tang and Hongyu Zhou and Kenkun Liu and Ailin Huang and Bin Wang and Changxin Miao and Deshan Sun and En Yu and Fukun Yin and Gang Yu and Hao Nie and Haoran Lv and Hanpeng Hu and Jia Wang and Jian Zhou and Jianjian Sun and Kaijun Tan and Kang An and Kangheng Lin and Liang Zhao and Mei Chen and Peng Xing and Rui Wang and Shiyu Liu and Shutao Xia and Tianhao You and Wei Ji and Xianfang Zeng and Xin Han and Xuelin Zhang and Yana Wei and Yanming Xu and Yimin Jiang and Yingming Wang and Yu Zhou and Yucheng Han and Ziyang Meng and Binxing Jiao and Daxin Jiang and Xiangyu Zhang and Yibo Zhu},
      year={2025},
      eprint={2508.10711},
      archivePrefix={arXiv},
      primaryClass={cs.CV},
}

@inproceedings{gloy2025hollowflowefficientsamplelikelihood,
      title={HollowFlow: Efficient Sample Likelihood Evaluation using Hollow Message Passing}, 
      author={Johann Flemming Gloy and Simon Olsson},
      year={2025},
      booktitle={Neural Information Processing Systems (NeurIPS)}
}

@inproceedings{li2024autoregressiveimagegenerationvector,
      title={Autoregressive Image Generation without Vector Quantization}, 
      author={Tianhong Li and Yonglong Tian and He Li and Mingyang Deng and Kaiming He},
      year={2024},
      booktitle={Neural Information Processing Systems}
}

@inproceedings{rehman2025falconfewstepaccuratelikelihoods,
      title={FALCON: Few-step Accurate Likelihoods for Continuous Flows}, 
      author={Danyal Rehman and Tara Akhound-Sadegh and Artem Gazizov and Yoshua Bengio and Alexander Tong},
      year={2026},
      booktitle={International Conference on Learning Representations (ICLR)}
}

@inproceedings{tschannen2024givtgenerativeinfinitevocabularytransformers,
      title={GIVT: Generative Infinite-Vocabulary Transformers}, 
      author={Michael Tschannen and Cian Eastwood and Fabian Mentzer},
      year={2024},
      booktitle={ECCV}
}

@article{OLSSON2026103213,
title = {Generative molecular dynamics},
journal = {Current Opinion in Structural Biology},
year = {2026},
author = {Simon Olsson}
}

@article{rahman1964correlations,
  title={Correlations in the motion of atoms in liquid argon},
  author={Rahman, Aneesur},
  journal={Physical review},
  volume={136},
  number={2A},
  pages={A405},
  year={1964},
  publisher={APS}
}

@article{alder1959studies,
  title={Studies in molecular dynamics. I. General method},
  author={Alder, Berni J and Wainwright, Thomas Everett},
  journal={The Journal of Chemical Physics},
  volume={31},
  number={2},
  pages={459--466},
  year={1959},
  publisher={American Institute of Physics}
}

@article{henin2022enhanced,
  title={Enhanced sampling methods for molecular dynamics simulations},
  author={H{\'e}nin, J{\'e}r{\^o}me and Leli{\`e}vre, Tony and Shirts, Michael R and Valsson, Omar and Delemotte, Lucie},
  journal={arXiv preprint arXiv:2202.04164},
  year={2022}
}

@misc{cheng2025scalableautoregressive3dmolecule,
      title={Scalable Autoregressive 3D Molecule Generation}, 
      author={Austin H. Cheng and Chong Sun and Alán Aspuru-Guzik},
      year={2025},
      eprint={2505.13791},
}

@article{syed_nonreversible_2021,
    author = {Syed, Saifuddin and Bouchard-Côté, Alexandre and Deligiannidis, George and Doucet, Arnaud},
    title = {Non-Reversible Parallel Tempering: A Scalable Highly Parallel MCMC Scheme},
    journal = {Journal of the Royal Statistical Society Series B: Statistical Methodology},
    volume = {84},
    number = {2},
    pages = {321-350},
    year = {2021},
}

@inproceedings{kapuśniak2025marsfmgenerativemodelingmolecular,
      title={MarS-FM: Generative Modeling of Molecular Dynamics via Markov State Models}, 
      author={Kacper Kapuśniak and Cristian Gabellini and Michael Bronstein and Prudencio Tossou and Francesco Di Giovanni},
      year={2026},
      booktitle={International Conference on Learning Representations (ICLR)}
}

@article{
2,
  title={DPLM-2: A Multimodal Diffusion Protein Language Model},
  author={Xinyou Wang and Zaixiang Zheng and Fei Ye and Dongyu Xue and Shujian Huang and Quanquan Gu},
  journal={ArXiv},
  year={2024},
  volume={abs/2410.13782},
  url={https://api.semanticscholar.org/CorpusID:273403705}
}

@article{Bannwarth2018GFN2xTBAnAA,
  title={GFN2-xTB-An Accurate and Broadly Parametrized Self-Consistent Tight-Binding Quantum Chemical Method with Multipole Electrostatics and Density-Dependent Dispersion Contributions.},
  author={Christoph Bannwarth and Sebastian Ehlert and Stefan Grimme},
  journal={Journal of chemical theory and computation},
  year={2018},
  volume={15 3},
  pages={
          1652-1671
        }
}

@inproceedings{salimans2017pixelcnn++,
  title={Pixelcnn++: Improving the pixelcnn with discretized logistic mixture likelihood and other modifications},
  author={Salimans, Tim and Karpathy, Andrej and Chen, Xi and Kingma, Diederik P},
  booktitle={International Conference on Learning Representations (ICLR)},
  year={2017}
}

@misc{owen2013monte,
  title={Monte Carlo theory, methods and examples},
  author={Owen, Art B},
  year={2013},
  publisher={Stanford}
}

@inproceedings{yu2025unisimunifiedsimulatortimecoarsened,
	title        = {UniSim: A Unified Simulator for Time-Coarsened Dynamics of Biomolecules},
	author       = {Ziyang Yu and Wenbing Huang and Yang Liu},
	year         = 2025,
	booktitle = {International Conference on Machine Learning (ICML)}
}

@article{noe2019boltzmann,
	title        = {Boltzmann generators: Sampling equilibrium states of many-body systems with deep learning},
	author       = {No{\'e}, Frank and Olsson, Simon and K{\"o}hler, Jonas and Wu, Hao},
	year         = 2019,
	journal      = {Science},
	publisher    = {American Association for the Advancement of Science}
}

@inproceedings{klein2023timewarp,
	title        = {Timewarp: Transferable acceleration of molecular dynamics by learning time-coarsened dynamics},
	author       = {Klein, Leon and Foong, Andrew and Fjelde, Tor and Mlodozeniec, Bruno and Brockschmidt, Marc and Nowozin, Sebastian and No{\'e}, Frank and Tomioka, Ryota},
	year         = 2023,
	booktitle      = {Neural Information Processing Systems (NeurIPS)}
}

@inproceedings{rehman2025fortforwardonlyregressiontraining,
	title        = {Efficient Regression-Based Training of Normalizing Flows for Boltzmann Generators},
	author       = {Danyal Rehman and Oscar Davis and Jiarui Lu and Jian Tang and Michael Bronstein and Yoshua Bengio and Alexander Tong and Avishek Joey Bose},
	year         = 2026,
	booktitle      = {International Conference on Learning Representations (ICLR)}
}

@inproceedings{tan2025amortizedsamplingtransferablenormalizing,
	title        = {Amortized Sampling with Transferable Normalizing Flows},
	author       = {Charlie B. Tan and Majdi Hassan and Leon Klein and Saifuddin Syed and Dominique Beaini and Michael M. Bronstein and Alexander Tong and Kirill Neklyudov},
	year         = 2025,
	booktitle    = {Neural Information Processing Systems (NeurIPS)}
}

@inproceedings{akhoundsadegh2025progressiveinferencetimeannealingdiffusion,
	title        = {Progressive Inference-Time Annealing of Diffusion Models for Sampling from Boltzmann Densities},
	author       = {Tara Akhound-Sadegh and Jungyoon Lee and Avishek Joey Bose and Valentin De Bortoli and Arnaud Doucet and Michael M. Bronstein and Dominique Beaini and Siamak Ravanbakhsh and Kirill Neklyudov and Alexander Tong},
	year         = 2025,
	booktitle    = {Neural Information Processing Systems (NeurIPS)}
}

@article{albergo_building_2023,
	title        = {Building Normalizing Flows with Stochastic Interpolants},
	author       = {Albergo, Michael S. and {Vanden-Eijnden}, Eric},
	year         = 2023,
	journal      = {International Conference on Learning Representations (ICLR)}
}

@article{chen_neural_2018,
	title        = {Neural Ordinary Differential Equations},
	author       = {Chen, Ricky T. Q. and Rubanova, Yulia and Bettencourt, Jesse and Duvenaud, David K},
	year         = 2018,
	journal      = {Neural Information Processing Systems (NeurIPS)}
}

@inproceedings{dao2023flashattention2,
	title        = {Flash{A}ttention-2: Faster Attention with Better Parallelism and Work Partitioning},
	author       = {Dao, Tri},
	year         = 2024,
	booktitle    = {International Conference on Learning Representations (ICLR)}
}

@article{del2006sequential,
	title        = {Sequential {Monte Carlo} samplers},
	author       = {Del Moral, Pierre and Doucet, Arnaud and Jasra, Ajay},
	year         = 2006,
	journal      = {Journal of the Royal Statistical Society Series B: Statistical Methodology}
}

@article{dinh2016density,
	title        = {Density estimation using {Real NVP}},
	author       = {Dinh, Laurent and Sohl-Dickstein, Jascha and Bengio, Samy},
	year         = 2017,
	journal      = {International Conference on Learning Representations (ICLR)}
}

@book{doucet2001sequential,
	title        = {Sequential Monte Carlo methods in practice},
	author       = {Doucet, Arnaud and De Freitas, Nando and Gordon, Neil James and others},
	year         = 2001,
    publisher={Springer},
}

@inproceedings{zhao2024probabilisticinferencelanguagemodels,
      title={Probabilistic Inference in Language Models via Twisted Sequential Monte Carlo}, 
      author={Stephen Zhao and Rob Brekelmans and Alireza Makhzani and Roger Grosse},
      year={2024},
      booktitle={International Conference on Machine Learning}
}

@inproceedings{loshchilov2019decoupledweightdecayregularization,
      title={Decoupled Weight Decay Regularization}, 
      author={Ilya Loshchilov and Frank Hutter},
      year={2019},
      booktitle={International Conference on Representation Learning (ICLR)}
}

@article{hochbruck2020convergence,
title={On the convergence of Lawson methods for semilinear stiff problems},
author={Hochbruck, Marlis and Leibold, Jan and Ostermann, Alexander},
journal={Numerische Mathematik},
year={2020}
}

@article{hochbruck2010exponential,
title={Exponential integrators},
author={Hochbruck, Marlis and Ostermann, Alexander},
journal={Acta Numerica},
year={2010}
}

@article{Hjorth_Larsen_2017,
doi = {10.1088/1361-648X/aa680e},
year = {2017},
publisher = {IOP Publishing},
volume = {29},
number = {27},
pages = {273002},
author = {Hjorth Larsen, Ask and Jørgen Mortensen, Jens and Blomqvist, Jakob and Castelli, Ivano E and Christensen, Rune and Dułak, Marcin and Friis, Jesper and Groves, Michael N and Hammer, Bjørk and Hargus, Cory and Hermes, Eric D and Jennings, Paul C and Bjerre Jensen, Peter and Kermode, James and Kitchin, John R and Leonhard Kolsbjerg, Esben and Kubal, Joseph and Kaasbjerg, Kristen and Lysgaard, Steen and Bergmann Maronsson, Jón and Maxson, Tristan and Olsen, Thomas and Pastewka, Lars and Peterson, Andrew and Rostgaard, Carsten and Schiøtz, Jakob and Schütt, Ole and Strange, Mikkel and Thygesen, Kristian S and Vegge, Tejs and Vilhelmsen, Lasse and Walter, Michael and Zeng, Zhenhua and Jacobsen, Karsten W},
title = {The atomic simulation environment—a Python library for working with atoms},
journal = {Journal of Physics: Condensed Matter},
}

@article{eastman2024openmm,
author = {Eastman, Peter and Galvelis, Raimondas and Peláez, Raúl P. and Abreu, Charlles R. A. and Farr, Stephen E. and Gallicchio, Emilio and Gorenko, Anton and Henry, Michael M. and Hu, Frank and Huang, Jing and Krämer, Andreas and Michel, Julien and Mitchell, Joshua A. and Pande, Vijay S. and Rodrigues, João PGLM and Rodriguez-Guerra, Jaime and Simmonett, Andrew C. and Singh, Sukrit and Swails, Jason and Turner, Philip and Wang, Yuanqing and Zhang, Ivy and Chodera, John D. and De Fabritiis, Gianni and Markland, Thomas E.},
title = {OpenMM 8: Molecular Dynamics Simulation with Machine Learning Potentials},
journal = {The Journal of Physical Chemistry B},
volume = {128},
number = {1},
pages = {109-116},
year = {2024},
doi = {10.1021/acs.jpcb.3c06662},
}

@article{klein2023equivariant,
	title        = {Equivariant flow matching},
	author       = {Leon Klein and Andreas Krämer and Frank Noé},
	year         = 2023,
	journal      = {Neural Information Processing Systems (NeurIPS)}
}

@article{kohler2020equivariant,
	title        = {Equivariant flows: exact likelihood generative learning for symmetric densities},
	author       = {K{\"o}hler, Jonas and Klein, Leon and No{\'e}, Frank},
	year         = 2020,
	journal      = {International Conference on Machine Learning (ICML)}
}

@article{lipman_flow_2022,
	title        = {Flow Matching for Generative Modeling},
	author       = {Lipman, Yaron and Chen, Ricky T. Q. and {Ben-Hamu}, Heli and Nickel, Maximilian and Le, Matt},
	year         = 2023,
	journal      = {International Conference on Learning Representations (ICLR)}
}

@article{liu_rectified_2022,
	title        = {Rectified Flow: A Marginal Preserving Approach to Optimal Transport},
	author       = {Liu, Qiang},
	year         = 2022,
	journal      = {arXiv}
}

@article{matsumoto2002molecular,
	title        = {Molecular dynamics simulation of the ice nucleation and growth process leading to water freezing},
	author       = {Matsumoto, Masakazu and Saito, Shinji and Ohmine, Iwao},
	year         = 2002,
	journal      = {Nature}
}

@article{midgley_se3_2023,
	title        = {{SE(3)} Equivariant Augmented Coupling Flows},
	author       = {Midgley, Laurence I and Stimper, Vincent and Antor{\'a}n, Javier and Mathieu, Emile and Sch{\"o}lkopf, Bernhard and Hern{\'a}ndez-Lobato, Jos{\'e} Miguel},
	year         = 2023,
	journal      = {Neural Information Processing Systems (NeurIPS)}
}

@inproceedings{boffi2025buildconsistencymodellearning,
	title        = {How to build a consistency model: Learning flow maps via self-distillation},
	author       = {Nicholas M. Boffi and Michael S. Albergo and Eric Vanden-Eijnden},
	year         = 2025,
	booktitle = {Neural Information Processing Systems (NeurIPS)}
}

@article{ramachandran1963stereochemistry,
	title        = {Stereochemistry of polypeptide chain configurations},
	author       = {Ramachandran, G N and Ramakrishnan, C and Sasisekharan, V},
	year         = 1963,
	journal      = {Journal of Molecular Biology}
}

@article{rezende2015variational,
	title        = {Variational inference with normalizing flows},
	author       = {Rezende, Danilo and Mohamed, Shakir},
	year         = 2015,
	journal      = {International Conference on Machine Learning (ICML)}
}

@article{rizzi2021targeted,
	title        = {Targeted free energy perturbation revisited: Accurate free energies from mapped reference potentials},
	author       = {Rizzi, Andrea and Carloni, Paolo and Parrinello, Michele},
	year         = 2021,
	journal      = {The journal of physical chemistry letters}
}

@inproceedings{
mudgal2023controlled,
title={Controlled Decoding from Language Models},
author={Sidharth Mudgal and Jong Lee and Harish Ganapathy and YaGuang Li and Tao Wang and Yanping Huang and Zhifeng Chen and Heng-Tze Cheng and Michael Collins and Jilin Chen and Alex Beutel and Ahmad Beirami},
booktitle={International Conference on Machine Learning (ICML)},
year={2024}
}

@inproceedings{lawson2022sixosmoothinginferencetwisted,
      title={SIXO: Smoothing Inference with Twisted Objectives}, 
      author={Dieterich Lawson and Allan Raventós and Andrew Warrington and Scott Linderman},
      year={2022},
      booktitle={Neural Information Processing Systems (NeurIPS)}
}

@inproceedings{Yang_2021,
   title={FUDGE: Controlled Text Generation With Future Discriminators},
   booktitle={Proceedings of the 2021 Conference of the North American Chapter of the Association for Computational Linguistics: Human Language Technologies},
   publisher={Association for Computational Linguistics},
   author={Yang, Kevin and Klein, Dan},
   year={2021},
   pages={3511–3535} }

@inproceedings{gebauer_gschnet_2019,
title = {Symmetry-adapted generation of 3d point sets for the targeted discovery of molecules},
author = {Gebauer, Niklas and Gastegger, Michael and Sch\"{u}tt, Kristof},
booktitle = {Neural Information Processing Systems (NeurIPS)},
year = {2019},
}

@article { ObstaclestoHighDimensionalParticleFiltering,
      author = "Chris Snyder and Thomas Bengtsson and Peter Bickel and Jeff Anderson",
      title = "Obstacles to High-Dimensional Particle Filtering",
      journal = "Monthly Weather Review",
      year = "2008",
      publisher = "American Meteorological Society",
      volume = "136",
      number = "12",
      doi = "10.1175/2008MWR2529.1",
      pages=      "4629 - 4640",
}

@inproceedings{blessing2024elboslargescaleevaluationvariational,
      title={Beyond ELBOs: A Large-Scale Evaluation of Variational Methods for Sampling}, 
      author={Denis Blessing and Xiaogang Jia and Johannes Esslinger and Francisco Vargas and Gerhard Neumann},
      year={2024},
      booktitle={International Conference on Machine Learning (ICML)}
}

@inproceedings{vonklitzing2025learningboltzmanngeneratorsconstrained,
      title={Learning Boltzmann Generators via Constrained Mass Transport}, 
      author={Christopher von Klitzing and Denis Blessing and Henrik Schopmans and Pascal Friederich and Gerhard Neumann},
      year={2025},
      booktitle={International Conference on Learning Representations (ICLR)}
}

@misc{shazeer2020gluvariantsimprovetransformer,
      title={GLU Variants Improve Transformer}, 
      author={Noam Shazeer},
      year={2020},
}

@inproceedings{zhang2019rootmeansquarelayer,
      title={Root Mean Square Layer Normalization}, 
      author={Biao Zhang and Rico Sennrich},
      year={2019},
      booktitle={Neural Information Processing Systems (NeurIPS)}
}

@article{wirnsberger2020targeted,
	title        = {Targeted free energy estimation via learned mappings},
	author       = {Wirnsberger, Peter and Ballard, Andrew J and Papamakarios, George and Abercrombie, Stuart and Racani{\`e}re, S{\'e}bastien and Pritzel, Alexander and Jimenez Rezende, Danilo and Blundell, Charles},
	year         = 2020,
	journal      = {J. Chem. Phys.}
}

@article{zhai2024normalizing,
	title        = {Normalizing flows are capable generative models},
	author       = {Zhai, Shuangfei and Zhang, Ruixiang and Nakkiran, Preetum and Berthelot, David and Gu, Jiatao and Zheng, Huangjie and Chen, Tianrong and Bautista, Miguel Angel and Jaitly, Navdeep and Susskind, Josh},
	year         = 2025,
	journal      = {International Conference on Learning Representations (ICLR)}
}

@article{tabak2010density,
	title        = {Density estimation by dual ascent of the log-likelihood},
	author       = {Tabak, Esteban G and Vanden-Eijnden, Eric},
	year         = 2010,
	journal      = {Communications in Mathematical Sciences}
}

@article{noe_constructing_2009,
	title        = {Constructing the equilibrium ensemble of folding pathways from short off-equilibrium simulations},
	author       = {Noé, Frank and Schütte, Christof and Vanden-Eijnden, Eric and Reich, Lothar and Weikl, Thomas R.},
	year         = 2009,
	journal      = {Proceedings of the National Academy of Sciences},
}

@article{lindorff-larsen_how_2011,
	title        = {How {Fast}-{Folding} {Proteins} {Fold}},
	author       = {Lindorff-Larsen, Kresten and Piana, Stefano and Dror, Ron O. and Shaw, David E.},
	year         = 2011,
	journal      = {Science}
}

@inproceedings{tan_scalable_2025,
	title        = {Scalable {Equilibrium} {Sampling} with {Sequential} {Boltzmann} {Generators}},
	author       = {Tan, Charlie B. and Bose, Avishek Joey and Lin, Chen and Klein, Leon and Bronstein, Michael M. and Tong, Alexander},
	year         = 2025,
	booktitle = {International Conference on Learning Representations (ICLR)}
}

@inproceedings{schopmans_temperature-annealed_2025,
	title        = {Temperature-{Annealed} {Boltzmann} {Generators}},
	author       = {Schopmans, Henrik and Friederich, Pascal},
	year         = 2025,
	booktitle = {International Conference on Machine Learning (ICML)}
}

@inproceedings{klein_transferable_2024,
	title        = {Transferable Boltzmann Generators},
	author       = {Klein, Leon and Noe, Frank},
	year         = 2024,
	booktitle      = {Neural Information Processing Systems (NeurIPS)}
}

@misc{peluchetti2021,
	title        = {Non-Denoising Forward-Time Diffusions},
	author       = {Stefano Peluchetti},
	year         = 2021
}

@inproceedings{aggarwal2025boltznce,
	title        = {{BoltzNCE}: Learning Likelihoods for Boltzmann Generation with Stochastic Interpolants and Noise Contrastive Estimation},
	author       = {Aggarwal, Rishal and Chen, Jacky and Boffi, Nicholas M and Koes, David Ryan},
	year         = 2025,
	booktitle      = {Neural Information Processing Systems (NeurIPS)}
}

@article{kish1957confidence,
	title        = {Confidence intervals for clustered samples},
	author       = {Kish, Leslie},
	year         = 1957,
	journal      = {American Sociological Review},
	publisher    = {JSTOR}
}

@inproceedings{draxler2024free,
	title        = {Free-form flows: Make any architecture a normalizing flow},
	author       = {Draxler, Felix and Sorrenson, Peter and Zimmermann, Lea and Rousselot, Armand and K{\"o}the, Ullrich},
	year         = 2024,
	booktitle    = {International Conference on Artificial Intelligence and Statistics (AISTATS)},
}

@inproceedings{meyer2021hutch++,
  title={Hutch++: Optimal stochastic trace estimation},
  author={Meyer, Raphael A and Musco, Cameron and Musco, Christopher and Woodruff, David P},
  booktitle={Symposium on Simplicity in Algorithms (SOSA)},
  pages={142--155},
  year={2021},
  organization={SIAM}
}

@techreport{bishop1994mixture,
  title={Mixture density networks},
  author={Bishop, Christopher M},
  year={1994},
  institution={Aston University}
}

@article{parrinello1980crystal,
	title        = {Crystal structure and pair potentials: A molecular-dynamics study},
	author       = {Parrinello, Michele and Rahman, Aneesur},
	year         = {1980},
	journal      = {Physical review letters},
	volume       = {45},
	number       = {14},
	pages        = {1196}
}

@article{yang1952spontaneous,
  title={The spontaneous magnetization of a two-dimensional Ising model},
  author={Yang, Chen Ning},
  journal={Physical Review},
  volume={85},
  number={5},
  pages={808},
  year={1952},
  publisher={APS}
}

@article{perez2025breaking,
  title={Breaking the mold: Overcoming the time constraints of molecular dynamics on general-purpose hardware},
  author={Perez, Danny and Thompson, Aidan and Moore, Stan and Oppelstrup, Tomas and Sharapov, Ilya and Santos, Kylee and Sharifian, Amirali and Kalchev, Delyan Z and Schreiber, Robert and Pakin, Scott and others},
  journal={The Journal of Chemical Physics},
  volume={162},
  number={7},
  year={2025},
  publisher={AIP Publishing}
}

@inproceedings{cornish2020relaxing,
  title={Relaxing bijectivity constraints with continuously indexed normalising flows},
  author={Cornish, Rob and Caterini, Anthony and Deligiannidis, George and Doucet, Arnaud},
  booktitle={International conference on machine learning (ICML)},
  year={2020}
}

@book{runde2005taste,
  title={A taste of topology},
  author={Runde, Volker and Ribet, KA and Axler, S},
  year={2005},
  publisher={Springer}
}

@inproceedings{dupont2019augmented,
  title={Augmented neural odes},
  author={Dupont, Emilien and Doucet, Arnaud and Teh, Yee Whye},
  booktitle={Neural Information Processing Systems (NeurIPS)},
  year={2019}
}

@misc{jordan6muon,
  title={Muon: An optimizer for hidden layers in neural networks},
  author={Jordan, Keller and Jin, Yuchen and Boza, Vlado and Jiacheng, You and Cecista, Franz and Newhouse, Laker and Bernstein, Jeremy},
  url={https://kellerjordan. github.io/posts/muon},
  year={2024}
}

@article{pope2023efficiently,
  title={Efficiently scaling transformer inference},
  author={Pope, Reiner and Douglas, Sholto and Chowdhery, Aakanksha and Devlin, Jacob and Bradbury, James and Heek, Jonathan and Xiao, Kefan and Agrawal, Shivani and Dean, Jeff},
  journal={Proceedings of machine learning and systems},
  volume={5},
  pages={606--624},
  year={2023}
}

@inproceedings{dosovitskiy2020image,
  title={An image is worth 16x16 words: Transformers for image recognition at scale},
  author={Dosovitskiy, Alexey and
                  Lucas Beyer and
                  Alexander Kolesnikov and
                  Dirk Weissenborn and
                  Xiaohua Zhai and
                  Thomas Unterthiner and
                  Mostafa Dehghani and
                  Matthias Minderer and
                  Georg Heigold and
                  Sylvain Gelly and
                  Jakob Uszkoreit and
                  Neil Houlsby},
  booktitle={International Conference on Learning Representations (ICLR)},
  year={2021}
}

@article{comanici2025gemini,
  title={Gemini 2.5: Pushing the frontier with advanced reasoning, multimodality, long context, and next generation agentic capabilities},
  author={Comanici, Gheorghe and Bieber, Eric and Schaekermann, Mike and Pasupat, Ice and Sachdeva, Noveen and Dhillon, Inderjit and Blistein, Marcel and Ram, Ori and Zhang, Dan and Rosen, Evan and others},
  journal={arXiv preprint arXiv:2507.06261},
  year={2025}
}

@article{deng2022deep,
  title={Deep bingham networks: Dealing with uncertainty and ambiguity in pose estimation},
  author={Deng, Haowen and Bui, Mai and Navab, Nassir and Guibas, Leonidas and Ilic, Slobodan and Birdal, Tolga},
  journal={International Journal of Computer Vision},
  volume={130},
  number={7},
  pages={1627--1654},
  year={2022},
  publisher={Springer}
}

@article{razavi2020frmdn,
  title={FRMDN: Flow-based Recurrent Mixture Density Network},
  author={Razavi, Seyedeh Fatemeh and Hosseini, Reshad and Behzad, Tina},
  journal = {Expert Systems with Applications},
  year = {2024}
}

@inproceedings{ha2017neural,
  title={A neural representation of sketch drawings},
  author={Ha, David and Eck, Douglas},
  booktitle={International Conference on Learning Representations (ICLR)},
  year={2018}
}

@article{ha2018world,
  title={World models},
  author={Ha, David and Schmidhuber, J{\"u}rgen},
  journal={arXiv preprint arXiv:1803.10122},
  volume={2},
  number={3},
  year={2018}
}

@article{li2023xrmdn,
  author={Li, Xiaoming and Normandin-Taillon, Hubert and Wang, Chun and Huang, Xiao},
  journal={IEEE Transactions on Intelligent Vehicles}, 
  title={XRMDN: An Extended Recurrent Mixture Density Network for Short-Term Probabilistic Rider Demand Forecasting Considering High Volatility}, 
  year={2025}
}

@article{yoo2016improved,
  title={Improved parameterization of amine--carboxylate and amine--phosphate interactions for molecular dynamics simulations using the CHARMM and AMBER force fields},
  author={Yoo, Jejoong and Aksimentiev, Aleksei},
  journal={Journal of chemical theory and computation},
  volume={12},
  number={1},
  pages={430--443},
  year={2016},
  publisher={ACS Publications}
}

@article{jing2019polarizable,
  title={Polarizable force fields for biomolecular simulations: Recent advances and applications},
  author={Jing, Zhifeng and Liu, Chengwen and Cheng, Sara Y and Qi, Rui and Walker, Brandon D and Piquemal, Jean-Philip and Ren, Pengyu},
  journal={Annual Review of biophysics},
  volume={48},
  number={1},
  pages={371--394},
  year={2019},
  publisher={Annual Reviews}
}

@inproceedings{geng2025mean,
  title={Mean flows for one-step generative modeling},
  author={Geng, Zhengyang and Deng, Mingyang and Bai, Xingjian and Kolter, J Zico and He, Kaiming},
  booktitle={Neural Information Processing Systems (NeurIPS)},
  year={2025}
}

@inproceedings{douc2005comparison,
  title={Comparison of resampling schemes for particle filtering},
  author={Douc, Randal and Capp{\'e}, Olivier},
  booktitle={ISPA 2005. Proceedings of the 4th International Symposium on Image and Signal Processing and Analysis, 2005.},
  pages={64--69},
  year={2005},
  organization={Ieee}
}

@article{lewis2025scalable,
  title={Scalable emulation of protein equilibrium ensembles with generative deep learning},
  author={Lewis, Sarah and Hempel, Tim and Jim{\'e}nez-Luna, Jos{\'e} and Gastegger, Michael and Xie, Yu and Foong, Andrew YK and Satorras, Victor Garc{\'\i}a and Abdin, Osama and Veeling, Bastiaan S and Zaporozhets, Iryna and others},
  journal={Science},
  volume={389},
  number={6761},
  pages={eadv9817},
  year={2025},
  publisher={American Association for the Advancement of Science}
}
